\documentclass[10pt,twocolumn,letterpaper]{article}
\usepackage{cvpr}
\usepackage{times}
\usepackage{epsfig}
\usepackage{graphicx}
\usepackage{amsmath}
\usepackage{amssymb}
\usepackage[czech,english]{babel}

\usepackage{xcolor}
\usepackage[export]{adjustbox}
\usepackage{url}
\usepackage{amsthm}
\usepackage{caption}
\usepackage{subcaption}
\usepackage{appendix}

\usepackage[pagebackref=true,breaklinks=true,letterpaper=true,colorlinks,bookmarks=false]{hyperref}

\newtheorem{lemma}{Lemma}

\cvprfinalcopy % *** Uncomment this line for the final submission

\def\cvprPaperID{466} % *** Enter the CVPR Paper ID here

\newcommand{\occ}{o}

\newcommand\blue[1]{\textcolor{blue}{#1}}
\newcommand\yellow[1]{\textcolor{orange}{#1}}

\ifcvprfinal\fi

\newcommand{\LwithF}{\mathcal{L}}
\newcommand{\LwithoutF}{\mathcal{L}{\setminus\!\{\!f\!\}}}

\begin{document}

%%%%%%%%% TITLE
\title{Semantic 3D Reconstruction with Continuous Regularization and Ray Potentials Using a Visibility Consistency Constraint}
\author{Nikolay Savinov, Christian H\"ane, \v Lubor Ladick\'y and Marc Pollefeys\\
ETH Z\"urich, Switzerland\\
{\tt\small \{nikolay.savinov,christian.haene,lubor.ladicky,marc.pollefeys\}@inf.ethz.ch}
}

\maketitle
\thispagestyle{empty}

%%%%%%%%% ABSTRACT
\begin{abstract}
\vspace{-0.25cm}
We propose an approach for dense semantic 3D reconstruction which uses a data term that is defined as potentials over viewing rays, combined with continuous surface area penalization. Our formulation is a convex relaxation which we augment with a crucial non-convex constraint that ensures exact handling of visibility. To tackle the non-convex minimization problem, we propose a majorize-minimize type strategy which converges to a critical point. We demonstrate the benefits of using the non-convex constraint experimentally. For the geometry-only case, we set a new state of the art on two datasets of the commonly used Middlebury multi-view stereo benchmark. Moreover, our general-purpose formulation directly reconstructs thin objects, which are usually treated with specialized algorithms. A qualitative evaluation on the dense semantic 3D reconstruction task shows that we improve significantly over previous methods. Source code is available at \url{https://github.com/nsavinov/ray_potentials/}.
\end{abstract}

%%%%%%%%% BODY TEXT
\vspace{-0.45cm}
\section{Introduction}
One of the major goals in computer vision is to compute dense 3D geometry from images.
Recently, also approaches that jointly reason about the geometry and semantic segmentation have emerged \cite{Hane13}.
Due to the noise in the input data often strong regularization has to be performed.
Optimizing jointly over 3D geometry and semantics has the advantage that the smoothness for a surface can be chosen depending on the involved semantic labels and the normal direction to the surface.
Eventually, this leads to more faithful reconstructions that directly include a semantic labeling.

Posing the reconstruction task as a volumetric segmentation problem \cite{curless1996volumetric} is a widely used approach. A volume gets segmented into occupied and free space. In case of dense semantic 3D reconstruction, the occupied space label is replaced by a set of semantic labels \cite{Hane13}. To get smooth, noise-free reconstructions, the final labeling is normally determined by energy minimization.
The formulation of the energy comes with the challenge that the observations are given in the image space but the reconstruction is volumetric. Therefore each pixel of an image contains information about a ray composed out of voxels. This naturally leads to energy formulations with potential functions that depend on the configuration of a whole ray. Including such potentials in a naive way leads to (on current hardware) infeasible optimization problems. Hence, many approaches try to approximate such a potential. One often utilized strategy is to derive a per-voxel unary potential (cost for assigning a specific label to a specific voxel). However, this is only possible in a restricted setting and under a set of assumptions that often do not hold in practice.
By modeling the true ray potential, more faithful reconstructions are obtained \cite{Savinov15}. Thus, efficient ways to minimize energies with ray potentials, while at the same time being able to benefit from the joint formulation of 3D modeling and semantic segmentation, are desired.
\def \wcl {0.24} 
\begin{figure}
  \centering
  \includegraphics[width=\wcl\linewidth]{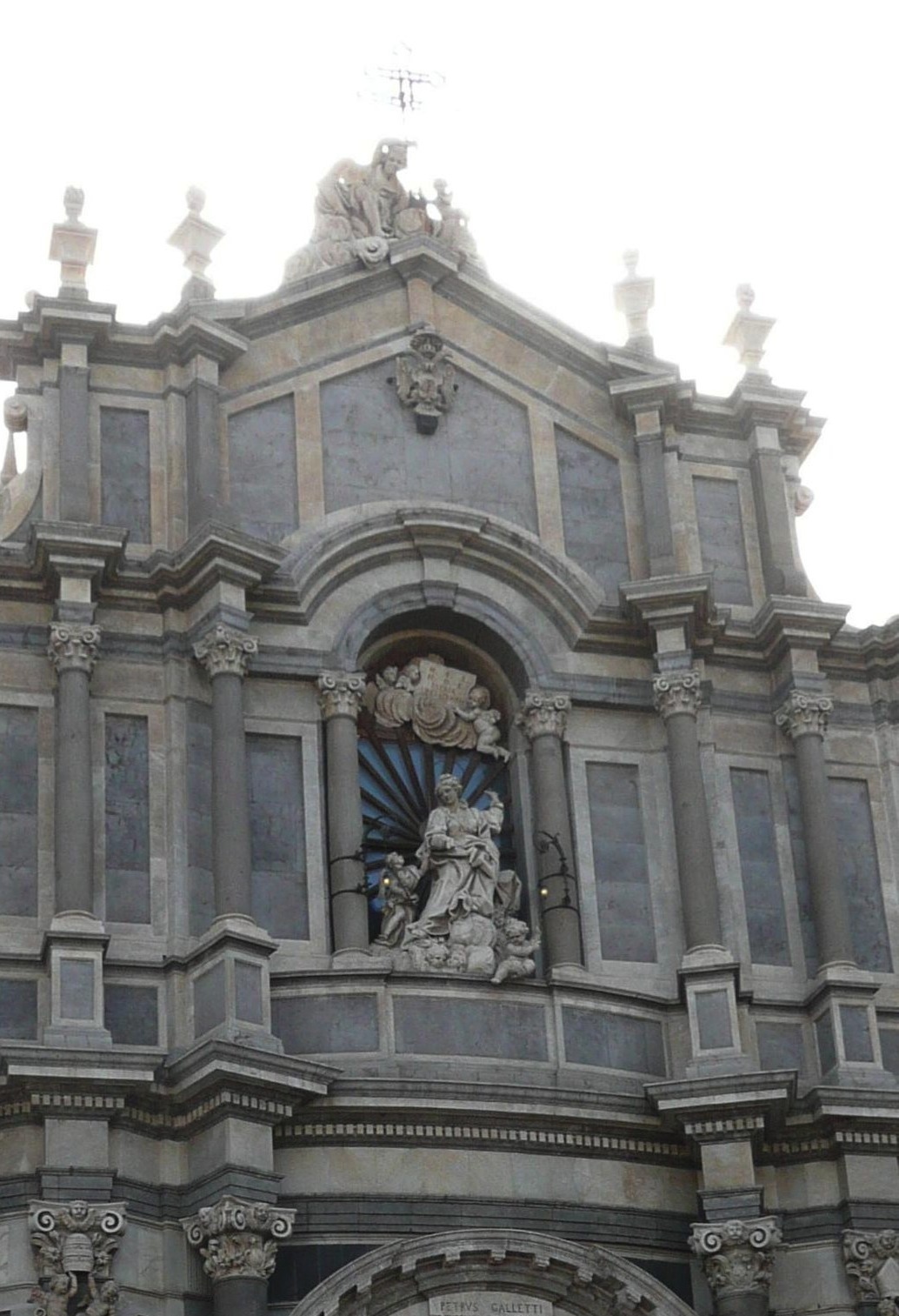}
  \includegraphics[width=\wcl\linewidth]{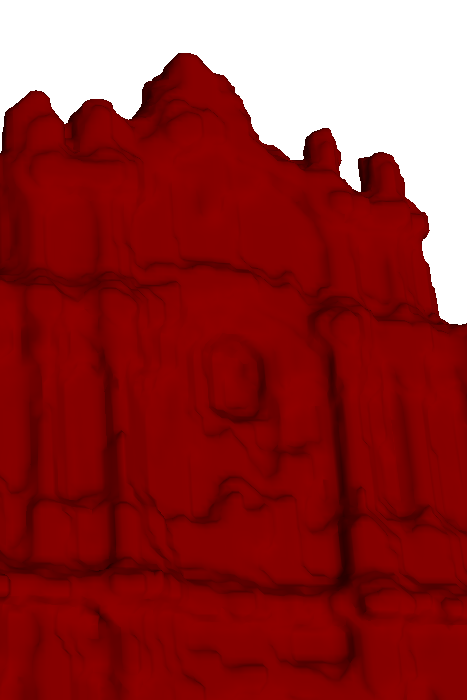}
  \includegraphics[width=\wcl\linewidth]{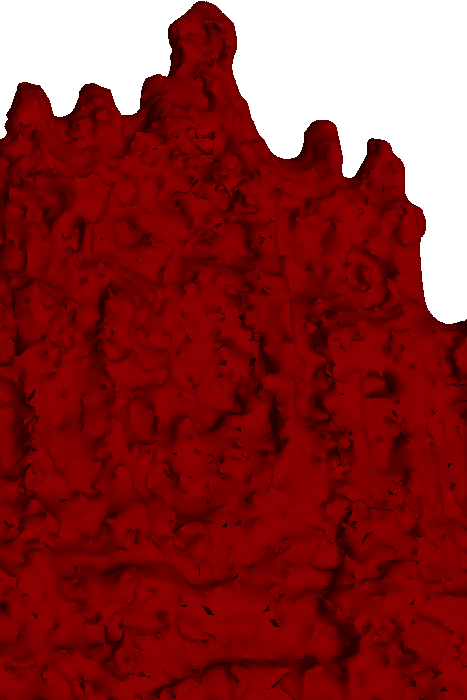}
    \includegraphics[width=\wcl\linewidth]{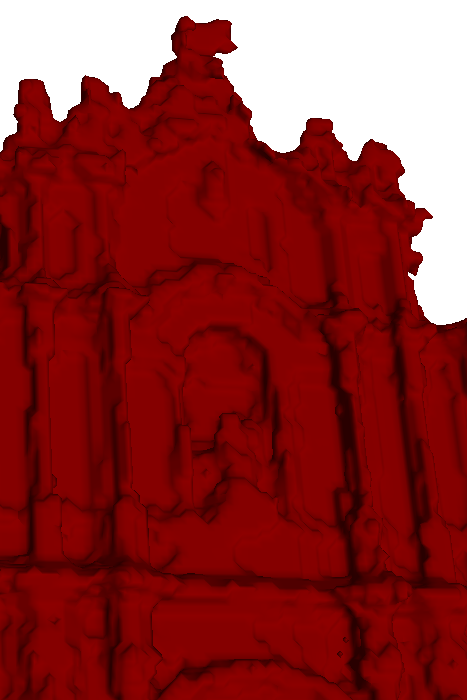}
   \caption{Left to right: example image, close-ups of \cite{Hane13}, \cite{Savinov15} and our proposed approach.}
\label{fig:closeUps}
\vspace{-0.25cm}
\end{figure}

In this work, we propose an energy minimization strategy for ray potentials that can be directly used together with continuously inspired surface regularization approaches and hence does not suffer from metrication artifacts~\cite{klodt2008experimental}, common to discrete formulations on grid graphs. When using our ray potential formulation for dense semantic 3D reconstruction, it additionally allows for the usage of class-specific anisotropic smoothness priors. Continuously inspired surface regularization approaches are formulated as convex optimization problems. We identify that a convex relaxation for the ray potential is weak and unusable in practice. We propose to add a non-convex term that handles visibility exactly and optimize the resulting energy by linearly majorizing the non-convex part. By regularly re-estimating the linear majorizer during the optimization, we devise an energy minimization algorithm with guaranteed convergence.

\subsection{Related Work}
Visibility relations in 3D reconstructions were studied for computing a single depth map out of multiple images \cite{kolmogorov2002multi, kolmogorov2003generalized}. To generate a full consistent 3D model from many depth maps, a popular approach is posing the problem in the volume \cite{curless1996volumetric}. To handle the noise in the input data a surface regularization term is added \cite{LempitskyB07,zach2007globally}. A discrete graph-based formulation is used in \cite{LempitskyB07} and continuous surface area penalization in \cite{zach2007globally}. One of the key questions is how to model the data term. Starting from depth maps, \cite{LempitskyB07,zach2007globally} model the 2.5D data as per-voxel unary potentials.
Such a modeling utilizes information contained in the depth map only partially.
Using a discrete graph formulation \cite{labatut2007delaunay,Vu2012} propose to model the free space between the camera center and the measured depth with pairwise potentials.

Another approach to modeling the data term is to directly use a photo-consistency-based smoothness term \cite{Sinha07,Hernandez07,kolev2008integration}. To resolve the visibility relations, image silhouettes are used. This is done either in the optimization as a constraint, meaning that the reconstruction needs to be consistent with the image silhouettes \cite{Sinha07, kolev2008integration}, or by deriving per-voxel occupancy probability \cite{Hernandez07}. Silhouette consistency is achieved through a discrete graph-cut optimization in \cite{Sinha07}, and with a convex-relaxation-based approach in the continuous domain in \cite{kolev2008integration}. The resulting relaxed problem in the latter case is not tight and hence does not generate binary solutions. Therefore, to guarantee silhouette consistency a special thresholding scheme is required. Handling visibility has also been done in mesh-based photo-consistency minimization \cite{delaunoy2011optimizing}.

To fully address the 2.5D nature of the input data, the true ray potential should be used, meaning the data cost depends on the first occupied voxel along the ray. What happens behind is unobserved and hence has no influence. This was formulated in \cite{pollard2007change, gargallo2007occupancy, liu2010ray, liu2014statistical, ulusoy2015rayPotentials} as a problem of finding a voxel labeling in terms of color and occupancy such that the first occupied voxel along a ray has a similar color as the pixel it projects to. One of the limitations all these works share is that they only compare colors of single pixels, which often does not give a strong enough signal to recover weakly textured areas. We use depth maps that are computed based on comparing image patches and interpret them as noisy input data containing outliers. We use regularization to handle the noise and outliers in the input data, but in contrast to other approaches with ray potentials that use a purely discrete graph-based formulation \cite{gargallo2007occupancy, liu2010ray, liu2014statistical} our proposed method is the first one that combines true ray potentials with a continuous surface regularization term. This allows us to set a new state of the art on two commonly used benchmark datasets. Unlike in any previous volumetric depth map fusion approach, thin surfaces do not pose problems in our formulation, due to an accurate representation of the input data.

Most earlier formulations of ray potentials are for purely geometry-based 3D reconstruction. Ours is more general and also allows to incorporate semantic labels.
\cite{Savinov15} shows that by using a discrete graph-based approach the true multi-label ray potential can be used as data term. Several artifacts present in the unary potential approximation \cite{Hane13} can be resolved using a formulation over rays. However, utilizing a discrete graph-based approach it is not directly possible to use the class-specific anisotropic regularization proposed in \cite{Hane13}. We bridge this gap and show how the full multi-label ray potential can be used together with continuously inspired anisotropic surface regularization \cite{chambolle2012convex, zach2014optimized}.

\section{Formulation}
In this section we will introduce the mathematical formulation that we are using to represent the dense semantic 3D reconstruction as an optimization problem. The problem is posed over a 3D voxel volume $\Omega \subset \mathbb{N}^3$. We denote the label $f=0$ as the free space label and introduce the set $\mathcal{L} = \{0, 1,\ldots,L\}$ of $L$ semantic labels, which represent the occupied space, plus the free space label. The final goal of our method is to assign a label $\ell \in \LwithF$ to each of the voxels. The label assignment is formalized using indicator variables $x_s^\ell \in \{0,1\}$ indicating if label $\ell$ is assigned at voxel $s \in \Omega$, ($x_s^\ell = 1$) or not.

We denote the vector of all per-voxel indicator variables as $\mathbf{x}$. Finally, the energy that we are minimizing in this paper has the form
\begin{align}
 E(\mathbf{x}) & = \psi_R(\mathbf{x}) + \psi_S(\mathbf{x}) \nonumber \\
  \mbox{subject to } & \sum_{\ell \in \LwithF} x_s^{\ell} = 1, \qquad x_s^{\ell} \in \{0,1\},
  \label{eq:energy}
\end{align}
\vskip-0.25cm
\noindent where $\sum_{\ell \in \mathcal{L}} x_s^\ell = 1$ guarantees that exactly one label is assigned to each voxel.
The objective contains two terms. The term $\psi_R(\mathbf{x})$, the ray potential, contributes the data to the optimization problem. This is in contrast to many other formulations where the data term is specified as local per-voxel preferences for the individual labels. The second term $\psi_S(\mathbf{x})$ is a smoothness term, which penalizes the surface area of the transitions between different labels. For this term, we utilize formulations originating from convex continuous multi-label segmentation \cite{chambolle2012convex}. As we will see in Sec.~\ref{sec:optimization} this smoothness term allows for class-specific anisotropic penalization of the interfaces between all labels. Due to the continuous nature of the regularization term, it does not suffer from metrication artifacts like most of the graph-based formulations.
The straightforward way to utilize such a smoothness term would be to use a convex relaxation of the ray potential. Unfortunately, convex relaxations of the ray potential do not seem to lead to strong formulations (\cf Fig.\ref{fig:slices}). In this paper we show how to resolve this problem by adding a non-convex constraint.
In Sec.~\ref{sec:rayPotential} we introduce the convex relaxation formulation of the ray potential and its non-convex extension. The regularization term and optimization strategy are discussed in Sec. \ref{sec:optimization}.

\begin{figure}
 \begin{minipage}{4cm}
 \centering
 \includegraphics[width=4cm]{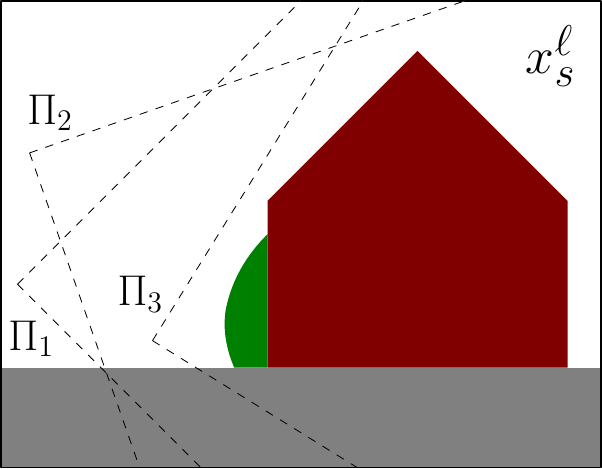} \\
 Global View
 \end{minipage}\hfill
 \begin{minipage}{4cm}
 \centering
 \includegraphics[width=4cm]{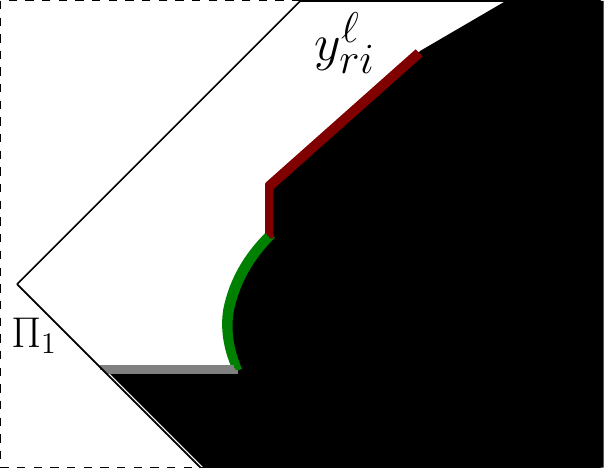} \hfill
 View from Camera $\Pi_1$
 \end{minipage}
 \caption{Variable Types: (left) the global $x_s^{\ell}$, indicate the label assigned to each voxel, (right) the per-ray variables $y_{ri}^{\ell}$ describe the visible surface.}
 \label{fig:variableTypes}
 \vspace{-0.25cm}
\end{figure}

\section{Ray Potential}
\label{sec:rayPotential}

The main idea of the ray potential \cite{Savinov15} is that for each ray, originating from an image pixel, a cost is induced that only depends on the position of the first non-free space label along the viewing ray or the ray is all free space. This means that the potential can take only linearly many (in the number of voxels along a ray) values, which is the reason why optimization of such potentials is tractable. Note that this is not a restriction we impose, it represents the fact that it is impossible to see behind occupied space. We denote the cost of having the first occupied space label at position $i$ with label $\ell \in \LwithF$ as $c_{ri}^{\ell}$ and the cost of having the whole ray free space as $c_r^{f}$.

The vector of indicator variables $x_s^{\ell}$ belonging to ray $r \in \mathcal{R}$ is denoted as $\mathbf{x}_r$. To index positions along a ray, we denote $s_{ri} \in \Omega$ as the positions of all the voxels belonging to ray $r \in \mathcal{R}$, where $i \in \{0,\cdots,N_r \}$ denotes the position index along the ray. Note that there exists only one $x_s^{\ell}$ variable per label for each voxel $s \in \Omega$, if $s_{ri}$ evaluates to the same position for different rays it refers to the same variable. Now we can state the ray potential part of the energy as a sum of potentials over rays
\begin{align}
 \psi_R(\mathbf{x}) &= \sum_{r \in \mathcal{R}} \psi_r(\mathbf{x}_r) \label{eq:rayPotential} \\
 \quad \psi_r(\mathbf{x}_r) & = \!\! \left( \! \sum_{\ell \in \LwithoutF} \sum_{i = 0}^{N_r} c_{ri}^{\ell} \!\! \left(\min_{j \leq i-1} \! x_{s_{rj}}^f \!\! \right) \! x_{s_{ri}}^\ell \!\! \right) \!\! + c_r^{f} \! \min_{j \leq N_r} \! x_{s_{rj}}^f \nonumber
\end{align}
with $\LwithoutF$ meaning the set of all labels excluding the free space label. The term $(\min_{j \leq i-1} x_{s_{ri}}^f)  x_{s_{ri}}^{\ell}$ is $1$ iff the first occupied label along the ray $r$ is $\ell$ at position $i$. Similarly, $\min_{j \leq N_r} x_{s_{rj}}^f$ equals $1$ iff the whole ray $r$ is free space. Thus, in Eq.~\ref{eq:rayPotential} only one term is non-zero, and its coefficient equals the desired cost of the ray configuration.

To make the derivations throughout the paper compact, we omit the last term without loss of generality by shifting the costs by a constant, $c_{ri}^\ell \leftarrow c_{ri}^\ell - c_r^f$ and $c_r^f \leftarrow 0$.

\subsection{Visibility Variables}

\begin{figure}
 \centering
 \includegraphics[width=0.8\linewidth]{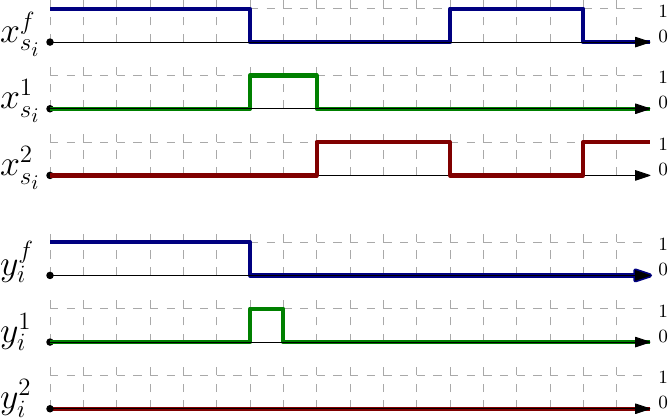}
 \caption{Example of variable assignments along a single viewing ray for a three-label problem.}
 \label{fig:exampleVariableAssignments}
 \vspace{-0.25cm}
\end{figure}

Before we state a convex relaxation of the ray potential, which we eventually augment with a non-convex constraint, we rewrite the potential using visibility variables. First, we introduce the visibility variables $y_{ri}^{\ell}$ indicating that the ray $r$ only contains free space up to the position $i-1$ and the label assigned at position $i$ is $\ell \in \LwithF$. 
\begin{align}
 y_{ri}^\ell & = \min(y_{r,i-1}^f, x_{s_{ri}}^\ell)
\label{eq:defVisVariables}
\end{align}
To anchor the definition we assume that the -1\textsuperscript{st} voxel of the ray has free space assigned, $y_{r,-1}^{f} = 1$. Note that if we insert all the nested definitions for a free space variable we get $y_{ri}^{f} = \min_{j \leq i} x_{s_{rj}}^f$.
Note that this variables are per-ray local variables and multiple ones can exist per voxel in case multiple rays cross that voxel in contrast to the global per-voxel variables $x_s$, which exist only once per voxel (\cf Fig.~\ref{fig:variableTypes}). In the remainder of Sec.~\ref{sec:rayPotential} we will drop the index $r$ at most places for better readability. We now state a reformulation of the ray potential as
\begin{align}
\psi_r({\bf x}_r) & = \sum_{\ell \in \LwithF} \sum_{i = 0}^{N} c_{i}^\ell y_{i}^\ell \label{eq:rayPotentialVisibility} \\
\mbox{subject to } \quad &  y_{i}^\ell = \min(y_{i-1}^f, x_{s_{i}}^\ell) \;\; \forall \ell \in \LwithF, \forall i, \nonumber
\end{align}
Here we introduced costs along the ray also for the free space label $c_i^f = 0, \, \forall i$. This does not change the potential but is required for our next step, where we reformulate the non-convex equality constraints as a series of inequality constraints. To make sure that the corresponding equalities are still satisfied in the optimum, we show that the costs $c_{i}^\ell$ can be replaced by non-positive ones without changing the minimizer of the energy. This means that the $y_{i}^\ell$ are bounded from above by the linear inequality constraints and are tight from below through the minimization of the cost function, so the resulting constraints model the same optimization problem. The inequality constraints read as follows
\begin{align}
 0 \leq y_{i}^\ell \leq y_{i-1}^{f}, \quad y_{i}^\ell \leq x_{s_{i}}^\ell \qquad \forall \ell \in \LwithF
\end{align}
To derive the transformation to non-positive costs, we first notice that after applying $\min(y_{i-1}^f, \cdot)$ to both sides of the constraint $\sum_{\ell \in \LwithF} x_{s_{i}}^\ell = 1$ from Eq.~\ref{eq:energy}, we can plug in the constraints of Eq.~\ref{eq:rayPotentialVisibility} to obtain 
\begin{equation}\label{tightness_condition}
y_{i-1}^f = \sum_{\ell \in \LwithF} y_{i}^\ell.
\end{equation}
Intuitively, this means if position $i-1$ is in the observed visible free space then the next position is either free space or one of the occupied space labels and if $i-1$ is in the occupied space then all the $y_i^{\ell}$ are $0$ (see Fig.~\ref{fig:exampleVariableAssignments}).
The cost transformation is done for every ray separately. Starting with the last position $i = N$, we add the following expression, which always evaluates to $0$, to the ray potential.
\begin{equation}
\left( \max_{\ell' \in \LwithF} c_{i}^{\ell'} \right) \left(y_{i-1}^f -\sum\limits_{\ell \in \LwithF} y_{i}^\ell\right) = 0
\end{equation}
This moves one non-negative term to the previous position and make all the $c_{i}^{\ell}$ for the current position non-positive.
\begin{align}
& c_{i-1}^f \leftarrow c_{i-1}^f + \max_{\ell' \in \LwithF} c_{i}^{\ell'} \nonumber \\
& c_{i}^\ell \leftarrow c_{i}^\ell - \max_{\ell' \in \LwithF} c_{i}^{\ell'} \quad \forall \ell \in \LwithF
\end{align}
This is done iteratively for all $i \in \{N,\ldots,0\}$, leaving just a constant, which can be omitted. 

\subsection{Convex Relaxation and Visibility Consistency} \label{sec:convex}
 \begin{figure} 
 \includegraphics[width=0.3\linewidth]{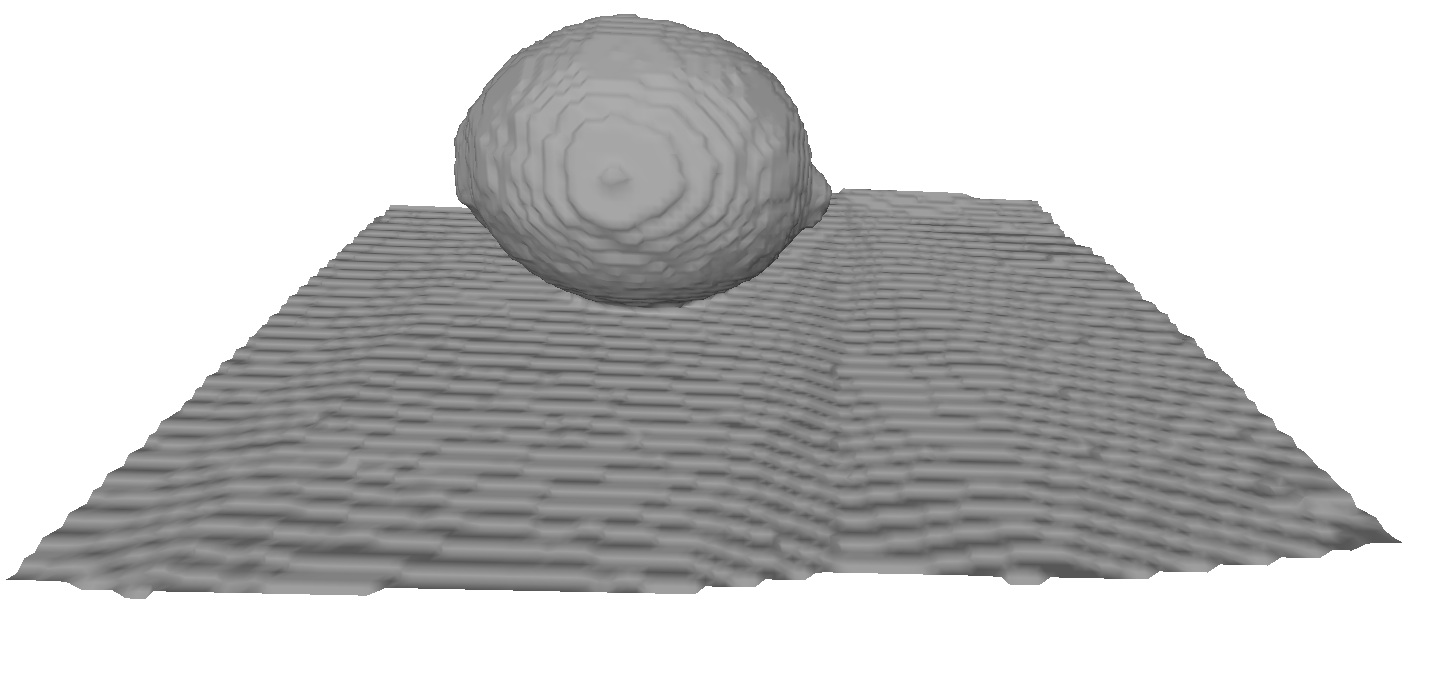} \hfill
 \includegraphics[width=0.3\linewidth, cfbox=blue 1pt 0pt]{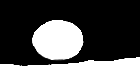} \hfill
 \includegraphics[width=0.3\linewidth, cfbox=blue 1pt 0pt]{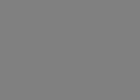}
 \caption{\label{fig:slices} Evaluation of the convex relaxation for two-label problem: \textbf{(left)} a reconstruction of the model obtained by our non-convex procedure, slices through the volume ($0$ black, $1$ white, $0.5$ grey) in the non-convex formulation \textbf{(middle)} and the convex formulation \textbf{(right)}.}
 \vspace{-0.25cm}
 \end{figure}

So far our derivation has been done using binary variables $x_s^{\ell} \in \{0,1\}$ and hence also all the $y_{i}^{\ell} \in \{0,1\}$. To minimize the energy, we relax this constraint by replacing $x_s^{\ell} \in \{0,1\}$ with $x_s^{\ell} \in [0,1]$ in Eq.~\ref{eq:energy}. This directly leads to a convex relaxation of the ray potential. Unfortunately, this relaxation is weak and therefore inapplicable in practice. In Fig.~\ref{fig:slices} we evaluate the convex relaxation on a two-label example (Lemon dataset), using surface area penalization via a total variation (TV) smoothness prior. The convex relaxation fails entirely, producing variable assignments to the $x_s^{\ell}$ that are $0.5$ up to machine precision and hence no meaningful solution can be extracted. A comparison of the energies reveals that there is a significant difference between the non-convex and the convex solution ($626614$ and $431893$, respectively), which indicates that the relaxed problem is far from the original binary one. Most importantly, our earlier convex formulation \cite{Savinov15} shares this behavior of not making a decision for any voxel, when run without initialization on a two-label problem. The aspects of initialization, heuristic assignment of unassigned variables, move making algorithm, and a coarse-to-fine scheme are essential elements of the algorithm in \cite{Savinov15}.

The reason for the weak relaxation is that Eq.~\ref{tightness_condition} is unsatisfied for the solution of the convex relaxation. This equation ensures that the per camera local view is consistent with the global model (\cf Fig.~\ref{fig:variableTypes}). Concretely, the equation states that the change in visibility is directly linked to the cost that can be taken by the potential \eg a surface can only be placed iff the occupancy along the ray changes. Hence we propose a formulation that directly enforces this constraint, which we will call visibility consistency constraint.
Eq.~\ref{tightness_condition} can be reformulated using the definition of $y_{i}^f$ as
\begin{multline}
\sum_{\ell \in \LwithoutF} y_{i}^\ell =  y_{i-1}^f - y_{i}^f = y_{i-1}^f - \min(y_{i-1}^f, x_{s_{i}}^f) \\ = \max(0, y_{i-1}^f - x_{s_{i}}^f). \label{eq:visibilityConsistency}
\end{multline}
This means that we can only have an occupied space label $\ell \in \LwithoutF$ assigned as the visible surface at position $i$, if position $i$ does not have free space assigned and $y_{i-1}^f = 1$ and hence the whole ray from the camera center to the position $i-1$ has free space assigned (see Fig.~\ref{fig:exampleVariableAssignments}).

Since we minimize the objective with non-positive $c_{i}^{\ell}$, the visibility consistency constraint is equivalent to the inequality
\begin{equation}
\sum\limits_{\ell \in \LwithoutF} y_{i}^\ell \leq \max(0, y_{i-1}^f - x_{s_{i}}^f).
\end{equation}

Our final formulation for the ray potential is
\begin{align}\label{eq:rayPotNonCvx}
 \psi_r({\bf x}_r) & = \sum_{\ell \in \LwithF} \sum_{i = 0}^{N} c_{i}^\ell y_{i}^\ell \\
\mbox{s.t. } &\; y_{i}^\ell \leq y_{i-1}^{f}, \; y_{i}^\ell \leq x_{s_{i}}^\ell, \; y_{i}^{\ell} \geq 0 \;\; \forall \ell \in \LwithF, \forall i \nonumber \\
 &  \sum\limits_{\ell \in \LwithoutF} y_{i}^\ell \leq \max(0, y_{i-1}^f - x_{s_{i}}^f)  \qquad \forall i \nonumber
\end{align}

The above potential is non-convex because of the non-convex inequality which describes visibility consistency.
We follow the strategy of using a surrogate convex constraint for the non-convex one that majorizes the objective of the non-convex program. The majorization, as we will see in Sec.~\ref{sec:optimization}, happens during the iterative optimization. Therefore, at each iteration, we have a current assignment to the variables, which we denote by $\mathbf{x}^{(n)}$ and $\mathbf{y}^{(n)}$. Here we introduced the notation that variable assignments at iteration $n$ are denoted with a superscript $(n)$. Replacing 
\begin{align}
 & \sum_{\ell \in \LwithoutF} y_{i}^\ell \leq \max \{ 0, y_{i-1}^f - x_{s_{i}}^f \} \nonumber \\ & \mbox{ by } \quad
 \sum_{\ell \in \LwithoutF} y_{i}^\ell \leq g(x_{s_{i}}^f,y_{i-1}^f | x_{s_{i}}^{f,(n)}, y_{i-1}^{f,(n)}) \label{eq:majorization}
\end{align}
with the linear majorizer,
\begin{multline}
  g(x_{s_{i}}^f,y_{i-1}^f | x_{s_{i}}^{f,(n)}, y_{i-1}^{f,(n)})  \\ =  \begin{cases}
   0 & \text{if } y_{i-1}^{f,(n)} \leq  x_{s_{i}}^{f,(n)} \\
   y_{i-1}^f - x_{s_{i}}^f & \text{if } y_{i-1}^{f,(n)} > x_{s_{i}}^{f,(n)}
 \end{cases}
 \label{eq:linearization}
\end{multline}
leads to a surrogate linear (and therefore convex) ray potential, which we will denote by $\psi_r^{(n)}(\mathbf{x},\mathbf{y} | \mathbf{x}^{(n)}, \mathbf{y}^{(n)})$. The variables $\mathbf{x}^{(n)}$ and $\mathbf{y}^{(n)}$ denote the position of the linearization. We handle the corner case where both branches are feasible to always take the first branch. In numerical experiments we observed that this choice is not critical, it makes no significant difference which branch is used in this case.

Next we state a Lemma that will be a crucial part of the optimization strategy detailed in Sec.~\ref{sec:optimization}. 

\begin{lemma}
 Given $\mathbf{x}^{(n)}$, with $\mathbf{x}^{(n)} \geq 0$ point-wise, we can find $\mathbf{\tilde{y}}^{(n)}$ such that all the constraints of the ray potential Eq.~\ref{eq:rayPotNonCvx} are fulfilled and the value of the potential is minimal. \label{lm:rayFeasiblity}
\end{lemma}
\vspace{-0.3cm}
Intuitively, the lemma states that given the global per-voxel variable assignments $x_s^{\ell}$, an assignment to the per-ray variables $y_{i}^{\ell}$ can be found. This is not surprising given that the whole information about the scene is contained in the variables $x_s^{\ell}$ (\cf Fig.~\ref{fig:variableTypes}). We prove the lemma by giving a construction.
\begin{proof}
We provide an algorithm that computes $\mathbf{\tilde{y}}^{(n)}$ for each ray individually.
First we set $\tilde{y}^{f,(n)}_{i} = \min_{j \leq i} x_{s_{j}}^{f,(n)}$, which satisfies $\tilde{y}^{f,(n)}_{i} \leq \tilde{y}^{f,(n)}_{i-1}$, $\tilde{y}^{(n)}_{i\ell} \geq 0$.
Now we iteratively increase $\tilde{y}_{i}^{\ell,(n)}$ such that  $\sum_{\ell \in \LwithoutF} \tilde{y}_{i}^{\ell,(n)} \leq \max(0,\tilde{y}_{i-1}^{f,(n)} - x_{s_{i}}^f)$ and $\tilde{y}_{i}^{\ell,(n)} \leq x_{s_{i}}^{\ell}$. For an optimal assignment we do this procedure in an increasing order of $c_{i}^\ell$. The observation holds by construction.
\end{proof}

\section{Energy Minimization Strategy}
\label{sec:optimization}

Before we discuss the proposed energy minimization, we complete the formulation by including the regularization term.

\subsection{Regularization Term}
There are several choices of regularization terms for continuously inspired multi-label segmentation that can be inserted into our formulation \cite{zach2008labeling, chambolle2012convex, strekalovskiy2011generalized, zach2014optimized}. They are all convex relaxations and are originally posed in the continuum and discretized for numerical optimization. The main differences are the strength of relaxation and generality of the allowed smoothness priors. We directly describe the strongest, most general version, which allows for non-metric and anisotropic smoothness \cite{zach2014optimized}. We only state the smoothness term and explain the meaning of the individual variables. For a thorough mathematical derivation we refer the reader to the original publications \cite{zach2014optimized, Hane13}.
\begin{align}
 \psi_S(\mathbf{x}, \mathbf{z}) &= \sum_{s \in \Omega} \psi_s(\mathbf{x}, \mathbf{z}) \quad \mathrm{with} \label{eq:regularizationTerm} \\
 \psi_s(\mathbf{x}, \mathbf{z}) &= \sum_{\ell,m: \ell < m} \phi_s^{\ell m} (z_s^{\ell m} - z_s^{m \ell}) \nonumber \\
 \mbox{s.t. } x_s^\ell &= \sum_m \left(z_s^{\ell m}\right)_k, x_{s}^\ell = \sum_m \left( z_{s-e_k}^{m \ell} \right)_k, \forall k, z_s^{\ell m} \geq 0. \nonumber
\end{align}
The variables $z_s^{\ell m} \in \mathbb{R}^3$ describe the transitions between the assigned labels. They indicate how much change there is from label $\ell$ to label $m$ along the direction they point to and are hence called label transition gradients. For example, if there is a change from label $\ell$ to label $m$ at voxel $s$ along the first canonical direction, the corresponding $z_s^{\ell m}$ is $[1, 0, 0]^T$. The $z_s^{\ell m}$ need to be non-negative in order to allow for general, non-metric smoothness priors \cite{zach2014optimized}. Therefore the difference $z_s^{\ell m} - z_s^{m \ell}$ is used to allow for arbitrary transition directions. The variable $e_k$ denotes the canonical basis vector for the $k$-th component, \ie $e_1 = [1, 0, 0]^T$. $\phi_s^{\ell m} : \mathbb{R}^3 \rightarrow \mathbb{R}_0^+$ are convex positively $1$-homogeneous functions that act as anisotropic regularization of a surface between label $\ell$ and $m$. Note that the regularization term takes into account label combinations. This enables us to select class-specific smoothness priors, which depend on the surface direction and the involved labels and are inferred from training data \cite{Hane13}. For example, a surface between ground and building is treated differently from a transition between free space and building.
The following lemma will be necessary for our optimization strategy.
\begin{lemma}
 Given $\mathbf{x}^{(n)}$, $\mathbf{z}^{(n)}$, with $x_s^{\ell,(n)} \geq 0$ $\forall s,\ell$ an assignment $\mathbf{\tilde{z}}^{(n)}$ can be determined that fulfills the constraints of the regularization term.
 \label{lm:regularizerFeasibility}
\end{lemma}
For the full proof of the lemma we refer the reader to the supplementary material, here we only state the main idea of the proof. In a first step we project our current solution onto the space spanned by the equality constraints. This leads to an initialization of the $\mathbf{\tilde{z}}^{(n)}$ which fulfills the equality constraints but might lead to negative assignments to the $z_s^{\ell m}$. To get a non-negative solution, we notice that as long as there is a $z_s^{\ell',m'}$ which is negative we can find $\ell''$ and $m''$ such that we can increase $z_s^{\ell',m'}$ by $\epsilon$ along with changing $z_s^{\ell'',m'}, z_s^{\ell',m''}, z_s^{\ell'',m''}$ by the same $\epsilon$ in order not to affect the equality constraints.

\subsection{Optimization}
The goal of this section is to minimize the proposed energy using the non-convex ray potential Eq.~\ref{eq:rayPotNonCvx}. Optimizing non-convex functionals is an inherently difficult task. One often successfully utilized strategy is the so called majorize-minimize strategy (for example \cite{lange2000optimization}). The idea is to majorize the non-convex functional in some way with a surrogate convex one. Alternating between minimizing the surrogate convex energy, which we will call the minimization step in the following, and recomputing the surrogate convex majorizer, which we will denote the majorization step, leads to an algorithm that decreases the energy at each step and hence converges.

Note that we already discussed the majorization step of the ray potential in Sec.~\ref{sec:rayPotential}, Eq.~\ref{eq:linearization}. Together with the regularizer we end up with a surrogate convex but non-smooth program.
\begin{align}
 E^{(n)}(\mathbf{x},\mathbf{y},\mathbf{z}) & = \psi^{(n)}_R(\mathbf{x}, \mathbf{y} | \mathbf{x}^{(n)},\mathbf{y}^{(n)}) + \psi_S(\mathbf{x},\mathbf{z})  \label{eq:finalSurrogateConvex} \\
  \mbox{s.t. } \sum_{\ell \in \LwithF} x_s^{\ell} & \!= \!1 \; \forall s, \;\; x_s^{\ell} \in [0,1] \; \forall s \in \Omega,\; \forall \ell \in \LwithF. \nonumber
\end{align}
This energy can be globally minimized using the iterative first order primal-dual algorithm \cite{pock2011diagonal}.  However, there is no guarantee that the energy during the iterative minimization decreases monotonically nor that the constraints are fulfilled before convergence. One solution is to run the convex optimization until convergence however in practice this leads to slow convergence. Therefore, we follow a different strategy where we regularly run the majorization step during the optimization of the energy. Before we can state the final algorithm we present the following lemma.
\begin{lemma}
 Given $\mathbf{x}^{(n)},\mathbf{y}^{(n)},\mathbf{z}^{(n)}$, in the optimization problem Eq.~\ref{eq:finalSurrogateConvex}, which do not necessarily fulfill the constraints. A feasible solution $\mathbf{\tilde{x}}^{(n)},\mathbf{\tilde{y}}^{(n)},\mathbf{\tilde{z}}^{(n)}$ to the ray potential Eq.~\ref{eq:rayPotNonCvx} and the regularization term Eq.~\ref{eq:regularizationTerm} can be constructed in a finite number of steps.
 \label{lm:feasibility}
\end{lemma}
\begin{proof}
 To fulfill the constraints $\sum_{\ell \in \LwithF} x_s^{\ell} = 1$ and $x_s^{\ell} \in [0,1]$ we project the variables $\mathbf{x}^{(n)}$ individually per voxel $s$ to the unit probability simplex \cite{duchi2008efficient}. Subsequent application of Lemma \ref{lm:rayFeasiblity} and \ref{lm:regularizerFeasibility} leads to the desired result.
\end{proof}  
\begin{figure*}
 \begin{center}
 \begin{minipage}{4.5cm}
 \includegraphics[width=0.49\linewidth]{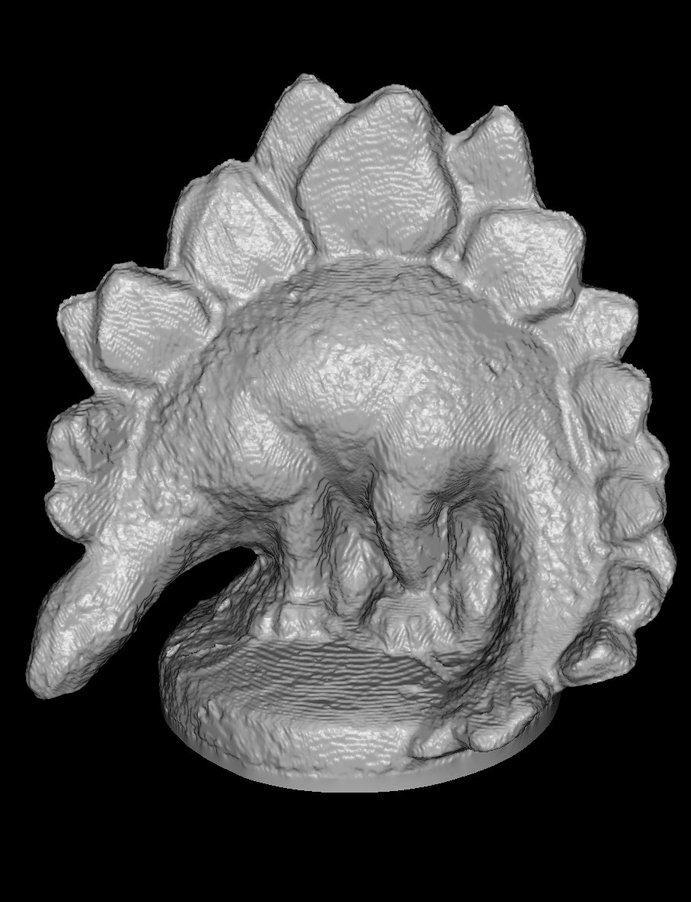} 
 \includegraphics[width=0.49\linewidth]{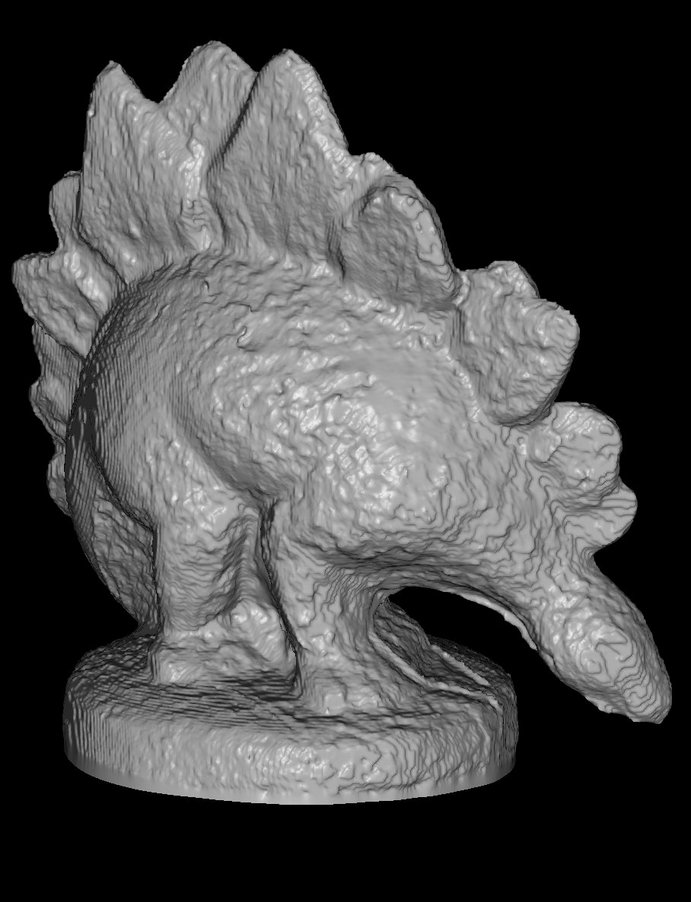} \vspace*{-0.35cm} \\
 \includegraphics[width=0.49\linewidth]{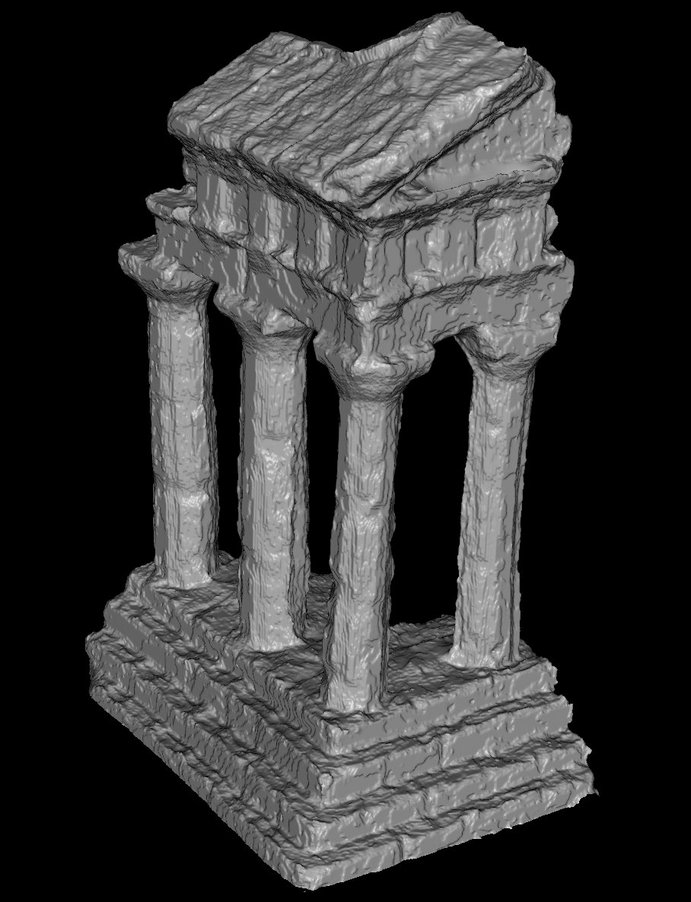} 
 \includegraphics[width=0.49\linewidth]{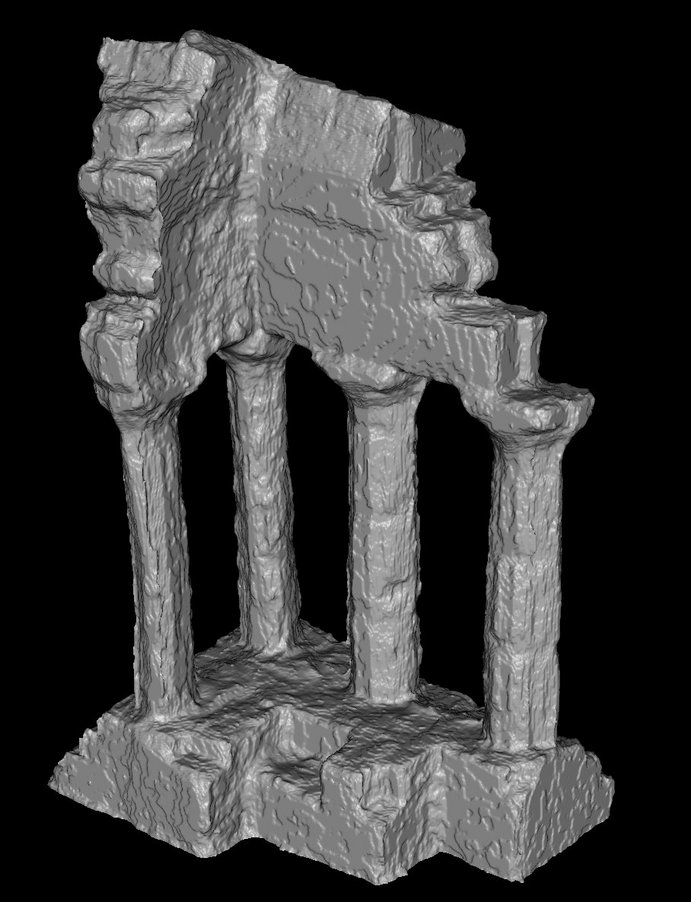} \vspace*{-0.3cm}
 \end{minipage}
 \begin{minipage}{12.5cm}
  \scriptsize
  \begin{tabular}{|l|c|c|c|c|c|c|}
\hline
                                & \begin{tabular}[c]{@{}c@{}}Temple Full\\ Acc / Comp\end{tabular} & \begin{tabular}[c]{@{}c@{}}Temple Ring\\ Acc / Comp\end{tabular} & \begin{tabular}[c]{@{}c@{}}Temple S. Ring\\ Acc / Comp\end{tabular} & \begin{tabular}[c]{@{}c@{}}Dino Full\\ Acc / Comp\end{tabular} & \begin{tabular}[c]{@{}c@{}}Dino Ring\\ Acc / Comp\end{tabular} & \begin{tabular}[c]{@{}c@{}}Dino S. Ring\\ Acc / Comp\end{tabular} \\ \hline
Our Method                       & 0.41 / 99.7                                                      & 0.5 / 99.5                                                       & 0.69 / 97.8                                                             & \textbf{0.26} / 99.8                                                    & \textbf{0.25} / 99.9                                                    & 0.34 / 99.7                                                           \\ \hline
Galliani et al.\cite{Galliani}                  & 0.39 / 99.2                                                      & 0.48 / 99.1                                                      & 0.53 / 97.0                                                             & 0.31 / 99.9                                                    & 0.3 / 99.4                                                     & 0.38 / 98.6                                                           \\ \hline
Zhu et al.\cite{Zhu}                       &                                                                  & \textbf{0.4} / 99.2                                                       & \textbf{0.45} / 95.7                                                             &                                                                & 0.38 / 98.3                                                    & 0.48 / 95.4                                                           \\ \hline
Li et al.\cite{DCV}             &                                                                  & 0.73 / 98.2                                                      & 0.66 / 97.3                                                             &                                                                & 0.28 / 100                                                     & \textbf{0.3} / 100                                                             \\ \hline
Wei et al. \cite{BMVC2014_183}   & \textbf{0.34} / 99.4                                                      &                                                                  &                                                                         & 0.42 / 98.1                                                    &                                                                &                                                                       \\ \hline
Xue et al.&                                                                  &                                                                  &                                                                         &                                                                &                                                                & \textbf{0.3} / 99.1                                                            \\ \hline
Furukawa et al. \cite{pmvs} & 0.49 / 99.6                                                      & 0.47 / 99.6                                                      & 0.63 / 99.3                                                             & 0.33 / 99.8                                                    & 0.28 / 99.8                                                    & 0.37 / 99.2                                                           \\ \hline
\end{tabular} \\
 \hspace*{1cm} \includegraphics[width=0.35\linewidth]{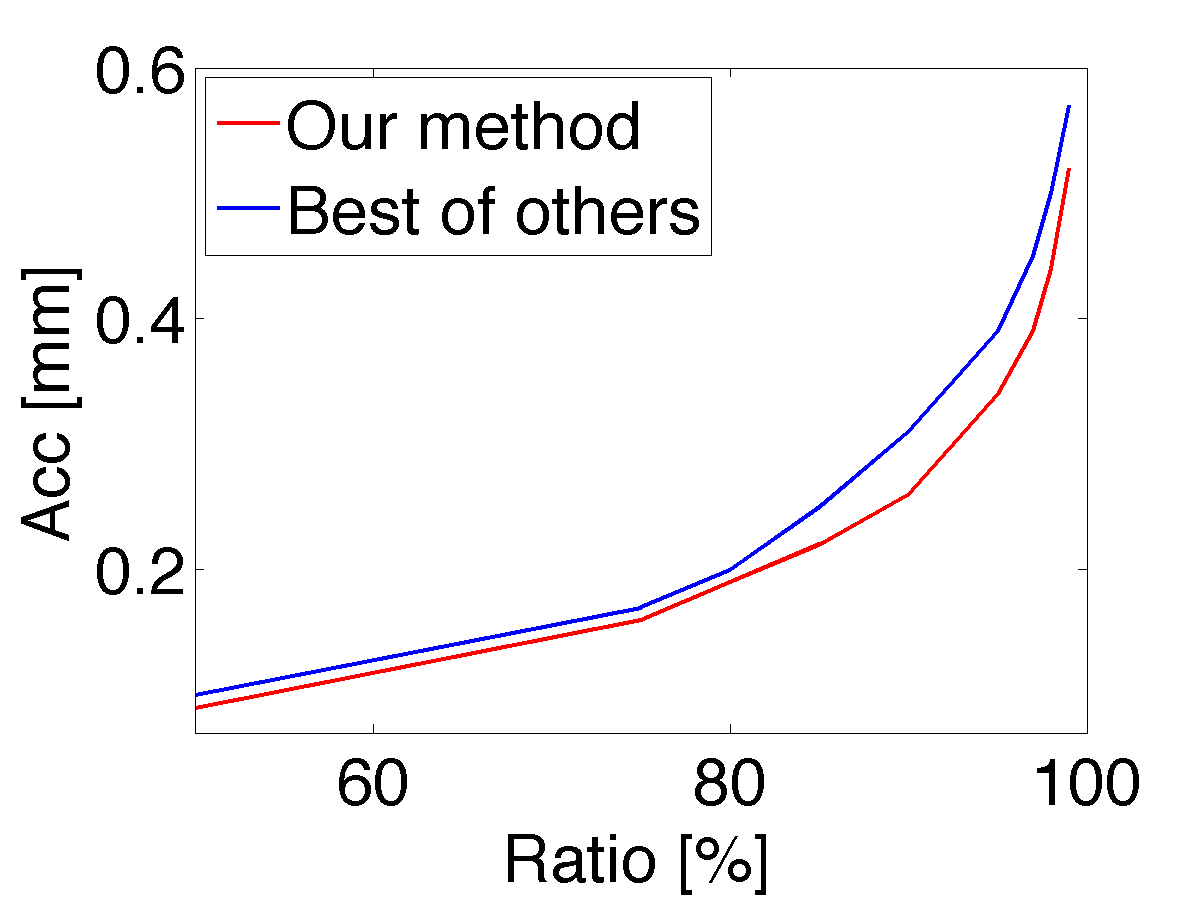} \hspace{1cm} \includegraphics[width=0.35\linewidth]{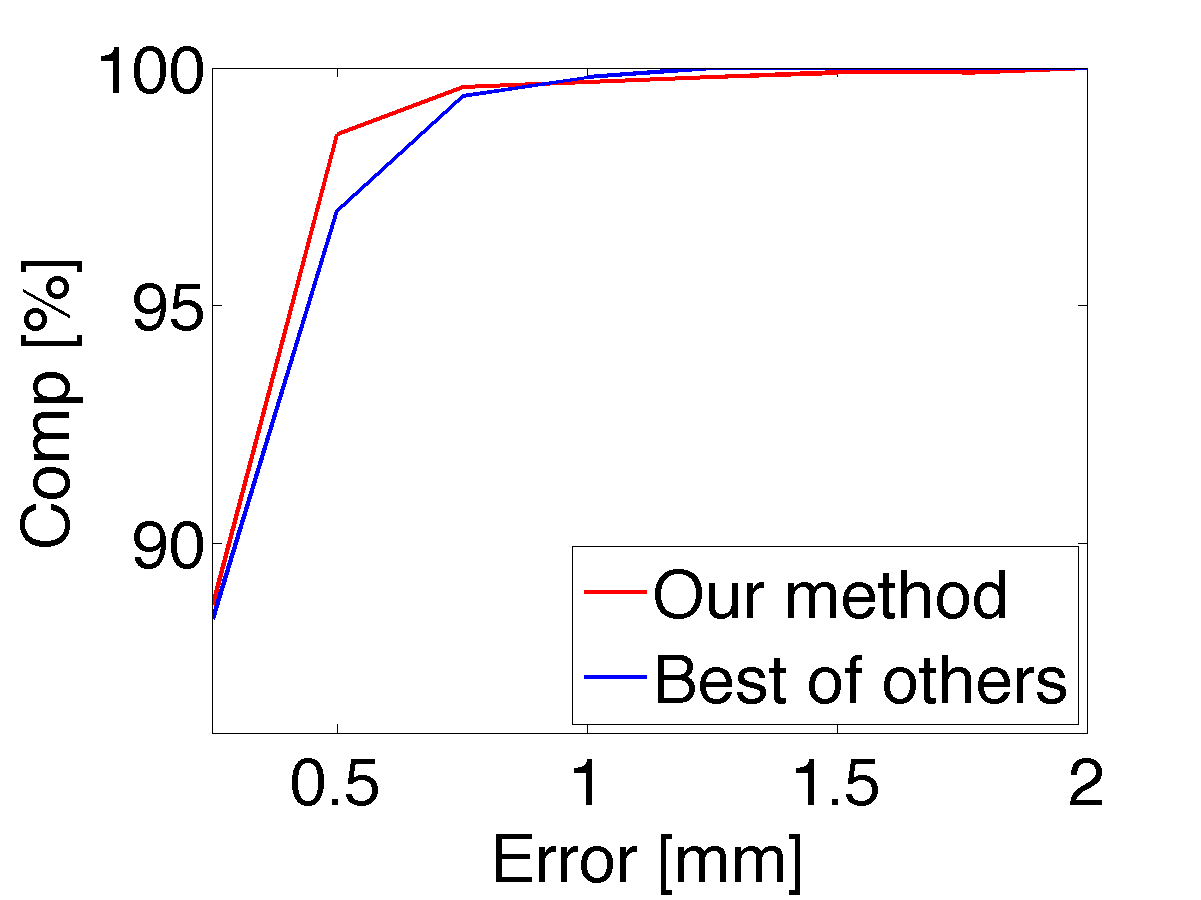}
 \end{minipage}
 \caption{\label{fig:middlebury_reconstructions} Middlebury Multi-View Stereo Benchmark: \textbf{(left)} Reconstructions of the Dino Ring and Temple Ring datasets computed by our algorithm, (left-most) view 1, (right-most) view 2, \textbf{(right, top)} benchmark's competitive methods in inverse chronological order (smaller Acc and higher Comp numbers are better), \textbf{(right, bottom)} Acc vs. Ratio (lower curve better) and Comp vs. Error (higher curve better) plots for the Dino Full dataset (for details on these plots see \cite{middlebury}).} 
 \end{center}
 \end{figure*}

Our final majorize-minimize optimization strategy can now be stated as follows.
\begin{description}
 \item [Majorization step] Using the current variables $\mathbf{x}^{(n)},\mathbf{y}^{(n)},\mathbf{z}^{(n)}$ and Lemma~\ref{lm:feasibility} a feasible solution $\mathbf{\tilde{x}}^{(n)},\mathbf{\tilde{y}}^{(n)},\mathbf{\tilde{z}}^{(n)}$ can be found. If the new energy is lower or equal than the last known energy the linearization Eq.~\ref{eq:linearization} is applied and the new optimization problem is passed to the minimization step. Otherwise the whole majorization step is skipped and the old variables $\mathbf{x}^{(n)},\mathbf{y}^{(n)},\mathbf{z}^{(n)}$ are passed to the minimization step.
 \item [Minimization step] The primal-dual algorithm \cite{pock2011diagonal} is run on the surrogate convex program for a fixed number of $p$ iterations. For guaranteed convergence, the primal dual gap $\eta$ can be evaluated and the minimization step can be restarted until we have $\eta \leq f(n)$ with a function $f(n) \rightarrow 0$ for $n \rightarrow \infty$. In practice, we get a good convergence behavior without a restart.
\end{description}
The majorization step either does no changes to the optimization state or finds a better or equal solution with a new linearization of the non-convex part because the linearization does not change the energy. If no changes are done this can be due to two reasons. Either the current solution was worse than the last known one, or the majorization stayed the same. In the latter case the primal-dual gap could reveal convergence and as a consequence we know that we arrived at a critical point or corner point of the original non-convex energy. In any other case the minimization step is run again and due to the convexity of the surrogate convex function a better solution or the optimality certificate will be found. The above procedure could stop at a non-critical point\footnote{similar to block-coordinate descent based message passing algorithms} due to the kink in the maximum in Eq.~\ref{eq:majorization}. To guarantee convergence to a critical point, the visibility consistency constraint Eq.~\ref{tightness_condition} can be smoothed slightly \eg using \cite{nesterov2005smooth}. Again, in our experiments this was unnecessary to achieve good convergence behavior.

\section{Experiments}
\label{sec:experiments}

Before we discuss the experiments, we describe the input data and state the costs $c_{ri}^{\ell}$ used for the ray potentials.

\subsection{Input Data}
We are using our approach for two different tasks: standard dense 3D reconstruction and dense semantic 3D reconstruction. In both cases, the initial input is a set of images with associated camera poses. Those camera poses are either provided with the dataset (as in the Middlebury Benchmark \cite{middlebury} or in the Thin Road Sign dataset \cite{UB13}) or computed via structure from motion algorithm \cite{Cohen12} (as in the semantic reconstruction experiments). We computed the depth maps using plane sweeping stereo for Middlebury Benchmark and semantic reconstruction datasets, while utilizing those already provided with the dataset for the experiment with Thin Road Sign. The patch similarity measure for stereo matching was zero-mean normalized cross correlation. For the dense semantic 3D reconstruction experiments, we computed per-pixel semantic labels using \cite{hierarchicalcrf}, trained on the datasets from \cite{BrostowSFC:eccv08}, \cite{ShottonWRC06} and \cite{Hane13}. 

\subsection{Ray Potential Costs}
In case of a two-label problem, there exists only one single label $\ell \in \mathcal{L}$. This allows us to directly insert the visibility consistency, Eq.~\ref{eq:visibilityConsistency}, into the objective. In this case the majorization can directly be done on the objective instead of the visibility constraint, leading to a more compact optimization problem with a smaller memory footprint. Like \cite{Hane13}, we assume exponential noise on the depth maps and define the assignments to the costs $c_{ri}^{\ell}$, given the position of the depth measurement along the ray $r$ as $i'$ as
\begin{equation}
 c_{ri}^{\ell} :=  \min\{0, \lambda|i - i'| - K\}.
\end{equation}
The parameters $\lambda \geq 0$ and $K$ are chosen such that the potential captures the uncertainty of the depth measurement.

For the multi-class case we also assume exponential noise on the depth data and independence between the depth measurement and the semantic measurement. Therefore the combined costs read as
\begin{equation}
 c_{ri}^{\ell} :=  \min\{0, \lambda|i - i'| - K\} + \sigma^{\ell},
\end{equation}
with $\sigma^{\ell}$ being the response of the semantic classifier for the respective pixel. This is the same potential that \cite{Hane13} approximates with unary potentials. 

\def \wrs {0.17}

   \begin{figure*}
   \centering
 \begin{tabular}{ccccc}
 \includegraphics[width=0.06\linewidth]{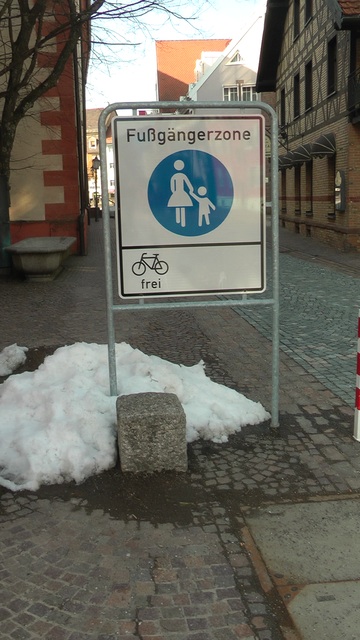} &
 \includegraphics[width=\wrs\linewidth]{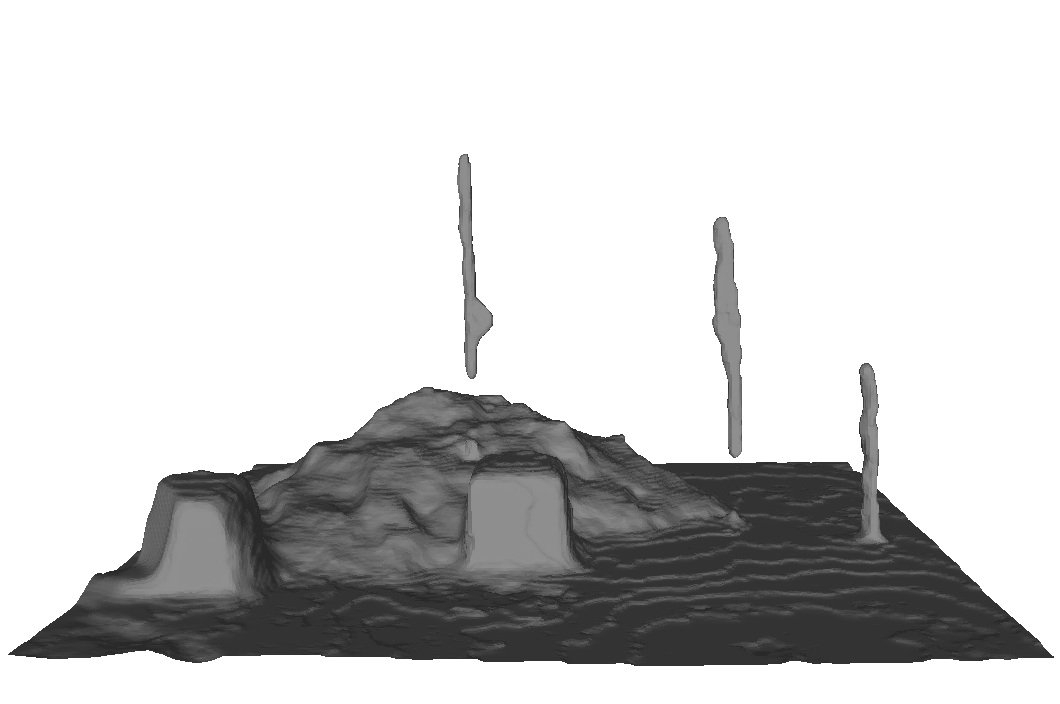} &
 \includegraphics[width=\wrs\linewidth]{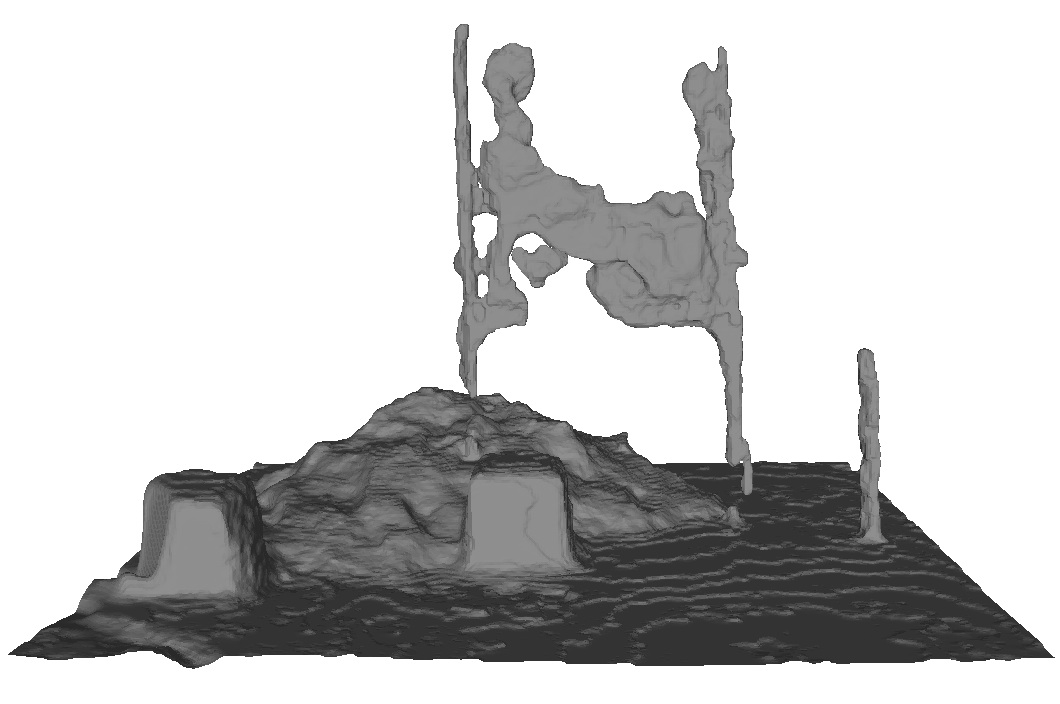} &
 \includegraphics[width=\wrs\linewidth]{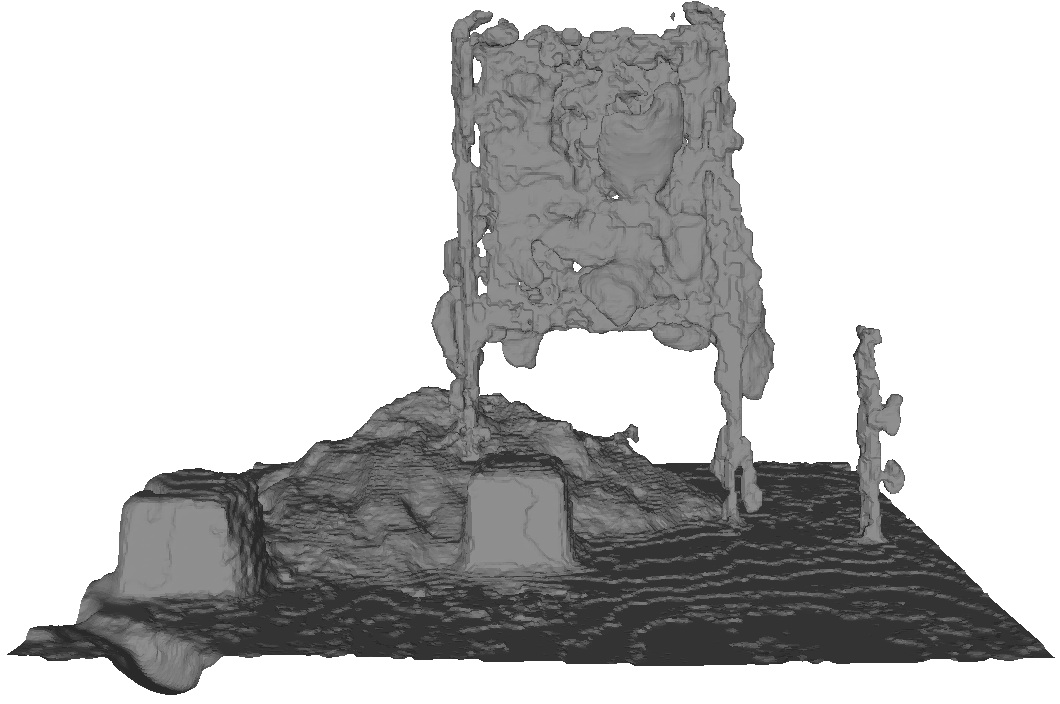} &
 \includegraphics[width=\wrs\linewidth]{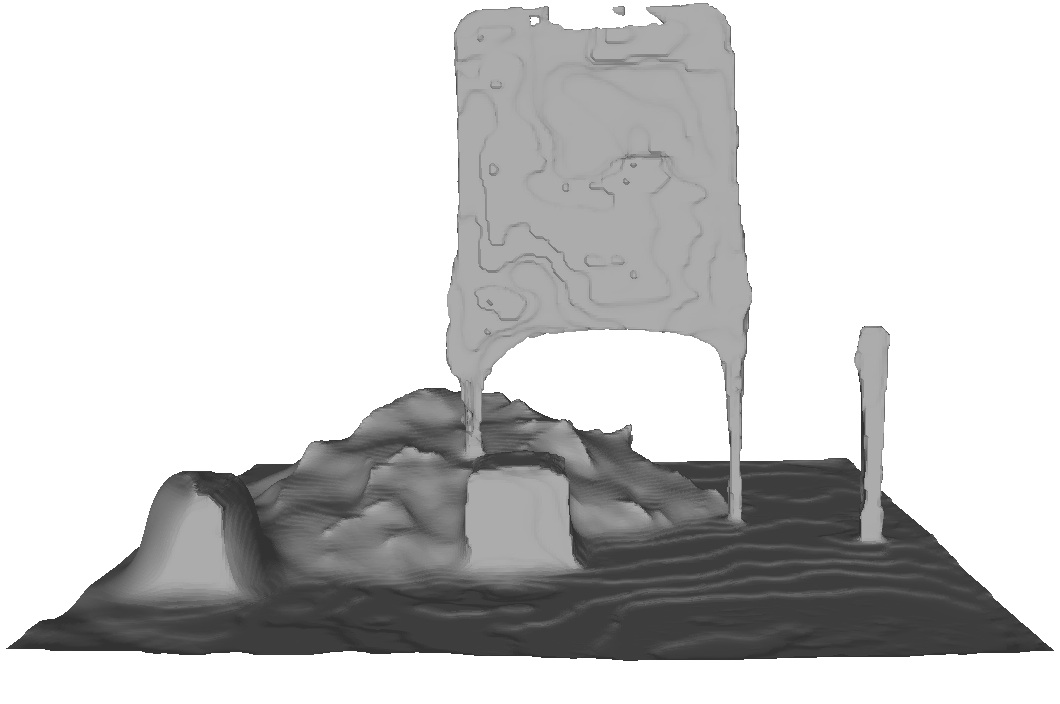} \vspace{-0.20cm} \\
 \includegraphics[width=0.06\linewidth]{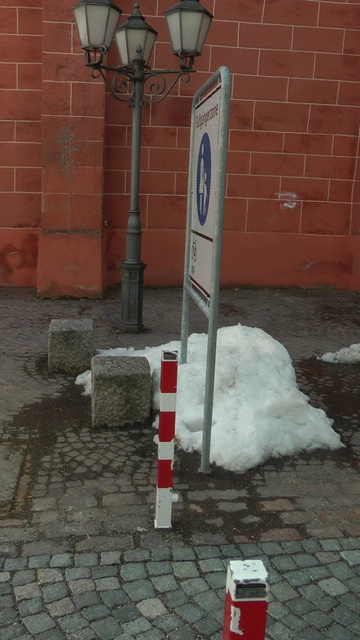} &
 \includegraphics[width=\wrs\linewidth]{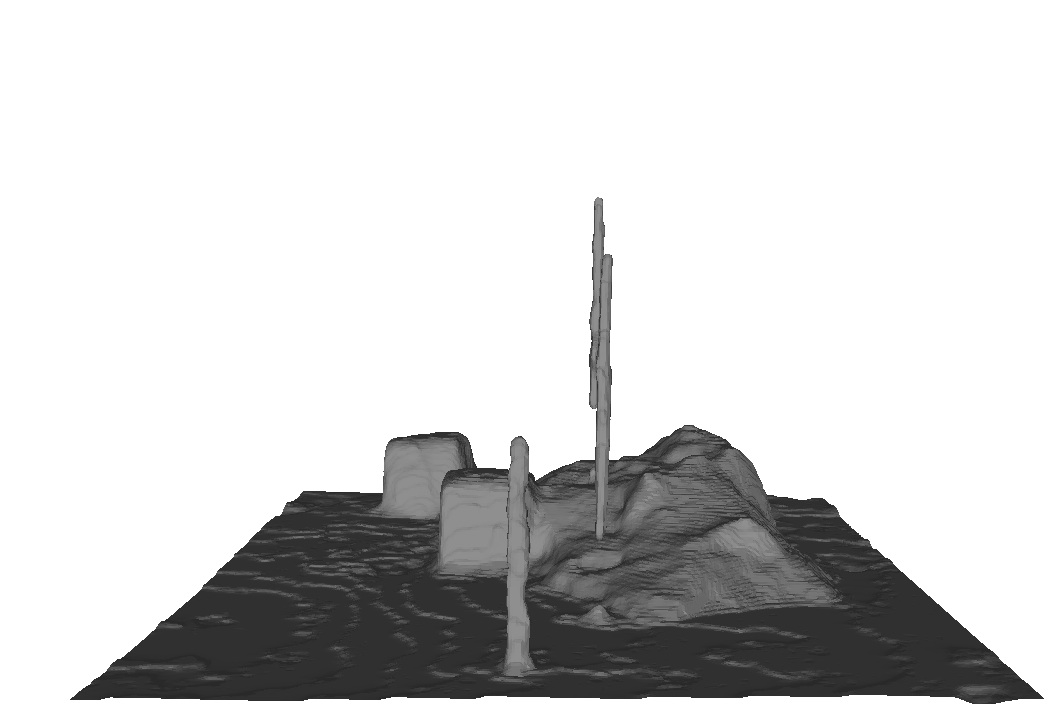} &
 \includegraphics[width=\wrs\linewidth]{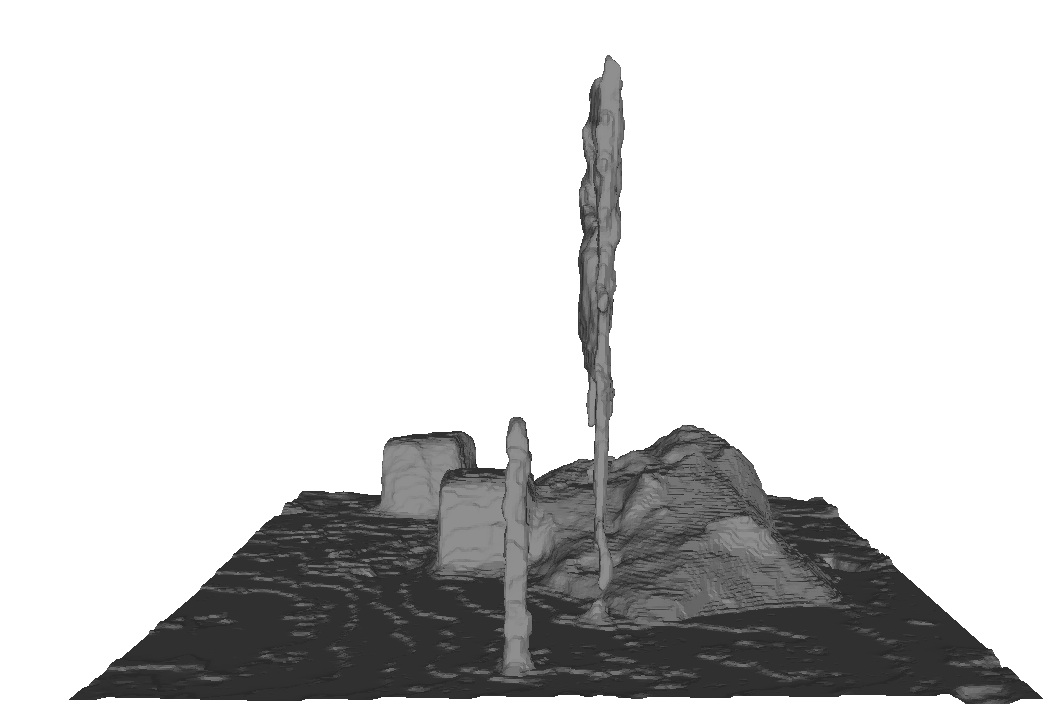} &
 \includegraphics[width=\wrs\linewidth]{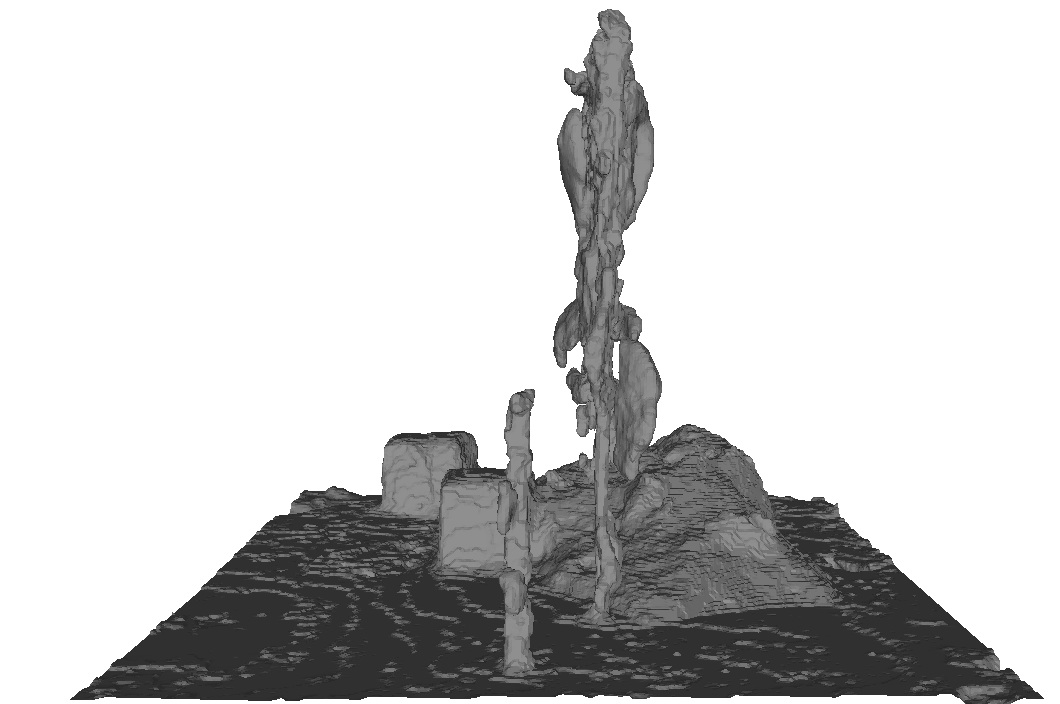} &
 \includegraphics[width=\wrs\linewidth]{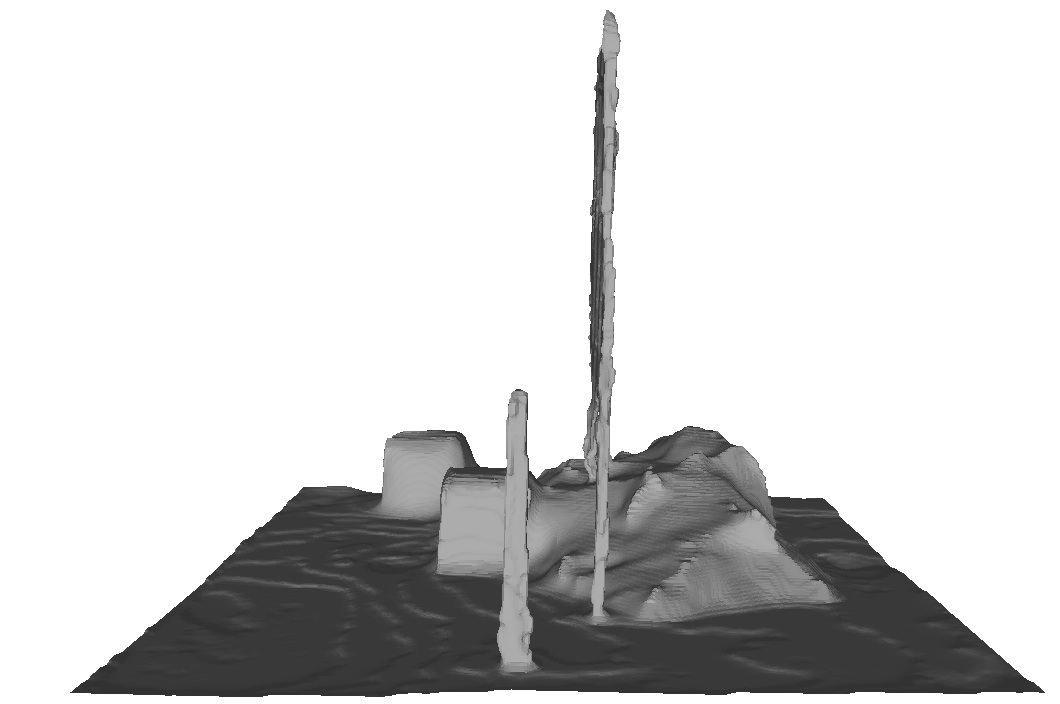}  \\ 
Example Images & TV-Flux (high) & TV-Flux (medium) & TV-Flux (low) & Our Method 
 \end{tabular}
 \vspace*{-0.125cm}
 \caption{\label{fig:road_sign_reconstructions} Reconstructions of the street sign dataset from \cite{UB13} using the TV-Flux fusion from \cite{zach2008fast} with three different smoothness settings (high/medium/low), and our proposed method.
 }
 \end{figure*}
 \begin{figure*}
 \centering
 \begin{tabular}{cccc}
   \includegraphics[width=0.15\linewidth]{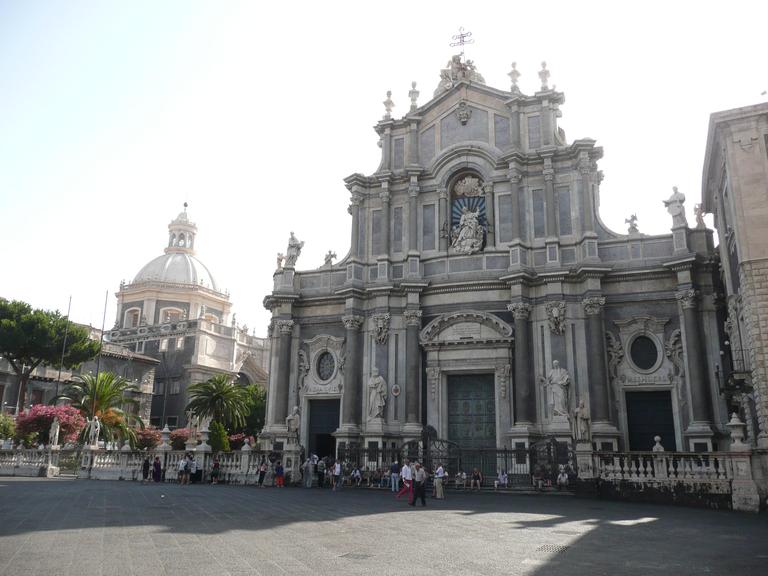} &
 \includegraphics[width=0.25\linewidth]{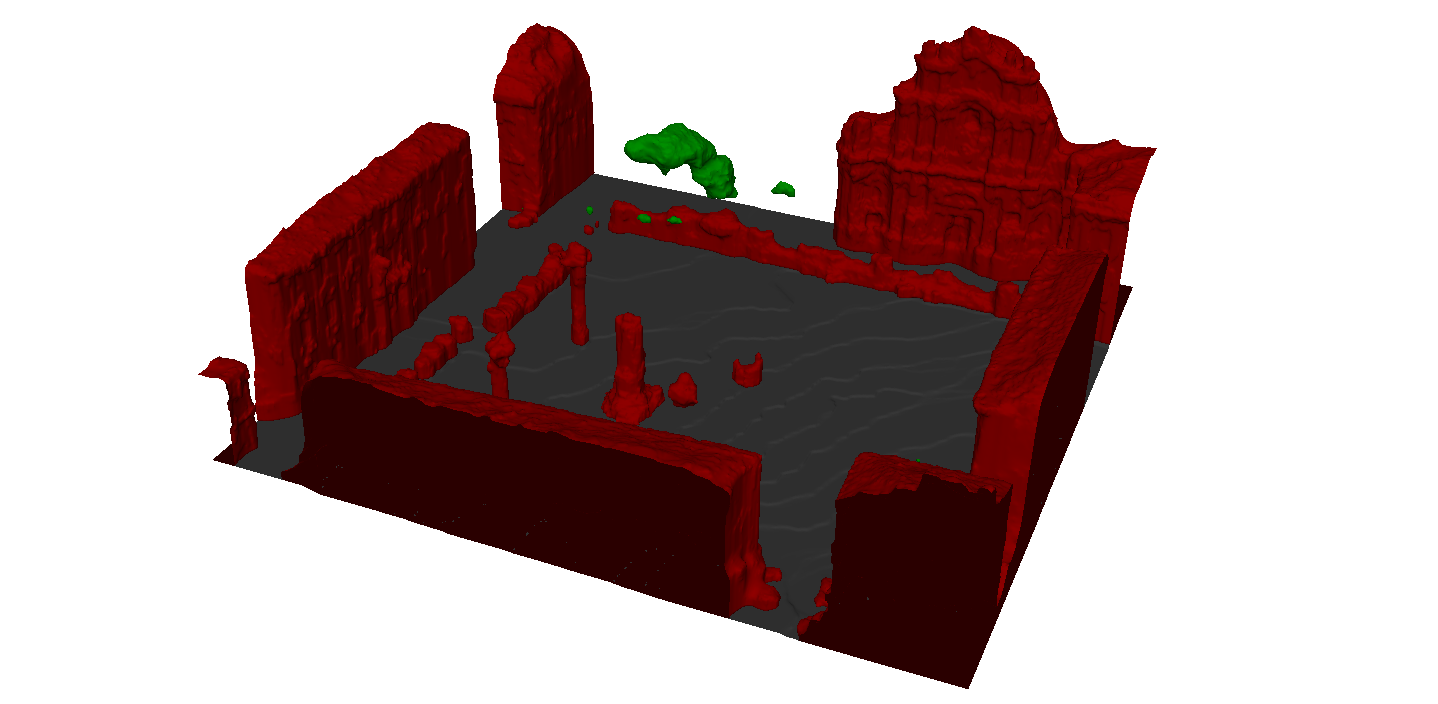} &
  \includegraphics[width=0.25\linewidth]{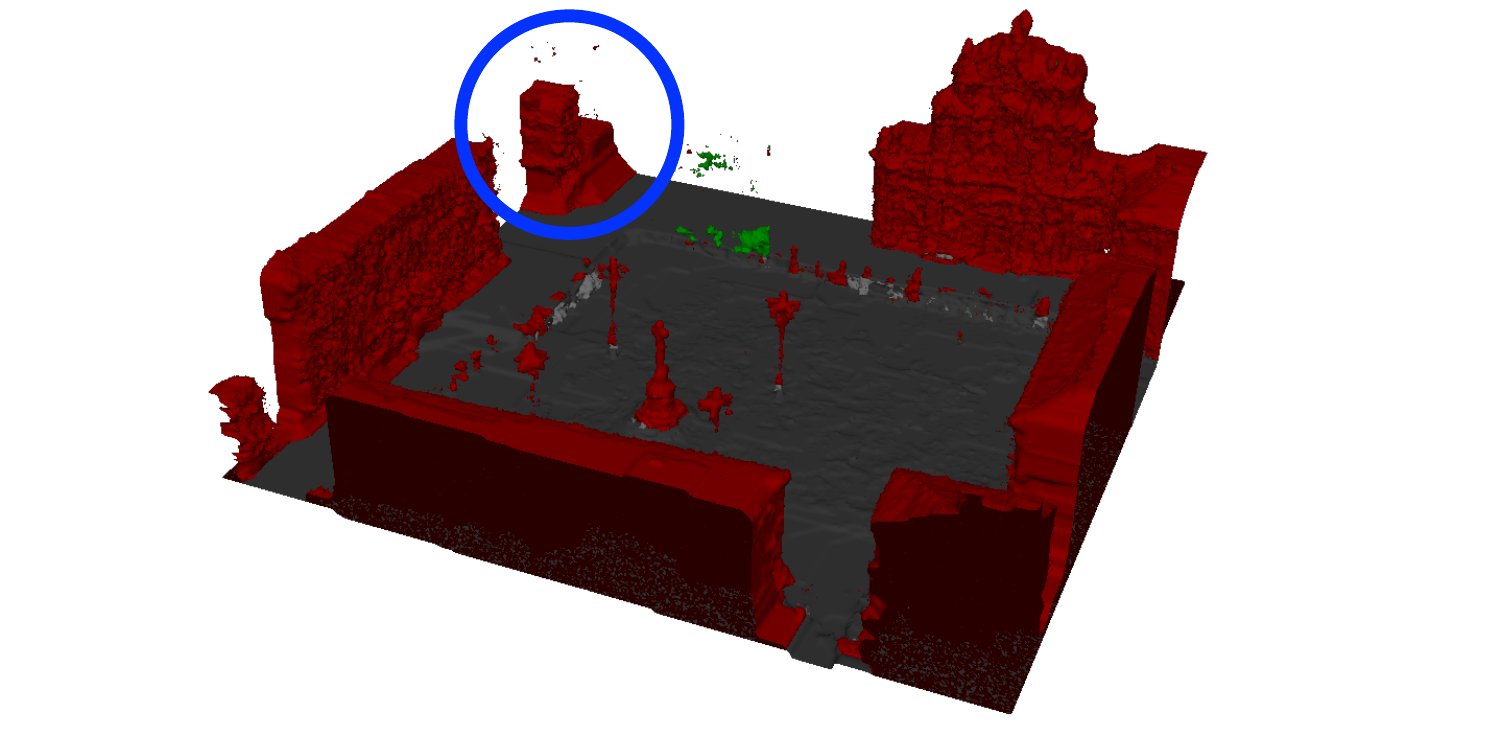} &
   \includegraphics[width=0.25\linewidth]{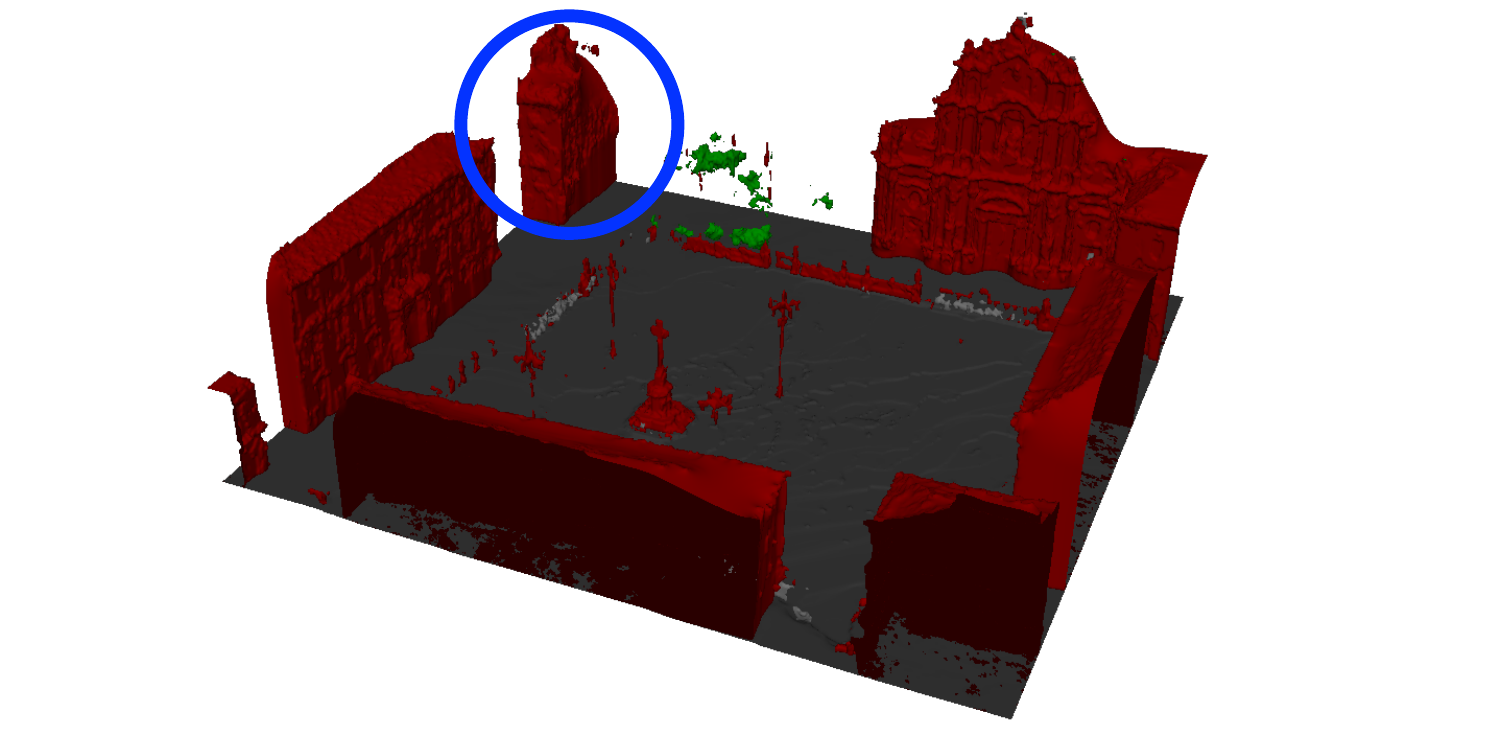} \vspace{-0.25cm} \\
   \includegraphics[width=0.15\linewidth,height=1.8cm]{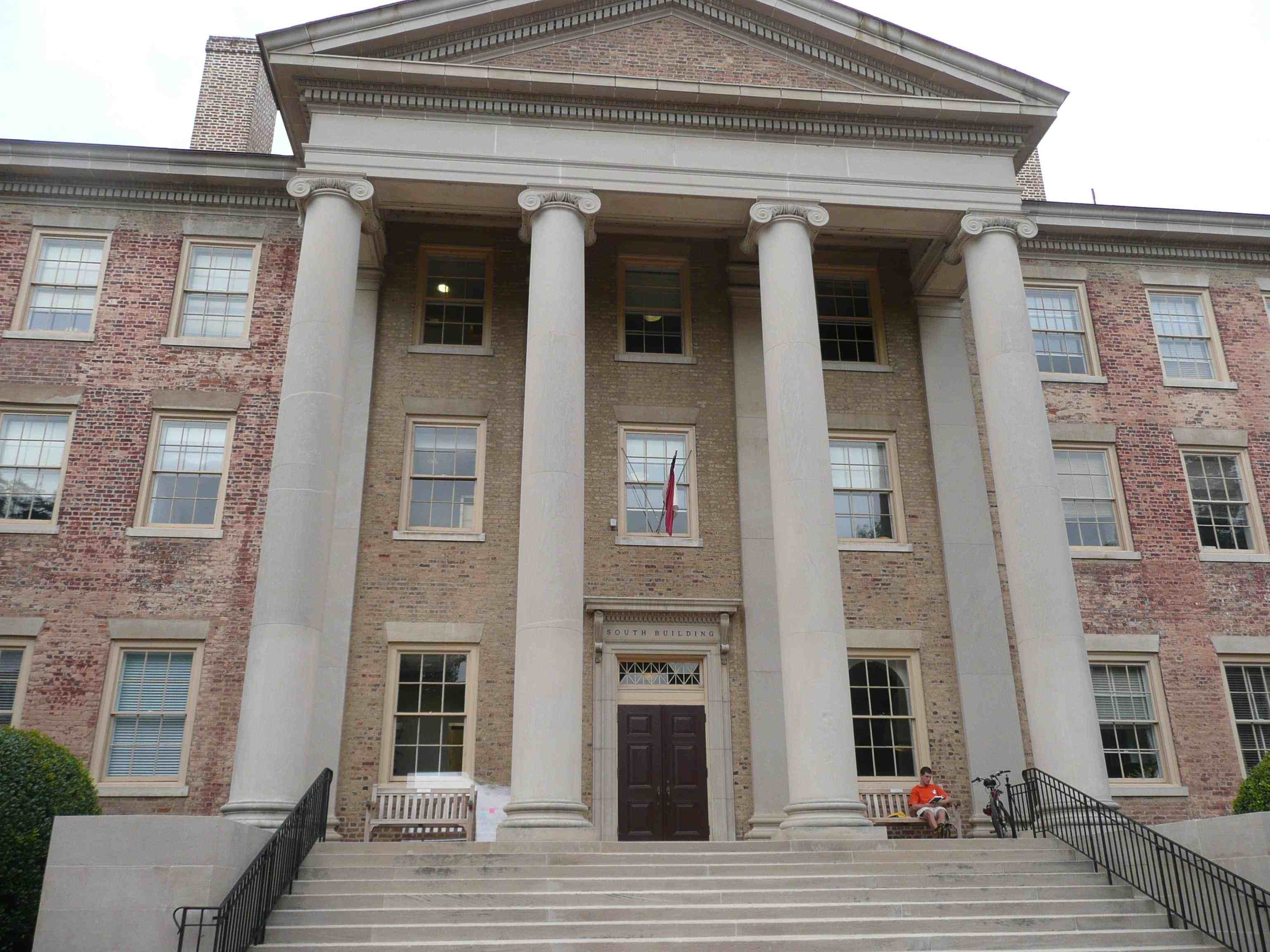} &
 \includegraphics[width=0.25\linewidth]{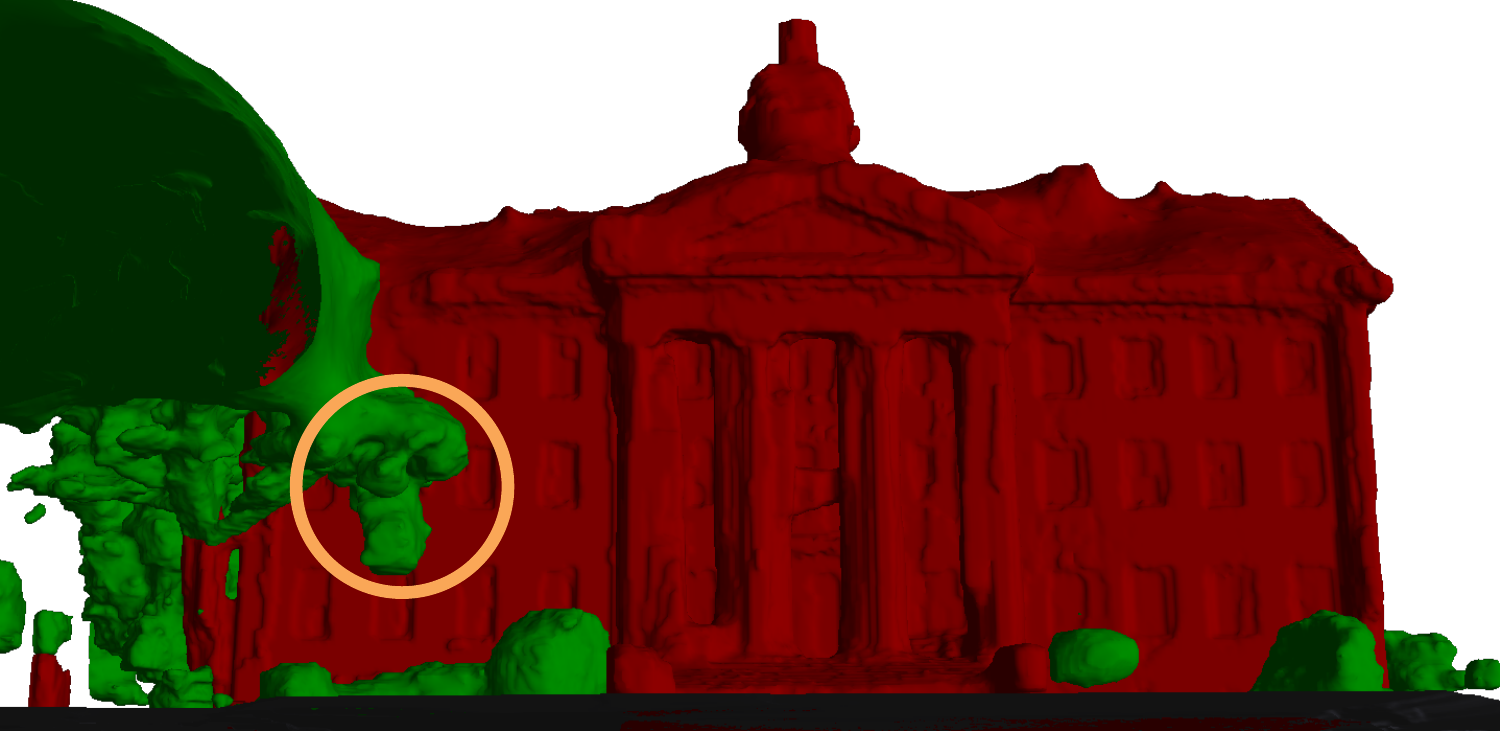} &
  \includegraphics[width=0.25\linewidth]{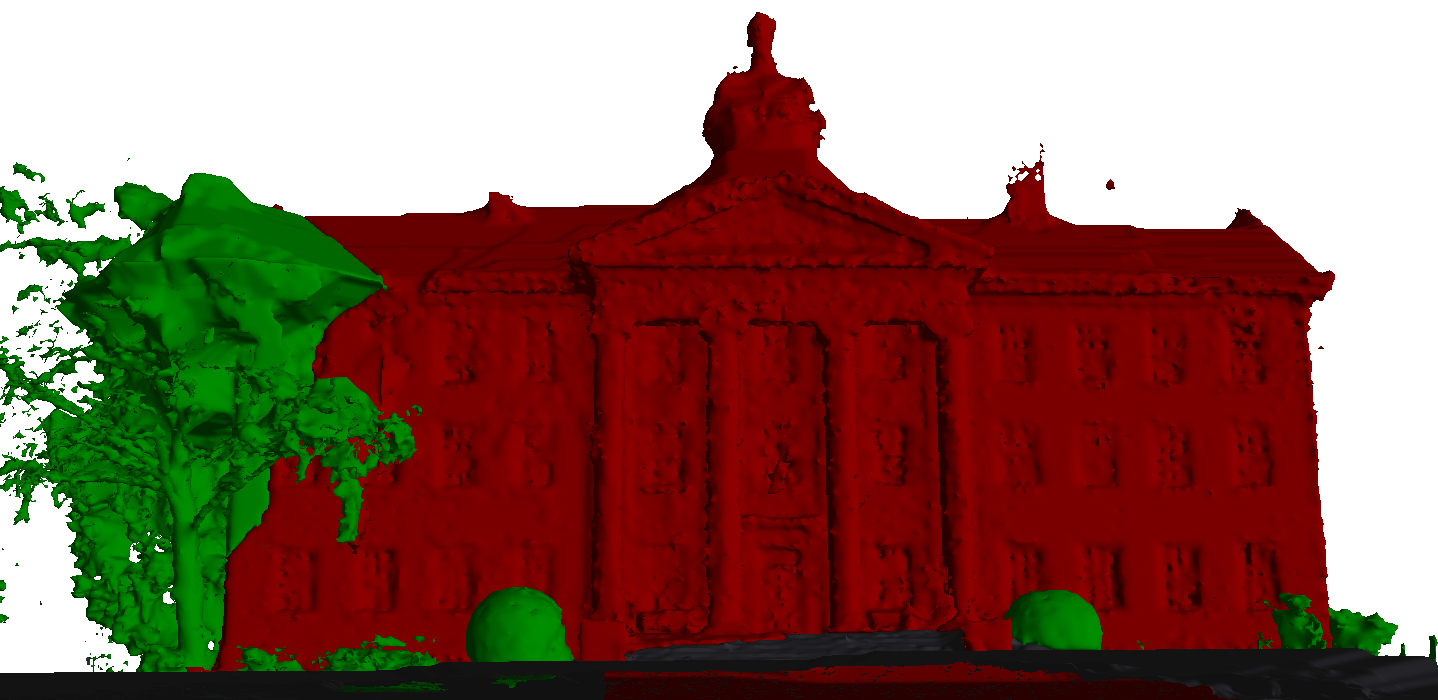} &
   \includegraphics[width=0.25\linewidth]{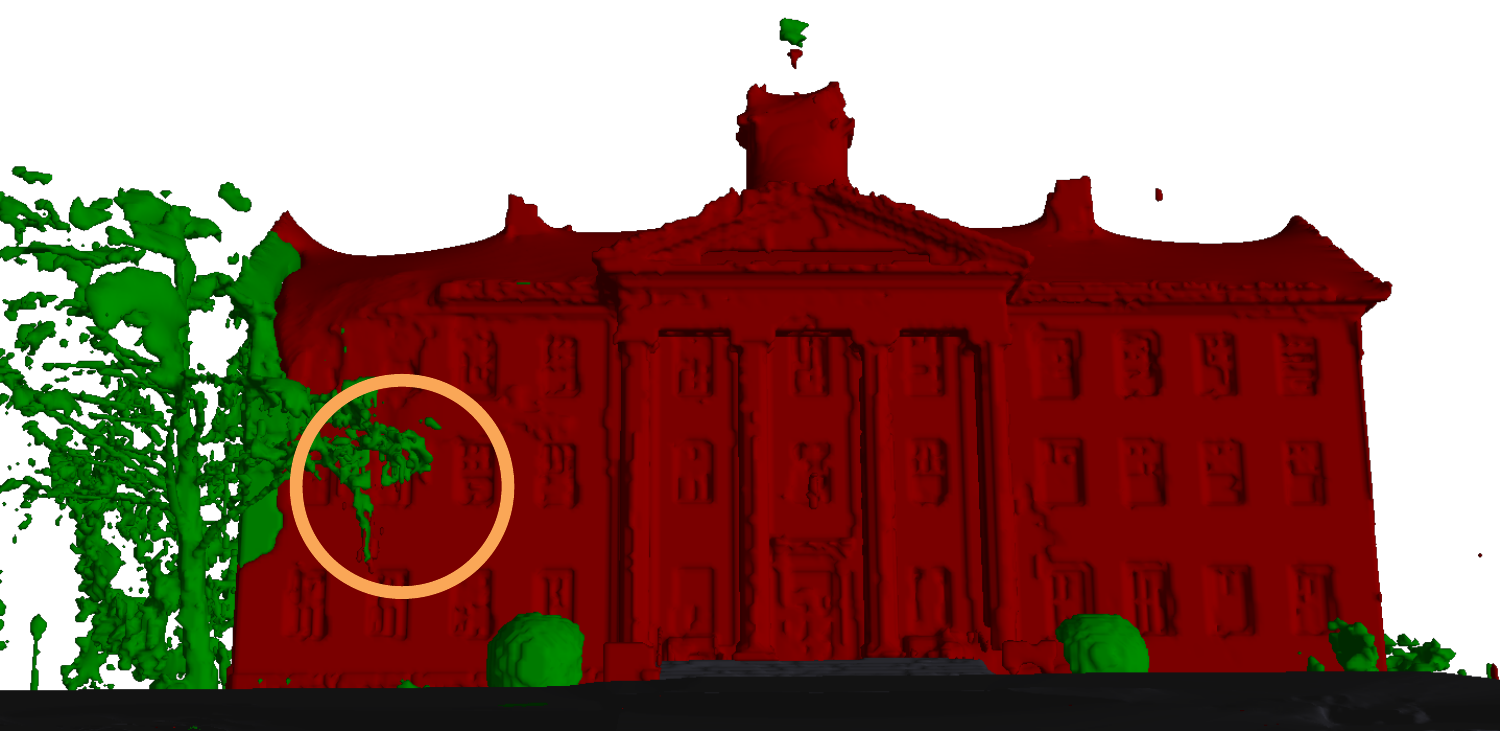} \vspace{-0.05cm} \\
   \includegraphics[width=0.15\linewidth,height=1.8cm]{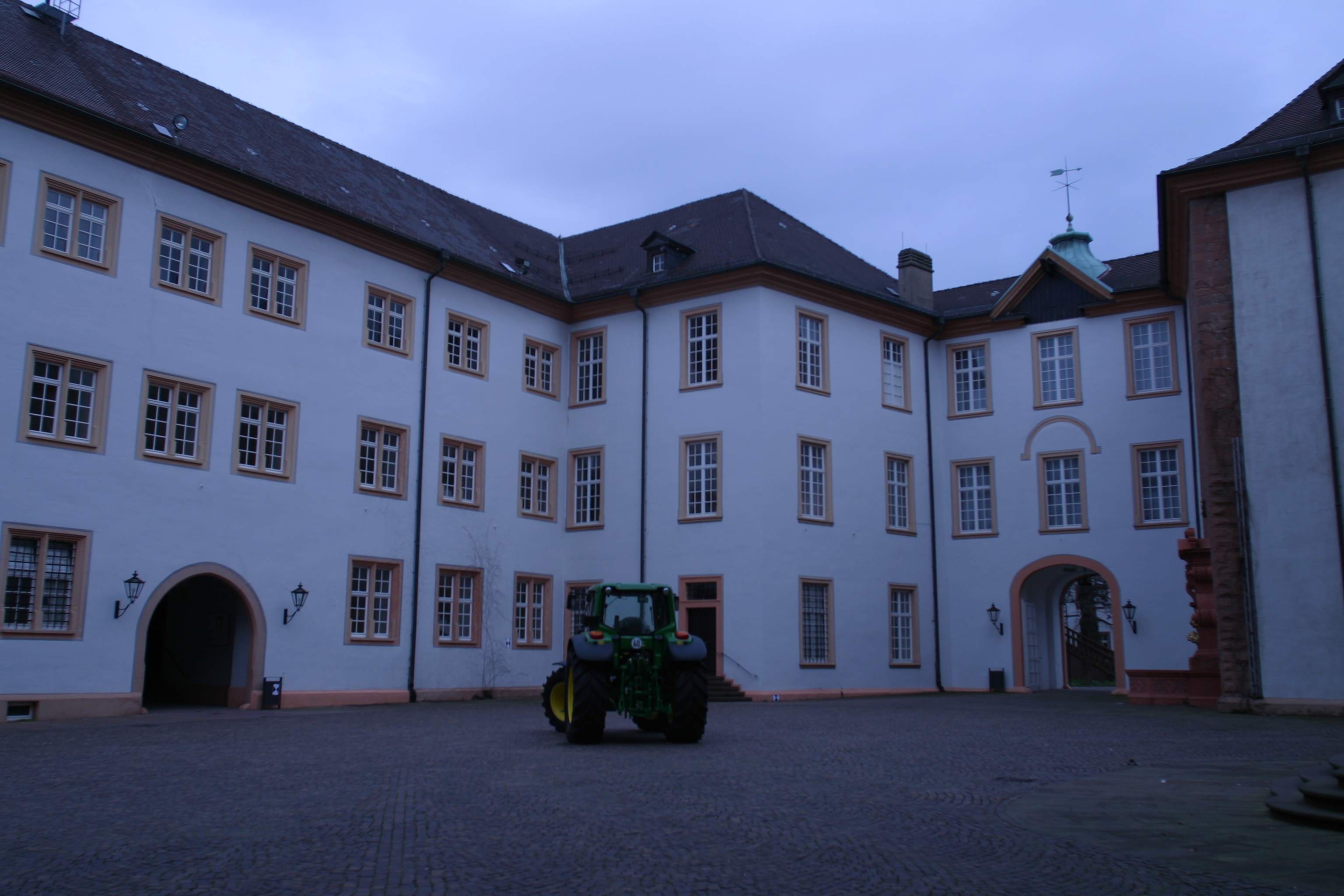} &
 \includegraphics[width=0.25\linewidth]{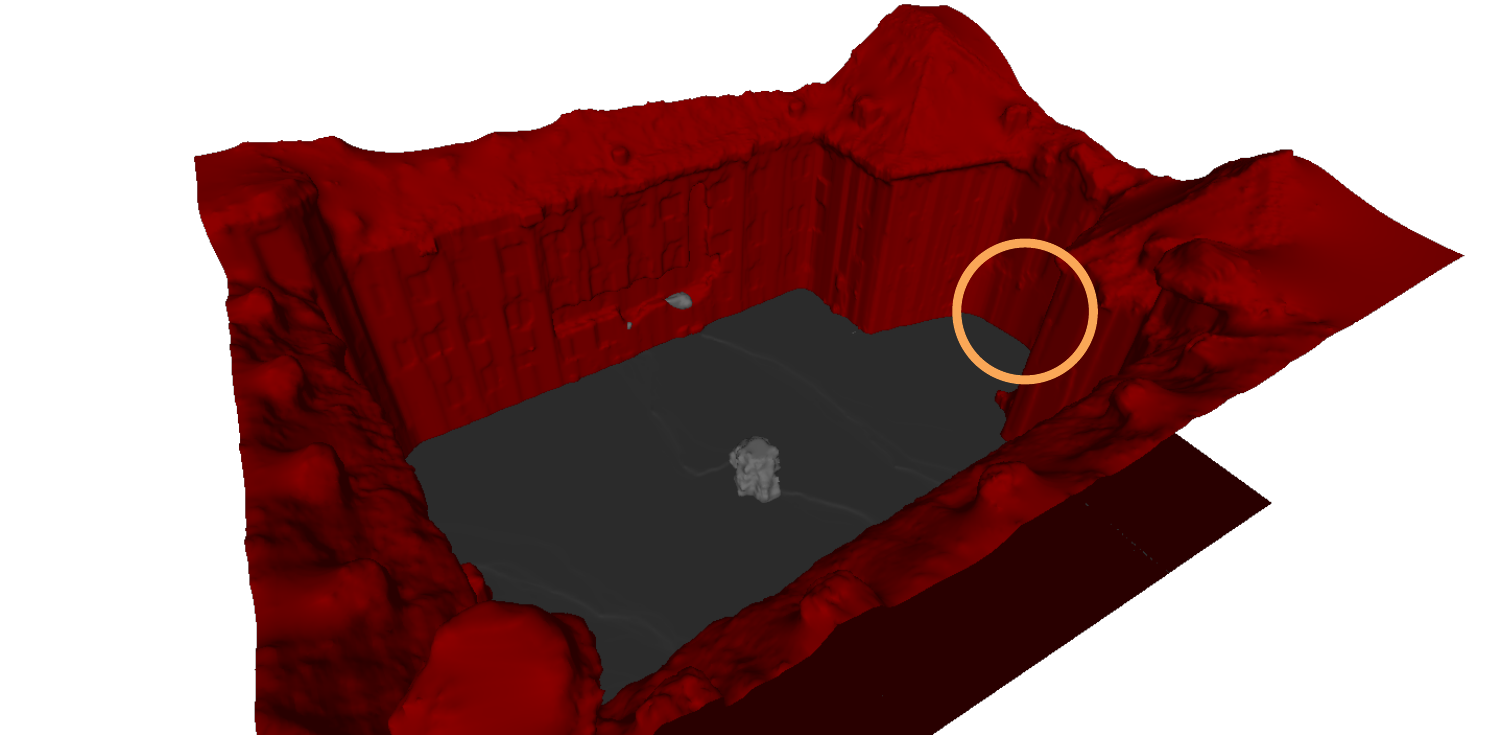} &
  \includegraphics[width=0.25\linewidth]{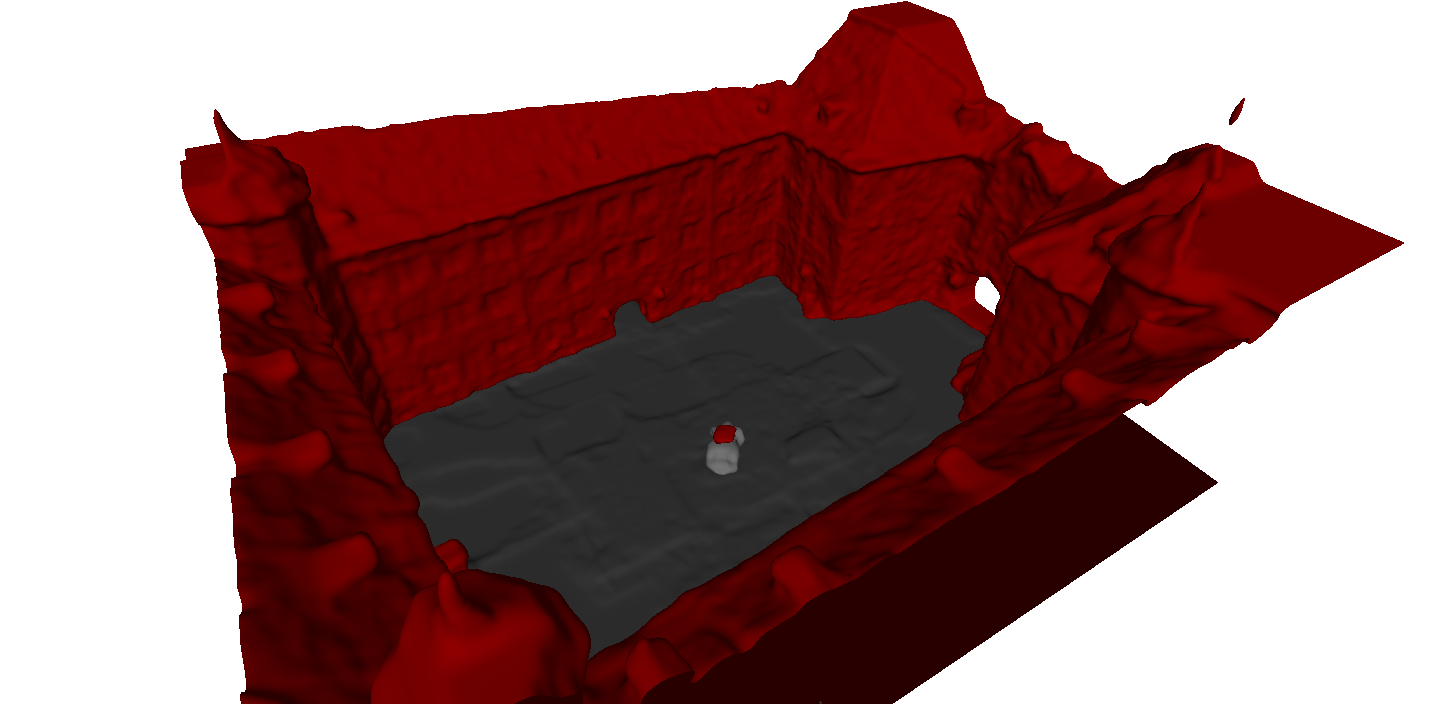} &
   \includegraphics[width=0.25\linewidth]{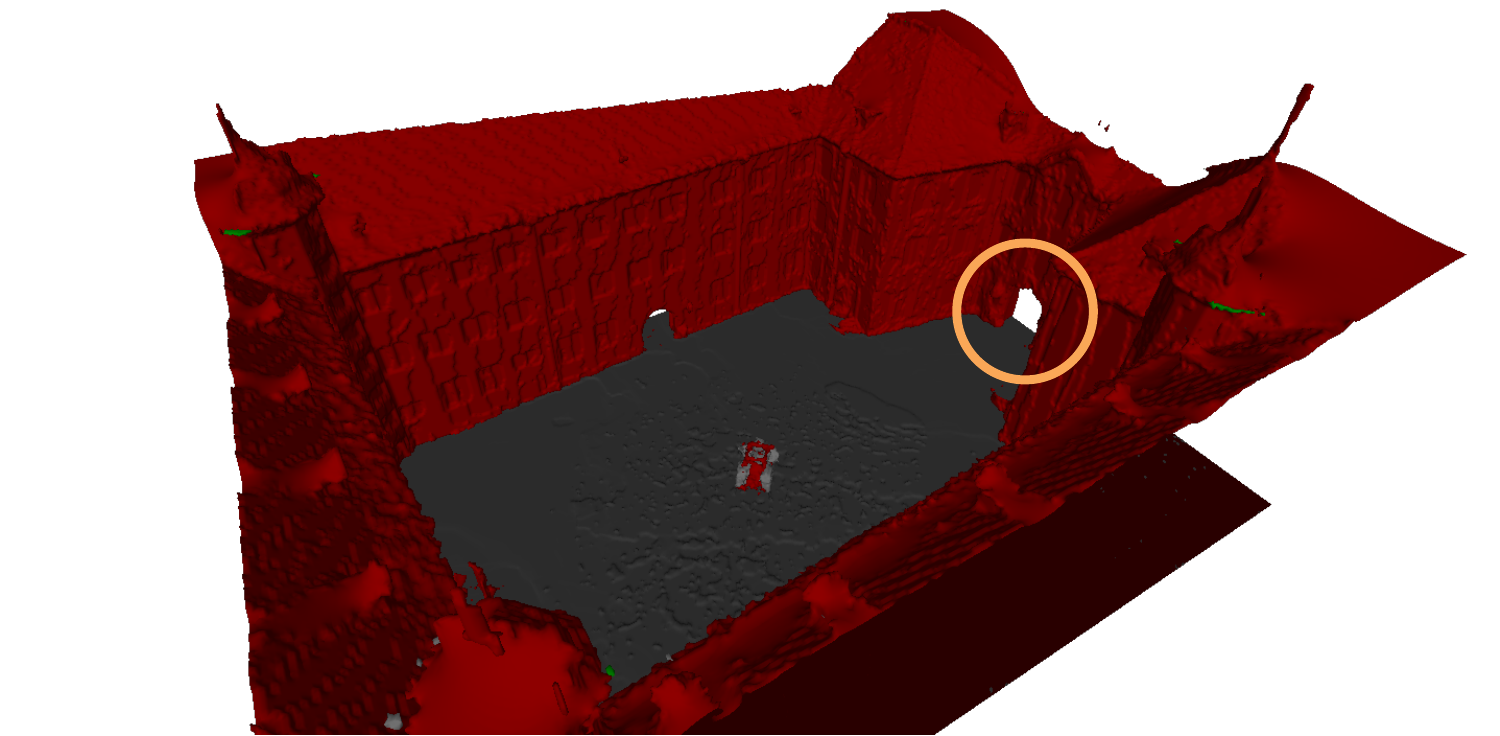} \vspace{0.025cm} \\
 Input Image & H{\"a}ne et al.\ 2013 \cite{Hane13} & Savinov et al. 2015 \cite{Savinov15} & Proposed Method   
 \end{tabular}
 \vspace*{-0.125cm}
 \caption{\label{fig:reconstructions} Semantic 3D Reconstructions: we improve in \blue{weakly observed areas} and resolve \yellow{unary potential artifacts} at the same time. Five semantic labels are used: ground, building, vegetation, clutter, free space.  \vspace*{-0.15cm}}
 \end{figure*}

\subsection{Middlebury Multi-View Stereo Benchmark}
We evaluate our method for dense 3D reconstuction on the Middlebury benchmark~\cite{middlebury}. We ran our algorithm on all 6 datasets (using the same parameters). Two quantitative measures are defined in this benchmark paper: accuracy (Acc) and completeness (Comp). In terms of accuracy our algorithm sets a new state-of-the-art for the Dino Full and Dino Ring datasets (\cf Fig.~\ref{fig:middlebury_reconstructions}). An actual ranking of the benchmark is difficult because there is no default, commonly accepted, way to combine the two measures. Taking into account both measures we are close to the state-of-the-art on all datasets (results can be found online~\footnote{\url{http://vision.middlebury.edu/mview/eval/}}).

\subsection{Street Sign Dataset}
A challenging case for volumetric 3D reconstruction are thin objects. When approximating the data term, which is naturally given as a ray potential in the 2.5D input data, by unary or pairwise potentials the data terms from both sides are prone to cancel out. Similarly, when using visual hulls a slight misalignment of the two sides might generate an empty visual hull. These are the reasons why thin objects are considered to be a hard case in volumetric 3D reconstruction. We evaluate the performance of our algorithm for such objects on the street sign dataset from \cite{UB13}. The dataset consists $50$ images of a street sign with corresponding depth maps. As depicted in Fig~\ref{fig:road_sign_reconstructions} the thin surface does not pose any problem to our method, thanks to an accurate representation of the input data in the optimization problem. To illustrate the result obtained with a standard volumetric 3D reconstruction algorithm we ran our implementation of the TV-Flux fusion from \cite{zach2008fast} on the same data. Note that this dataset is particularly hard because the two sides actually interpenetrate as detailed in \cite{UB13}.

\subsection{Multi-Label Experiments}
\vspace*{-0.125cm}
We evaluate our formulation for dense semantic 3D reconstruction on several real-world datasets. We show our results side-by-side with the method of \cite{Hane13} and \cite{Savinov15} in Figs.~\ref{fig:reconstructions} and \ref{fig:closeUps}. Our method uses the same smoothness prior as \cite{Hane13}. For all the datasets we observe that the approximation of the data cost with a unary potential in \cite{Hane13} artificially fattens corners and thin objects (\eg pillars or tree branches). In the close-ups (\cf Fig.~\ref{fig:closeUps}) we see that such a data term recovers significantly less surface detail with respect to our proposed method. This problem has been addressed in \cite{Savinov15}, but their discrete graph-based approach suffers from metrication artifacts, cannot be combined with the class-specific anisotropic smoothness prior and does not lead to smooth surfaces (\cf Fig.~\ref{fig:closeUps}). Moreover, their coarse-to-fine scheme produces artifacts in the reconstructions. Our approach takes the best of both worlds, the ray potential part ensures an accurate position of the observed surfaces, while the anisotropic smoothness prior faithfully handles weakly observed areas.

\section{Conclusion}
\vspace*{-0.125cm}
In this paper we proposed an approach for using ray potentials together with continuously inspired surface regularization. We demonstrated that a direct convex relaxation is too weak to be used in practice. We resolved this issue by adding a non-convex constraint to the formulation. Further, we detailed an optimization strategy and gave an extensive evaluation on two-label and multi-label datasets. Our algorithm allows for a general multi-label ray potential, at the same time it achieves volumetric 3D reconstruction with high accuracy. In semantic 3D reconstruction we are able to overcome limitations of earlier methods.

\textbf{Acknowledgements:}
We thank Ian Cherabier for providing code for Lemma \ref{lm:regularizerFeasibility}. This work is partially funded by the Swiss National Science Foundation projects 157101 and 163910. \v Lubor Ladick\'y is funded by the Max Planck Center for Learning Systems Fellowship.

{\small
\bibliographystyle{ieee}
\bibliography{ray_cvpr2016}

\begin{thebibliography}{10}\itemsep=-1pt

\bibitem{BrostowSFC:eccv08}
G.~J. Brostow, J.~Shotton, J.~Fauqueur, and R.~Cipolla.
\newblock Segmentation and recognition using structure from motion point
  clouds.
\newblock In {\em European Conference on Computer Vision}, 2008.

\bibitem{chambolle2012convex}
A.~Chambolle, D.~Cremers, and T.~Pock.
\newblock A convex approach to minimal partitions.
\newblock {\em SIAM Journal on Imaging Sciences}, 2012.

\bibitem{Cohen12}
A.~Cohen, C.~Zach, S.~N. Sinha, and M.~Pollefeys.
\newblock Discovering and exploiting 3{D} symmetries in structure from motion.
\newblock In {\em Conference on Computer Vision and Pattern Recognition}, 2012.

\bibitem{cremers2011multiview}
D.~Cremers and K.~Kolev.
\newblock Multiview stereo and silhouette consistency via convex functionals
  over convex domains.
\newblock {\em Transactions on Pattern Analysis and Machine Intelligence},
  2011.

\bibitem{curless1996volumetric}
B.~Curless and M.~Levoy.
\newblock A volumetric method for building complex models from range images.
\newblock In {\em Conference on Computer graphics and interactive techniques},
  1996.

\bibitem{delaunoy2011optimizing}
A.~Delaunoy and E.~Prados.
\newblock Gradient flows for optimizing triangular mesh-based surfaces:
  Applications to {3D} reconstruction problems dealing with visibility.
\newblock {\em International Journal of Computer Vision (IJCV)}, 2011.

\bibitem{duchi2008efficient}
J.~Duchi, S.~Shalev-Shwartz, Y.~Singer, and T.~Chandra.
\newblock Efficient projections onto the $\ell_1$-ball for learning in high
  dimensions.
\newblock In {\em International conference on Machine learning (ICML)}, 2008.

\bibitem{pmvs}
Y.~Furukawa and J.~Ponce.
\newblock Accurate, dense, and robust multi-view stereopsis.
\newblock {\em IEEE Trans. on Pattern Analysis and Machine Intelligence},
  32(8):1362--1376, 2010.

\bibitem{Galliani}
S.~Galliani, K.~Lasinger, and K.~Schindler.
\newblock Massively parallel multiview stereopsis by surface normal diffusion.
\newblock In {\em International Conference on Computer Vision}, 2015.

\bibitem{gargallo2007occupancy}
P.~Gargallo, P.~Sturm, and S.~Pujades.
\newblock An occupancy - depth generative model of multi-view images.
\newblock In {\em Asian Conference on Computer Vision (ACCV)}. 2007.

\bibitem{Hane13}
C.~H{\"a}ne, C.~Zach, A.~Cohen, R.~Angst, and M.~Pollefeys.
\newblock Joint {3}{D} scene reconstruction and class segmentation.
\newblock In {\em Conference on Computer Vision and Pattern Recognition}, 2013.

\bibitem{Hernandez07}
C.~Hernandez, G.~Vogiatzis, and R.~Cipolla.
\newblock Probabilistic visibility for multi-view stereo.
\newblock In {\em Conference on Computer Vision and Pattern Recognition}, 2007.

\bibitem{klodt2008experimental}
M.~Klodt, T.~Schoenemann, K.~Kolev, M.~Schikora, and D.~Cremers.
\newblock An experimental comparison of discrete and continuous shape
  optimization methods.
\newblock In {\em Computer Vision--ECCV 2008}, pages 332--345. Springer, 2008.

\bibitem{kolev2008integration}
K.~Kolev and D.~Cremers.
\newblock Integration of multiview stereo and silhouettes via convex
  functionals on convex domains.
\newblock In {\em European Conference on Computer Vision (ECCV)}, 2008.

\bibitem{Kolev07propagatedphotoconsistency}
K.~Kolev, M.~Klodt, T.~Brox, and D.~Cremers.
\newblock Propagated photoconsistency and convexity in variational multiview
  {3D} reconstruction.
\newblock In {\em IN WORKSHOP ON}, 2007.

\bibitem{kolmogorov2002multi}
V.~Kolmogorov and R.~Zabih.
\newblock Multi-camera scene reconstruction via graph cuts.
\newblock In {\em European Conference on Computer Vision (ECCV)}. 2002.

\bibitem{kolmogorov2003generalized}
V.~Kolmogorov, R.~Zabih, and S.~Gortler.
\newblock Generalized multi-camera scene reconstruction using graph cuts.
\newblock In {\em Energy Minimization Methods in Computer Vision and Pattern
  Recognition (EMMCVPR)}, 2003.

\bibitem{labatut2007delaunay}
P.~Labatut, J.-P. Pons, and R.~Keriven.
\newblock Efficient multi-view reconstruction of large-scale scenes using
  interest points, delaunay triangulation and graph cuts.
\newblock In {\em IEEE International Conference on Computer Vision}, 2007.

\bibitem{hierarchicalcrf}
L.~Ladicky, C.~Russell, P.~Kohli, and P.~H.~S. Torr.
\newblock Associative hierarchical {CRF}s for object class image segmentation.
\newblock In {\em International Conference on Computer Vision (ICCV)}, 2009.

\bibitem{lange2000optimization}
K.~Lange, D.~R. Hunter, and I.~Yang.
\newblock Optimization transfer using surrogate objective functions.
\newblock {\em Journal of computational and graphical statistics}, 2000.

\bibitem{LempitskyB07}
V.~S. Lempitsky and Y.~Boykov.
\newblock Global optimization for shape fitting.
\newblock In {\em Conference on Computer Vision and Pattern Recognition}, 2007.

\bibitem{DCV}
Z.~Li, K.~Wang, W.~Zuo, D.~Meng, and L.~Zhang.
\newblock Detail-preserving and content-aware variational multi-view stereo
  reconstruction.
\newblock {\em arXiv preprint arXiv:1505.00389}, 2015.

\bibitem{liu2010ray}
S.~Liu and D.~B. Cooper.
\newblock Ray {Markov} random fields for image-based {3D} modeling: Model and
  efficient inference.
\newblock In {\em Conference on Computer Vision and Pattern Recognition}, 2010.

\bibitem{liu2014statistical}
S.~Liu and D.~B. Cooper.
\newblock Statistical inverse ray tracing for image-based {3D} modeling.
\newblock {\em Transactions on Pattern Analysis and Machine Intelligence
  (TPAMI)}, 2014.

\bibitem{nesterov2005smooth}
Y.~Nesterov.
\newblock Smooth minimization of non-smooth functions.
\newblock {\em Mathematical programming}, 2005.

\bibitem{pock2011diagonal}
T.~Pock and A.~Chambolle.
\newblock Diagonal preconditioning for first order primal-dual algorithms in
  convex optimization.
\newblock In {\em 2011 IEEE International Conference on Computer Vision
  (ICCV)}, 2011.

\bibitem{pollard2007change}
T.~Pollard and J.~L. Mundy.
\newblock Change detection in a 3-d world.
\newblock In {\em Conference on Computer Vision and Pattern Recognition
  (CVPR)}, 2007.

\bibitem{Savinov15}
N.~Savinov, L.~Ladicky, C.~H{\"a}ne, and M.~Pollefeys.
\newblock Discrete optimization of ray potentials for semantic {3D}
  reconstruction.
\newblock In {\em Conference on Computer Vision and Pattern Recognition}, 2015.

\bibitem{middlebury}
S.~Seitz, B.~Curless, J.~Diebel, D.~Scharstein, and R.~Szeliski.
\newblock A comparison and evaluation of multi-view stereo reconstruction
  algorithms.
\newblock In {\em IEEE Computer Society Conference on Computer Vision and
  Pattern Recognition (CVPR'2006)}, volume~1, pages 519--526. IEEE Computer
  Society, June 2006.

\bibitem{ShottonWRC06}
J.~Shotton, J.~Winn, C.~Rother, and A.~Criminisi.
\newblock {\it TextonBoost}: Joint appearance, shape and context modeling for
  multi-class object recognition and segmentation.
\newblock In {\em European Conference on Computer Vision}, 2006.

\bibitem{Sinha07}
S.~N. Sinha, P.~Mordohai, and M.~Pollefeys.
\newblock Multi-view stereo via graph cuts on the dual of an adaptive
  tetrahedral mesh.
\newblock In {\em International Conference on Computer Vision}, 2007.

\bibitem{strekalovskiy2011generalized}
E.~Strekalovskiy and D.~Cremers.
\newblock Generalized ordering constraints for multilabel optimization.
\newblock In {\em International Conference on Computer Vision (ICCV)}, 2011.

\bibitem{ulusoy2015rayPotentials}
A.~O. Ulusoy, A.~Geiger, and M.~J. Black.
\newblock Towards probabilistic volumetric reconstruction using ray potentials.
\newblock In {\em International Conference on 3D Vision (3DV)}, 2015.

\bibitem{UB13}
B.~Ummenhofer and T.~Brox.
\newblock Point-based {3D} reconstruction of thin objects.
\newblock In {\em IEEE International Conference on Computer Vision (ICCV)},
  2013.

\bibitem{vogiatzis2005multi}
G.~Vogiatzis, P.~H. Torr, and R.~Cipolla.
\newblock Multi-view stereo via volumetric graph-cuts.
\newblock In {\em Conference on Computer Vision and Pattern Recognition}, 2005.

\bibitem{Vu2012}
H.-H. Vu, P.~Labatut, J.-P. Pons, and R.~Keriven.
\newblock High accuracy and visibility-consistent dense multiview stereo.
\newblock {\em Transactions on Pattern Analysis and Machine Intelligence},
  2012.

\bibitem{BMVC2014_183}
J.~Wei, B.~Resch, and H.~Lensch.
\newblock Multi-view depth map estimation with cross-view consistency.
\newblock In {\em Proceedings of the British Machine Vision Conference}. BMVA
  Press, 2014.

\bibitem{zach2008fast}
C.~Zach.
\newblock Fast and high quality fusion of depth maps.
\newblock In {\em 3D Data Processing, Visualization and Transmission}, 2008.

\bibitem{zach2008labeling}
C.~Zach, D.~Gallup, J.-M. Frahm, and M.~Niethammer.
\newblock Fast global labeling for real-time stereo using multiple plane
  sweeps.
\newblock In {\em International Workshop on Vision, Modeling and Visualization
  (VMV)}, 2008.

\bibitem{zach2014optimized}
C.~Zach, C.~Hane, and M.~Pollefeys.
\newblock What is optimized in convex relaxations for multilabel problems:
  Connecting discrete and continuously inspired {MAP} inference.
\newblock {\em IEEE Transactions on Pattern Analysis and Machine Intelligence
  (TPAMI)}, 2014.

\bibitem{zach2007globally}
C.~Zach, T.~Pock, and H.~Bischof.
\newblock A globally optimal algorithm for robust {TV-L1} range image
  integration.
\newblock In {\em International Conference on Computer Vision (ICCV)}, 2007.

\bibitem{Zhu}
Z.~Zhu, C.~Stamatopoulos, and C.~S. Fraser.
\newblock Accurate and occlusion-robust multi-view stereo.
\newblock {\em ISPRS Journal of Photogrammetry and Remote Sensing}, 109:47--61,
  2015.

\end{thebibliography}
}

\clearpage
\begin{appendices}
\section{Supplementary Material}
First, we provide the proof of Lemma~\ref{lm:regularizerFeasibility} from the main text. Then we give additional experimental evaluation which did not fit into the main text of the paper: the additional experiments are shown for semantic 3D reconstruction as well as for classic 3D reconstruction. Afterwards, we give an intuition why the convex formulation, which was introduced in Section~\ref{sec:convex} of the main text, provides a very weak solution. Eventually, we show convergence experiments for our algorithm. 

\subsection{Proof of Lemma 2}
\begin{proof}
 For readability we drop the iteration index $(n)$. First we note the following. If we fix $k$ and $s$ each $(z_{s}^{\ell m})_k$ only appears in two linear equation systems with $L$ equations.
 \begin{align}
  x_s^{\ell}  =  \sum_m \!\left(z_s^{\ell m}\right)_k, \quad x_{s+e_k}^{\ell} =  \sum_m \!\left( z_{s}^{m \ell} \right)_k \quad \forall \ell \! \in \! \LwithF
 \end{align}
 Hence, this constraints can be written in the form $A\mathbf{w_s^k} = b$ for each $k$ and $s$, where $\mathbf{w_s^k}$ is a vector containing the variables $(z_{s}^{m\ell})_k$ $\forall \ell,m$. $b$ contains the values of $x_s^\ell$ and $x_{s+e_k}^\ell$. The variables $\mathbf{\tilde{z}}$ are initialized by projecting the variables $\mathbf{z}$ to the affine space defined by the equation system $A\mathbf{w_s^k} = b$ for each $s$,$k$ combination individually. To also ensure that the non-negativity constraints on the $z_s^{\ell m}$ are fulfilled, the following substitution is applied until there are no more negative $\tilde{z}_s^{\ell m}$. Assuming $\tilde{z}_s^{\ell',m'} < 0$, from $x_s^{\ell} \geq 0$ it follows that there are $\tilde{z}_s^{\ell',m''} > 0$ and $\tilde{z}_s^{\ell'',m'} > 0$. Hence, we update
 \vspace{-0.15cm}
 \begin{align}
\vspace{-0.1cm}
  \tilde{z}_s^{\ell',m'} \leftarrow \tilde{z}_s^{\ell',m'} + \epsilon \quad & \quad \tilde{z}_s^{\ell'',m''} \leftarrow \tilde{z}_s^{\ell'',m''} + \epsilon \\
  \tilde{z}_s^{\ell',m''} \leftarrow \tilde{z}_s^{\ell',m''} - \epsilon \quad & \quad \tilde{z}_s^{\ell'',m'} \leftarrow \tilde{z}_s^{\ell'',m'} - \epsilon.
 \end{align}
 Note that this substitution does not affect the original constraints if we choose $\epsilon$ such that non-negative variables stay non-negative. The above substitution is iteratively applied until no more non-negative variables are left. By always choosing $\epsilon$ as big as possible, meaning such that either the non-positive variable $\tilde{z}_s^{\ell',m'}$ or one of the positive variables gets $0$, the number of iterations of the algorithm is bounded by $O(L^2)$. This holds because for each negative variable there is a maximum of $O(L)$ steps that can be made to increase it.
\end{proof}

\subsection{Semantic 3D Reconstruction: Additional Results}
Additional reconstructions are shown in Fig.~\ref{fig:add_reconstructions}.
We refer the reader to the supplementary video where renderings of our models can be found.

\subsection{Dataset "Head"}
We test our algorithm on a challenging specular "Head" dataset from \cite{cremers2011multiview}. It was shown in that paper that the results of traditional dense 3D reconstruction methods can be improved by utilizing the silhouette information. This information was included in their formulation as energy over rays. We show even more improvement by using our non-convex ray potential formulation in Fig.~\ref{fig:head}.

\subsection{Middlebury: Additional Analysis}
We provide accuracy (Acc) and completeness (Comp) plots for Dino Ring dataset in Fig.~\ref{fig:ring_accuracy_completeness}. We also show additional renderings of reconstructions in Fig.~\ref{fig:additional_middlebury}.

Overall, besides being accurate (as shown in the paper), our algorithm produces reconstructions with very high completeness: for $5$ out of $6$ datasets our reconstructions have completeness above $99.5\%$.

\subsection{Why is Convex Formulation so Weak?}
In this section we give a small intuitive example why the convex relaxation gives a solution which is far from binary. We give this example for a $2$-label problem without regularization and use the following notation for the labels: $\occ$ means occupied, $f$ means free-space. Consider one ray of the length $N=3$ with costs $c_0^\occ = -2, \, c_1^\occ = -3, \, c_2^\occ = -2$ and the rest of the costs are $0$. This is a realistic example since it corresponds to allowing the uncertainty around the estimated depth position $i=1$ (for example, camera sees the wall and stereo matching provides an estimate of depth, but this estimate is noisy in practice, so the uncertainty window along the ray is  very desirable). Since we only consider single ray, the ray index $r$ is omitted and the voxel space indexing function $s_i$ simplifies to just position $i$ along the ray. The exact problem, which we are solving, would be (as a reminder, $y_{-1}^f$ is always set to be $1$):
\begin{align}\label{eq:simple_example}
 \psi = & -2 y_{0}^\occ - 3 y_{1}^\occ - 2 y_{2}^\occ \rightarrow \min\limits_{\bf x, \, y} \\
 \mbox{s.t. } \; & y_{i}^\occ \leq y_{i-1}^{f}, \; y_{i}^f \leq y_{i-1}^{f}, \nonumber \\ 
 & y_{i}^\occ \leq x_{i}^\occ, \; y_{i}^f \leq 1 - x_{i}^\occ, \nonumber \\
 & x_{i}^\occ \in [0, 1], \;\; \forall i. \nonumber
\end{align}

The desired solution to this problem would be
\begin{align}\label{eq:binary_solution}
 & x_{0}^\occ = 0, \; x_{1}^\occ = 1, \; x_{2}^\occ = 0, \\
 & y_{0}^\occ = 0, \; y_{1}^\occ = 1, \; y_{2}^\occ = 0, \nonumber \\
 & y_{0}^f = 1, \; y_{1}^f = 0, \; y_{2}^f = 0. \nonumber
\end{align}
This means taking the best position in the uncertainty window. This solution has the cost $c_{binary} = -3$. Unfortunately, the solution where all the variables above take value $0.5$ has a better cost: $c_{0.5} = -3.5$.

Our preliminary investigations indicate that the "all-$0.5$" solution will always be the optimal solution to the convex relaxation as long as the best cost $c_{\min}^\occ = \min\limits_i c_i^\occ$ is larger than the sum of other occupied costs (as it is the case in the example above, $-3$ versus $-4$).

\subsection{Convergence Analysis}
In this section we analyze the convergence behavior of our method.

First, we evaluate how fast the algorithm converges using different minimization intervals in between the majorization steps. In Fig.~\ref{fig:switches} we can see that a frequent execution of the majorization step has a very beneficial effect on the convergence. Additionally, we see that for a broad range of values we reach similar (in energy) critical points of our cost function. This is a strong indication that our method is robust against bad solutions.

Second, we analyze tie handling in Eq.~\ref{eq:linearization} of the main text. As a reminder, this equation describes linear majorizer as
\begin{multline}
  g(x_{s_{i}}^f,y_{i-1}^f | x_{s_{i}}^{f,(n)}, y_{i-1}^{f,(n)})  \\ =  \begin{cases}
   0 & \text{if } y_{i-1}^{f,(n)} \leq  x_{s_{i}}^{f,(n)} \\
   y_{i-1}^f - x_{s_{i}}^f & \text{if } y_{i-1}^{f,(n)} > x_{s_{i}}^{f,(n)} %\\
   %\alpha(y_{i-1}^f - x_{s_{i}}^f) & \alpha \in [0,1] \text{ otherwise,}
 \end{cases}
 \label{eq:again_linearization}
\end{multline}
In that equation the tie case is $y_{i-1}^{f,(n)} =  x_{s_{i}}^{f,(n)}$ and it is possible to choose any of the two branches in this case: $0$ or $y_{i-1}^f - x_{s_{i}}^f$. Our experiment in Fig.~\ref{fig:tie} shows that the difference in final energies between these two choices is very small, $0.25\%$ of their values.

\begin{figure*}
 \centering
 \begin{tabular}{cccc}
 \includegraphics[width=0.15\linewidth]{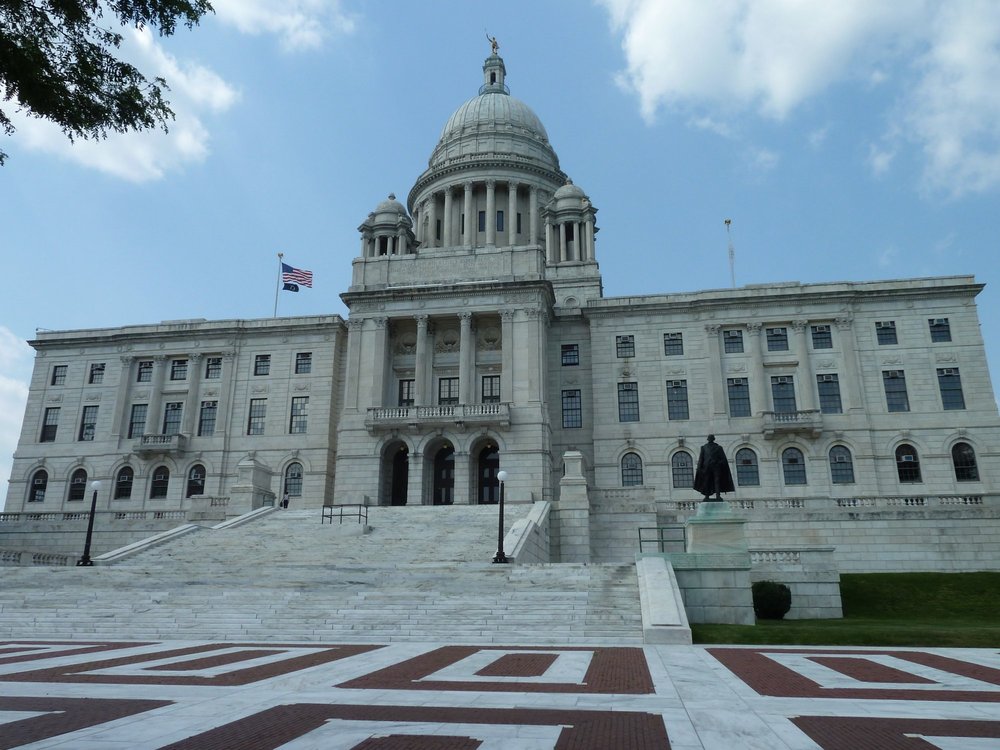} &
 \includegraphics[width=0.25\linewidth]{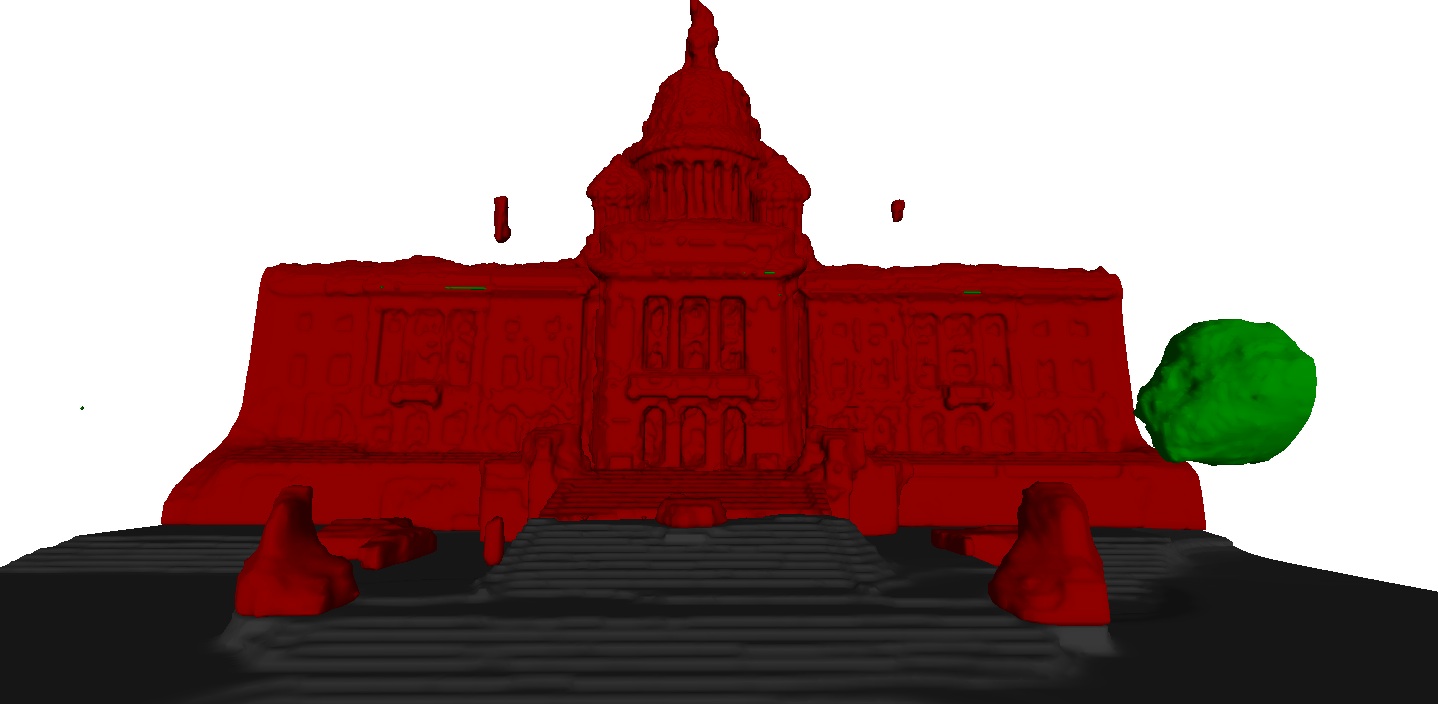} &
  \includegraphics[width=0.25\linewidth]{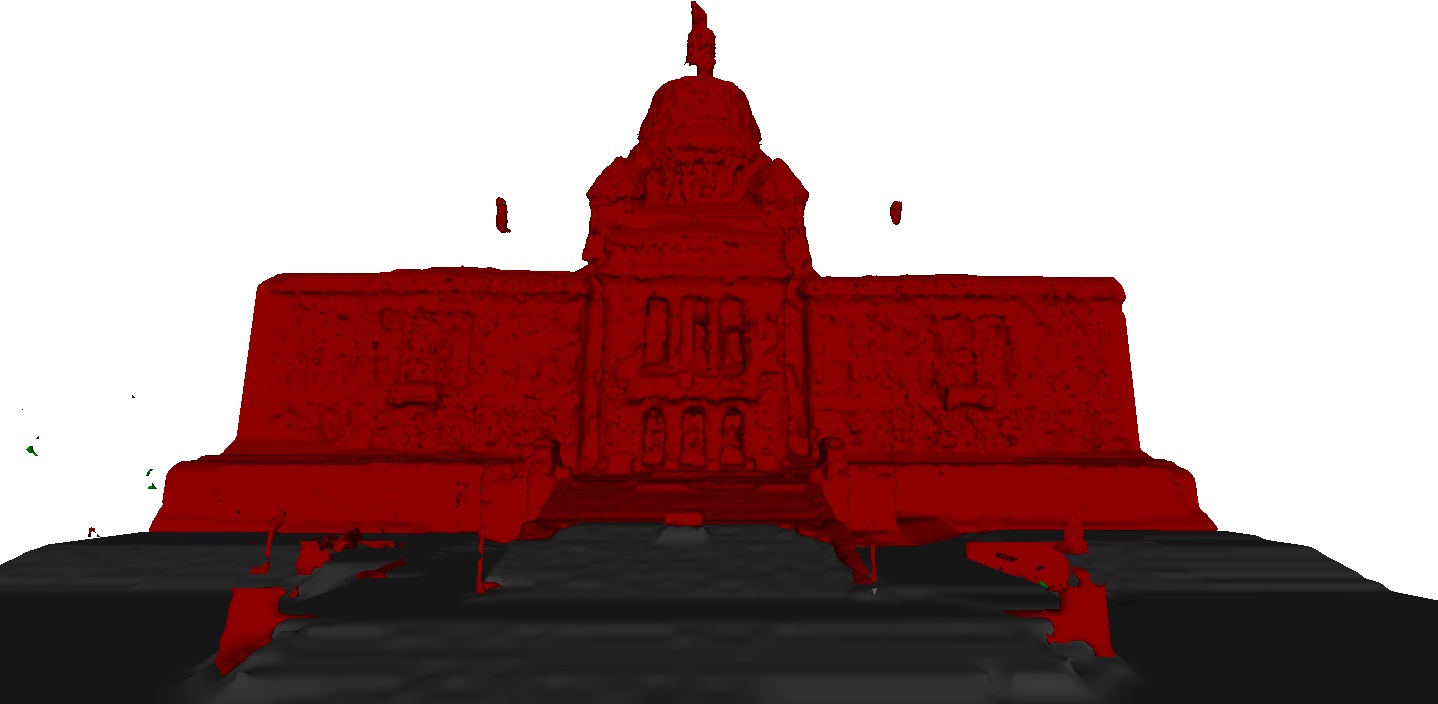} &
   \includegraphics[width=0.25\linewidth]{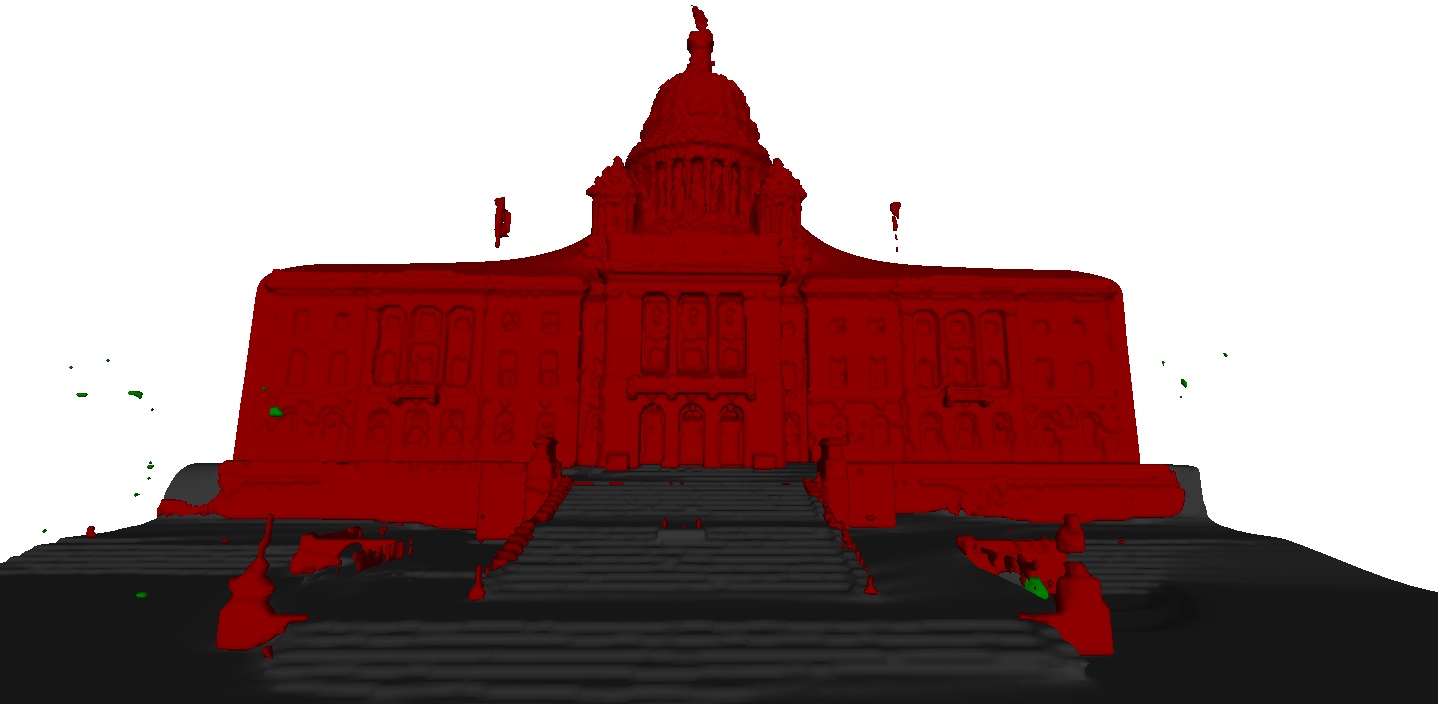} \vspace{0.00cm} \\
  Input Image & H{\"a}ne et al.\ 2013 \cite{Hane13} & Savinov et al. 2015 \cite{Savinov15} & Proposed Method
  \end{tabular}
 \caption{\label{fig:add_reconstructions} Semantic 3D Reconstructions.}
\end{figure*}

\begin{figure*}
 \begin{center}
    \begin{tabular}{ccccc}
    Original Images& \cite{vogiatzis2005multi} & \cite{Hernandez07}, \cite{Kolev07propagatedphotoconsistency} & \cite{cremers2011multiview} & Our Method \\
 \includegraphics[width=0.1\linewidth]{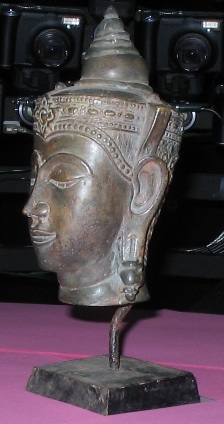} &
  \includegraphics[width=0.1\linewidth]{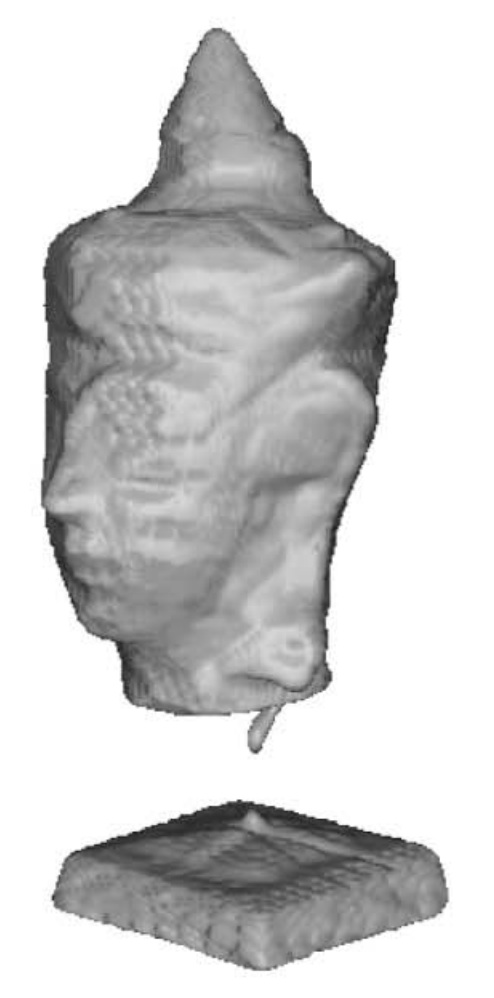} &
    \includegraphics[width=0.1\linewidth]{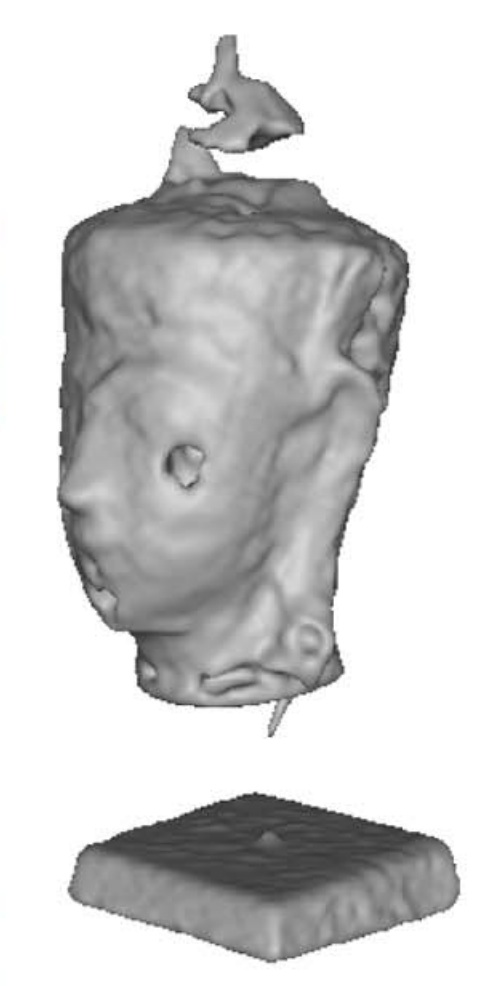} &
     \includegraphics[width=0.1\linewidth]{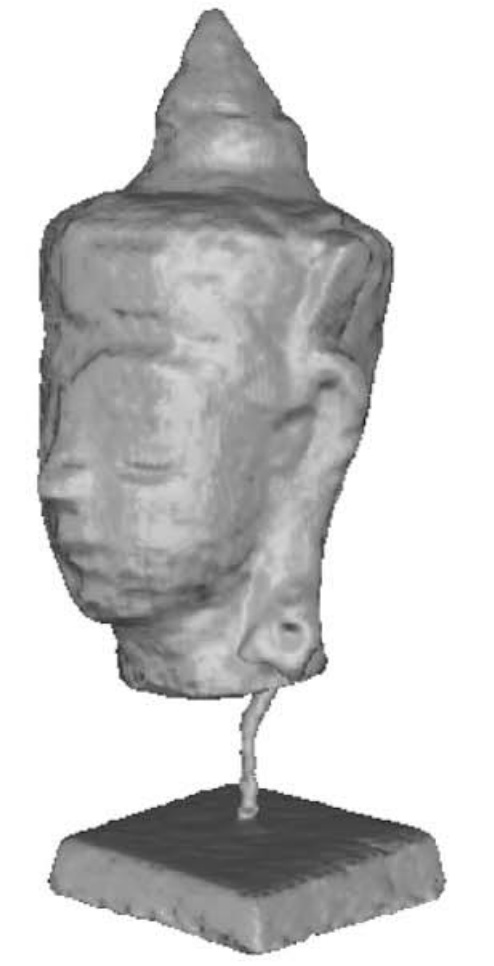} &
      \includegraphics[width=0.1\linewidth]{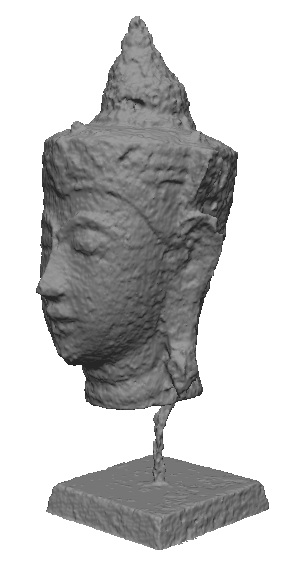} \\
 \includegraphics[width=0.1\linewidth]{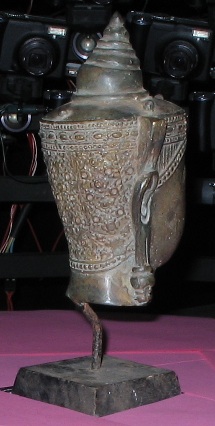} &
  \includegraphics[width=0.1\linewidth]{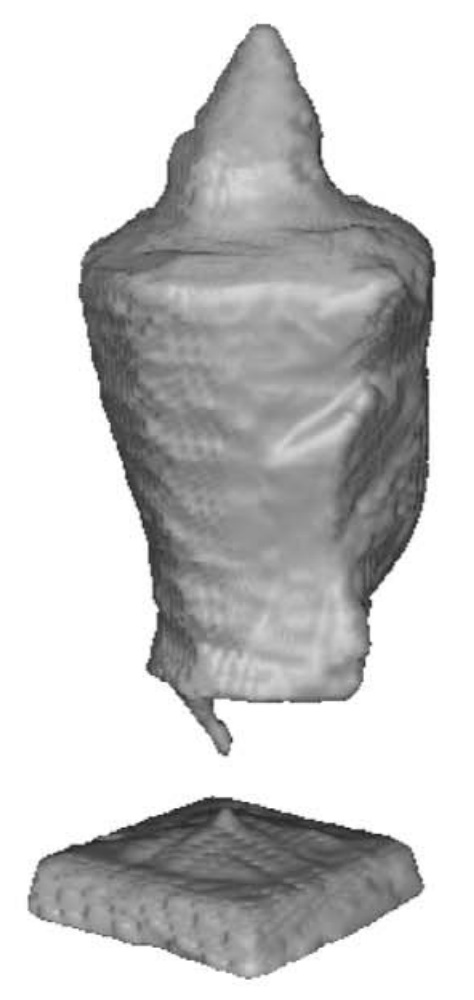} &
   \includegraphics[width=0.1\linewidth]{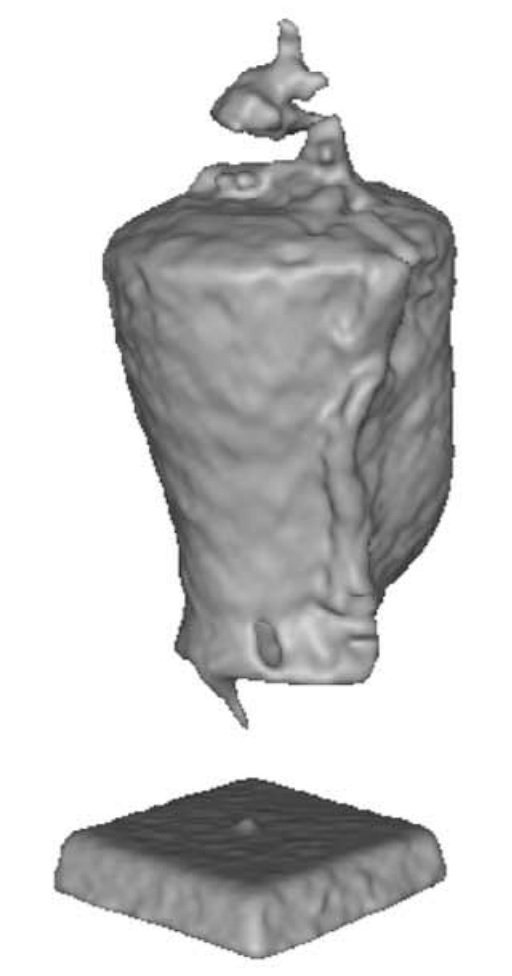} &
    \includegraphics[width=0.1\linewidth]{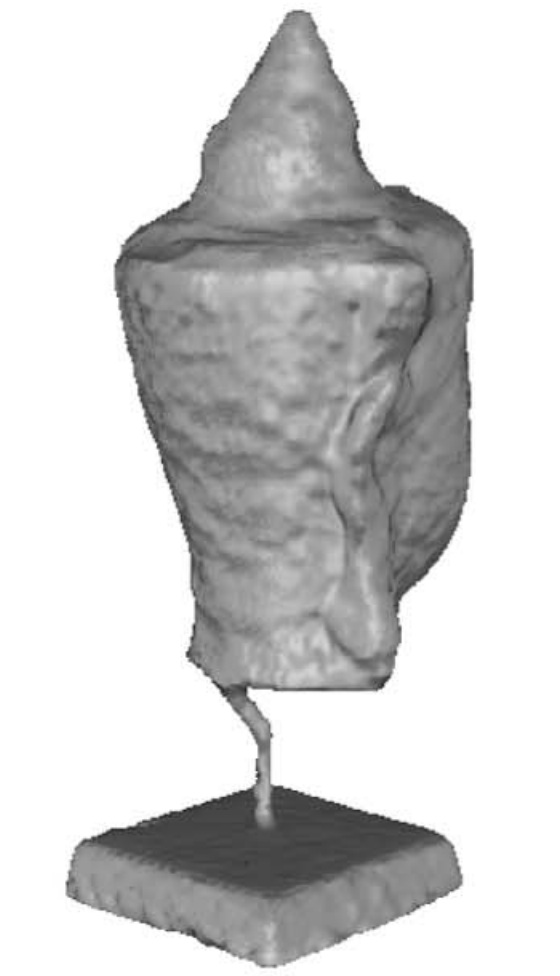} &
     \includegraphics[width=0.1\linewidth]{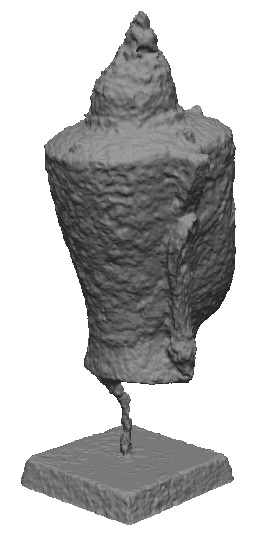} \\
 \includegraphics[width=0.1\linewidth]{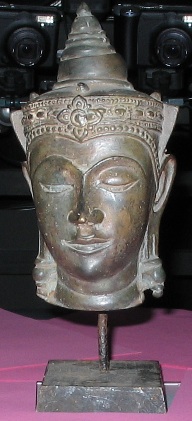} &
  \includegraphics[width=0.1\linewidth]{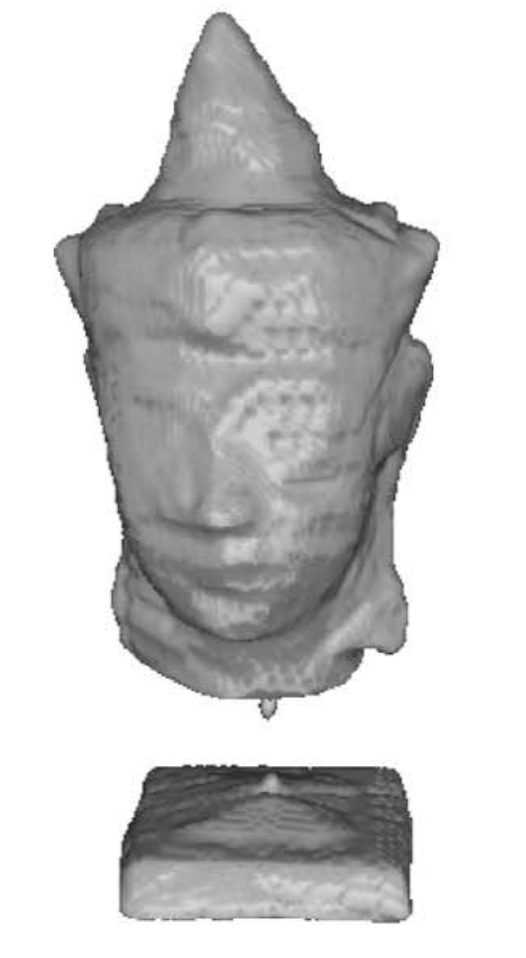} &
   \includegraphics[width=0.1\linewidth]{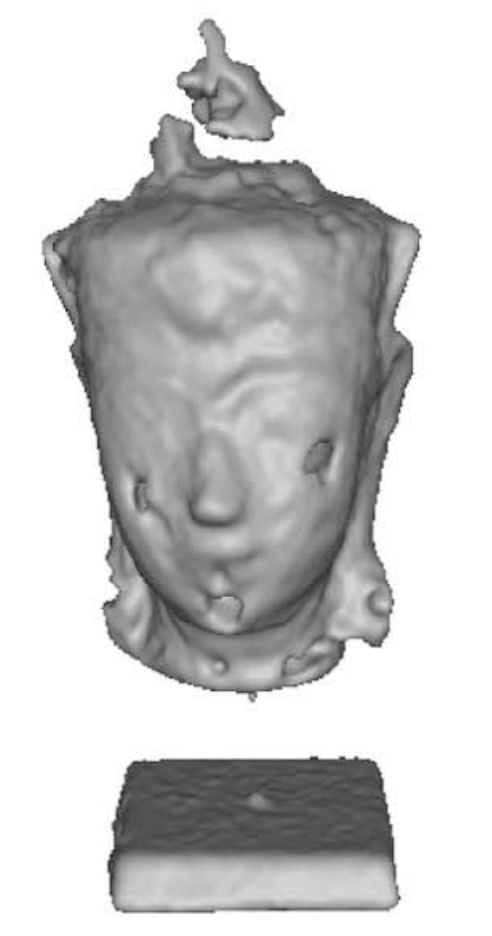} &
    \includegraphics[width=0.1\linewidth]{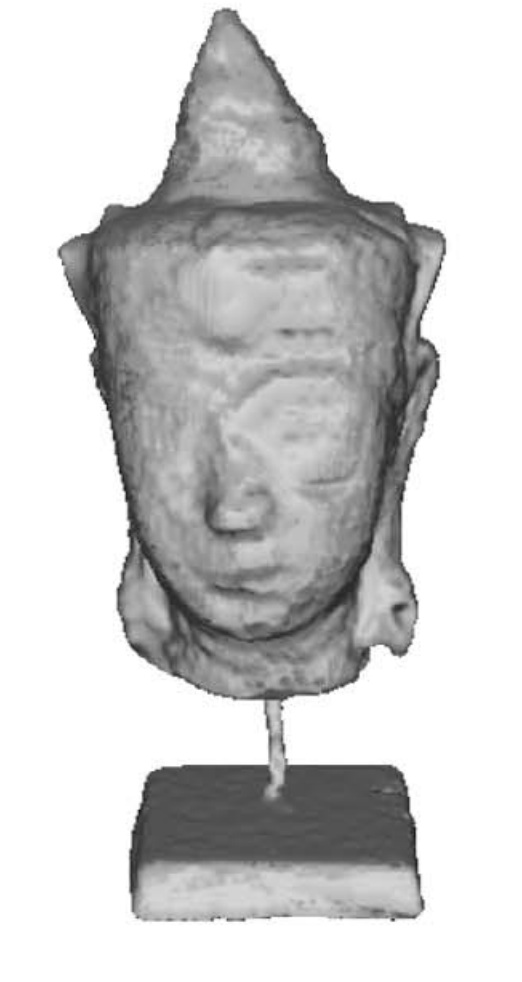} &
     \includegraphics[width=0.1\linewidth]{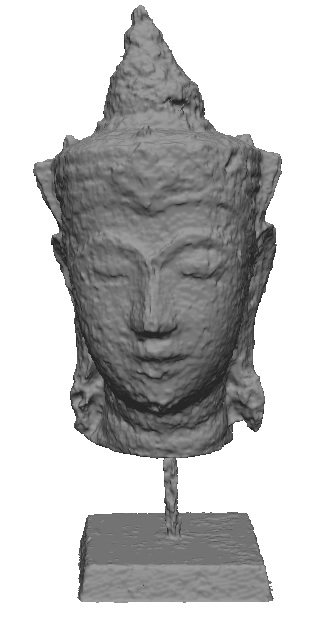} \\
 \includegraphics[width=0.1\linewidth]{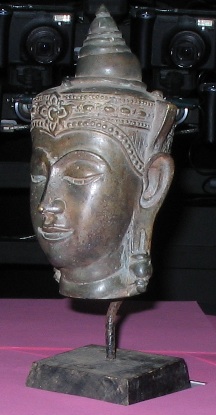} &
  \includegraphics[width=0.1\linewidth]{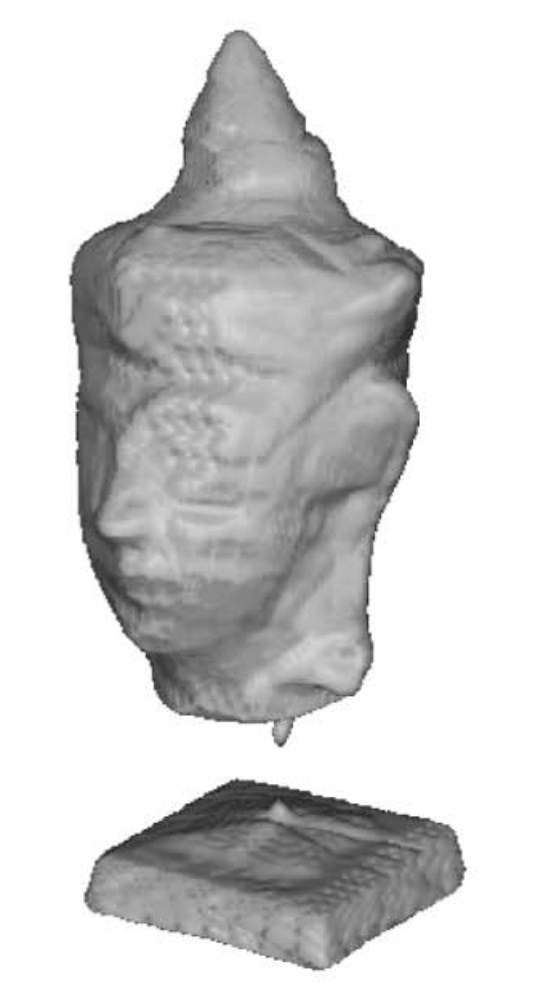} &
   \includegraphics[width=0.1\linewidth]{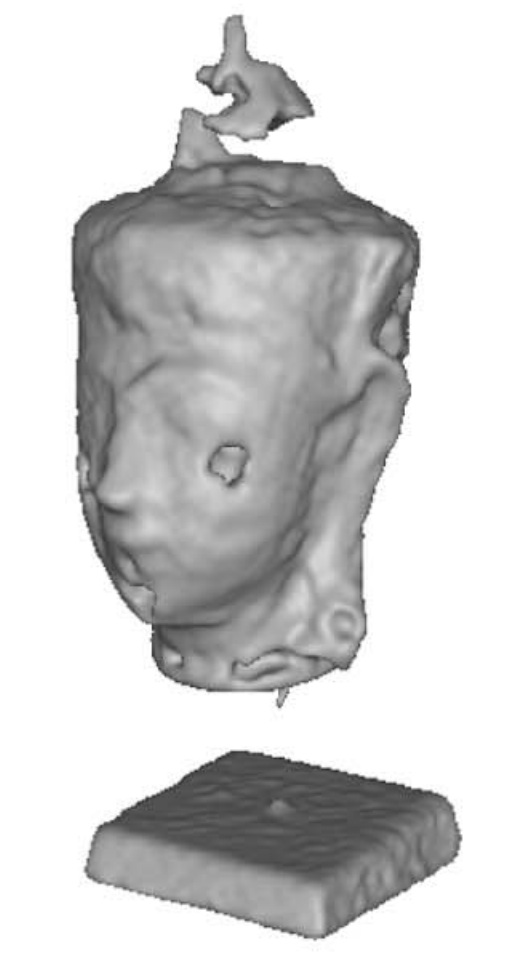} &
    \includegraphics[width=0.1\linewidth]{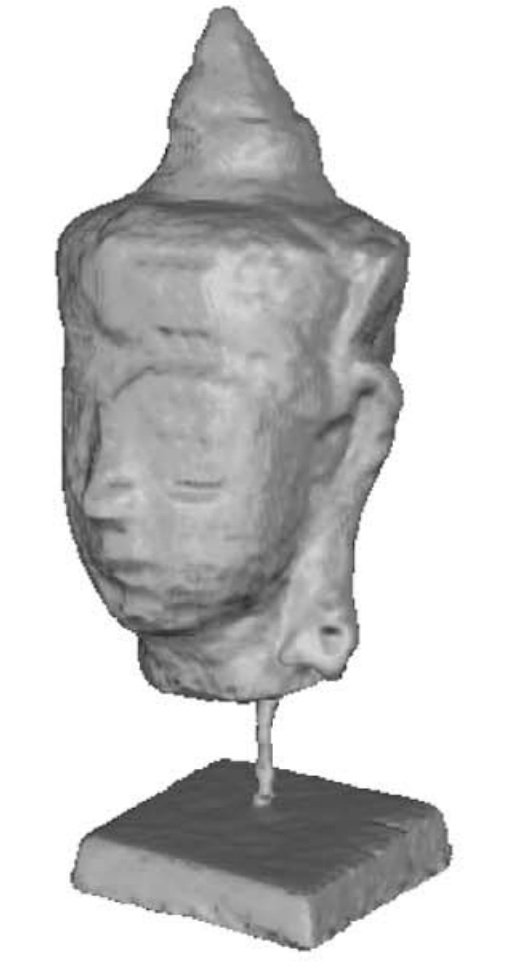} &
     \includegraphics[width=0.1\linewidth]{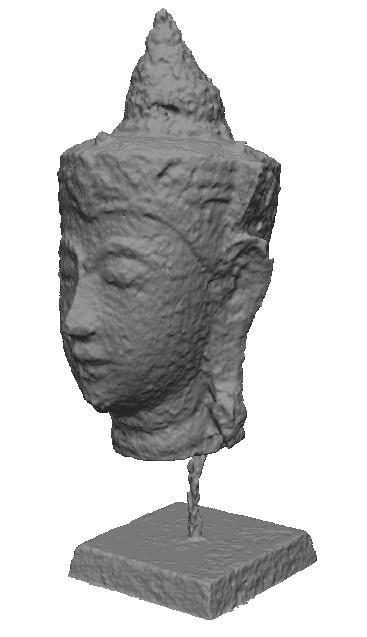}
    \end{tabular}
     \caption{\label{fig:head} Rendering of the results on the "Head" dataset. The columns from two to four are reported by \cite{cremers2011multiview}. It has been shown in \cite{cremers2011multiview} that ray information can help in reconstructing the thin pole on which the head is mounted. Our algorithm successfully reconstructs this pole as well.}
 \end{center}
 \end{figure*}

\begin{figure*}
\begin{center}
 \begin{tabular}{cc}
  \includegraphics[width=0.4\linewidth]{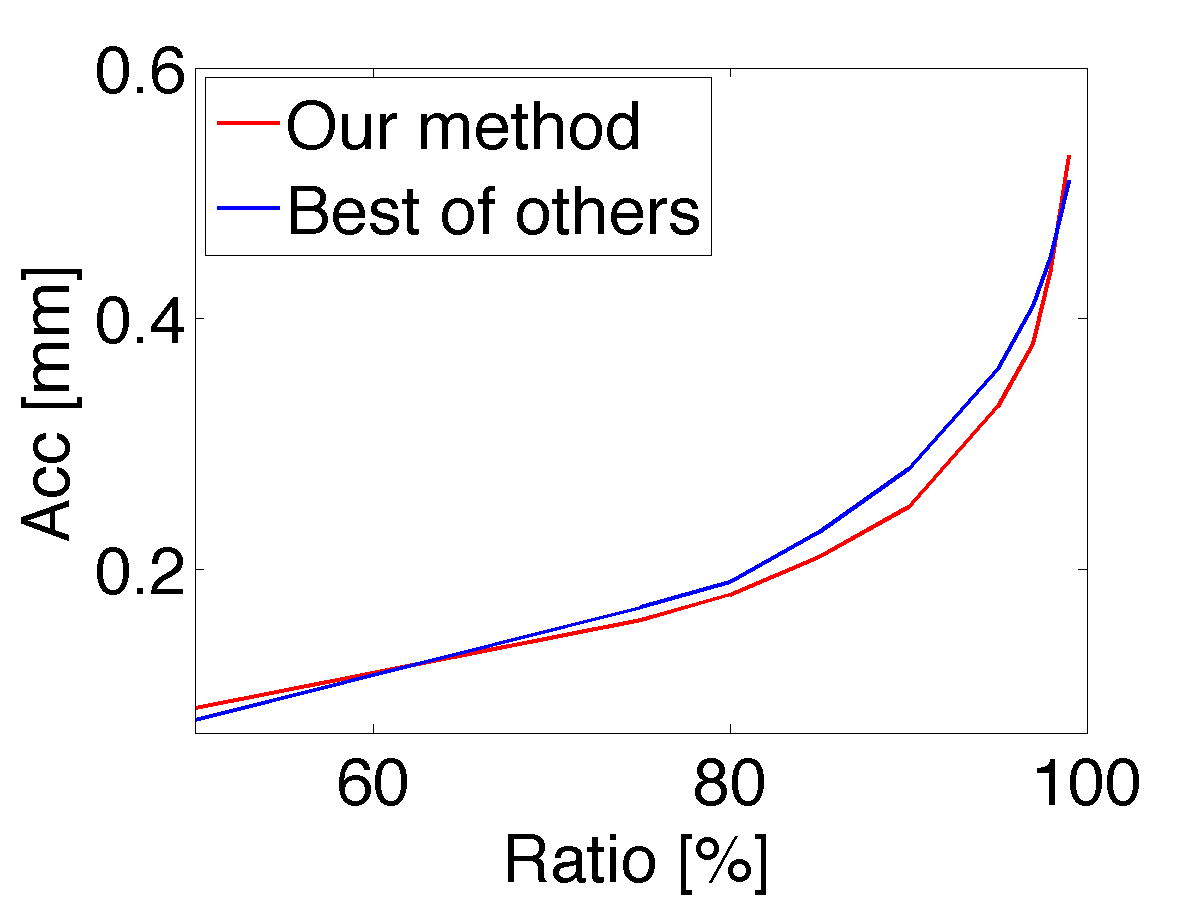} &
  \includegraphics[width=0.4\linewidth]{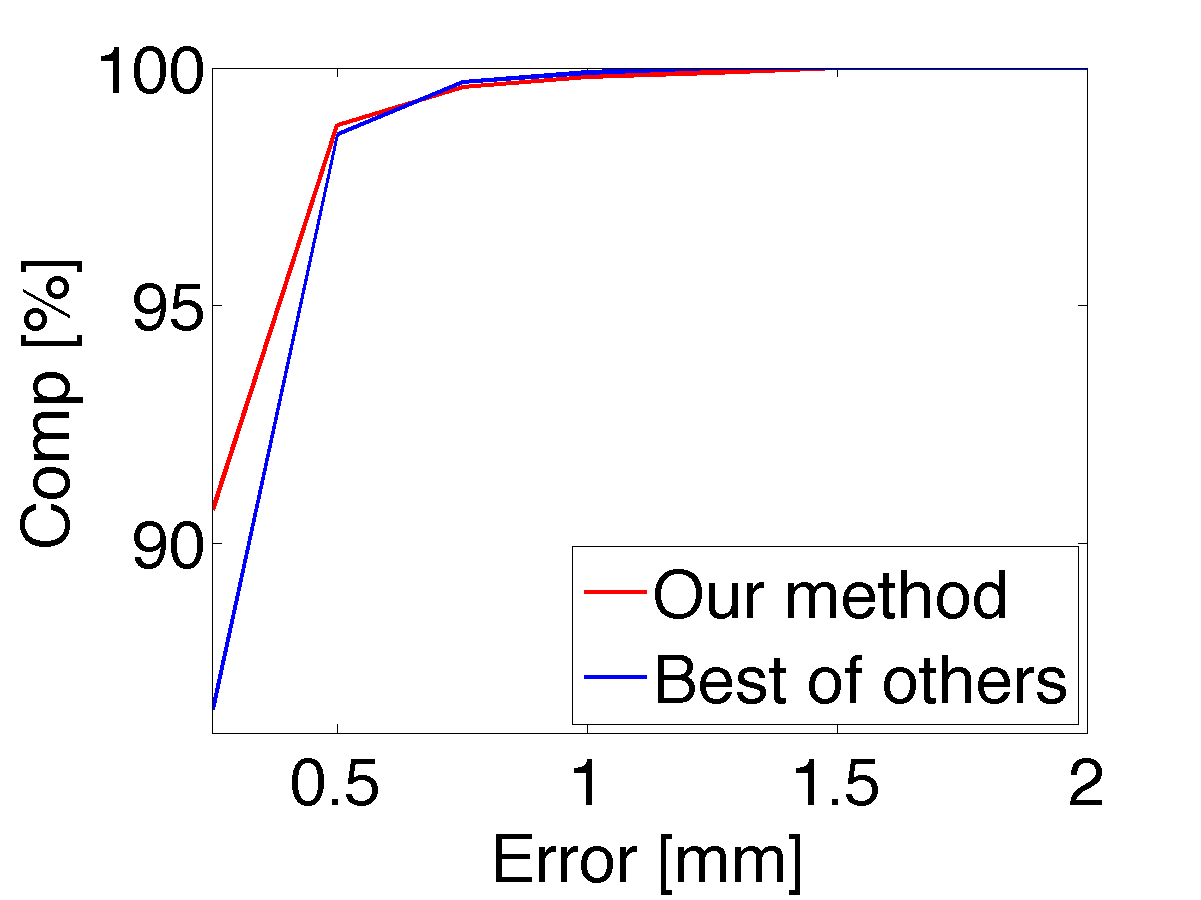}
  \end{tabular}
  \vspace{2mm}
  \caption{\label{fig:ring_accuracy_completeness} Acc vs. Ratio (lower curve better) and Comp vs. Error (higher curve better) plots for the Dino Ring dataset of the Middlebury benchmark (for details on these plots see \cite{middlebury}).}
  \end{center}
  \end{figure*}
  
\begin{figure*}
 \begin{center}
    \begin{tabular}{cccccc}
 Dino Full & Dino Sparse & Temple Full & Temple Sparse \\
 \includegraphics[width=0.13\linewidth]{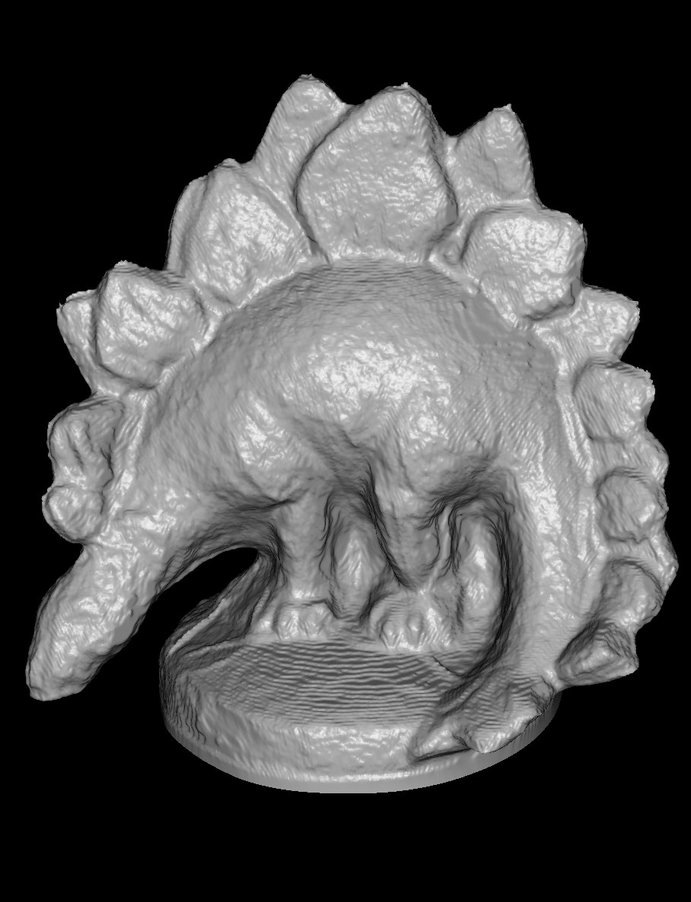} &
 \includegraphics[width=0.13\linewidth]{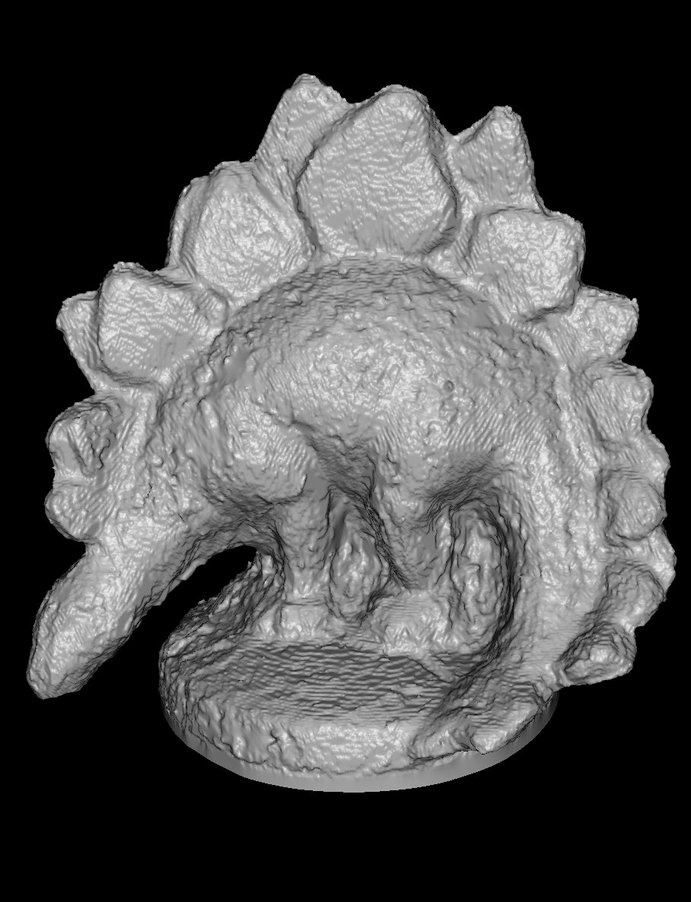} &
 \includegraphics[width=0.13\linewidth]{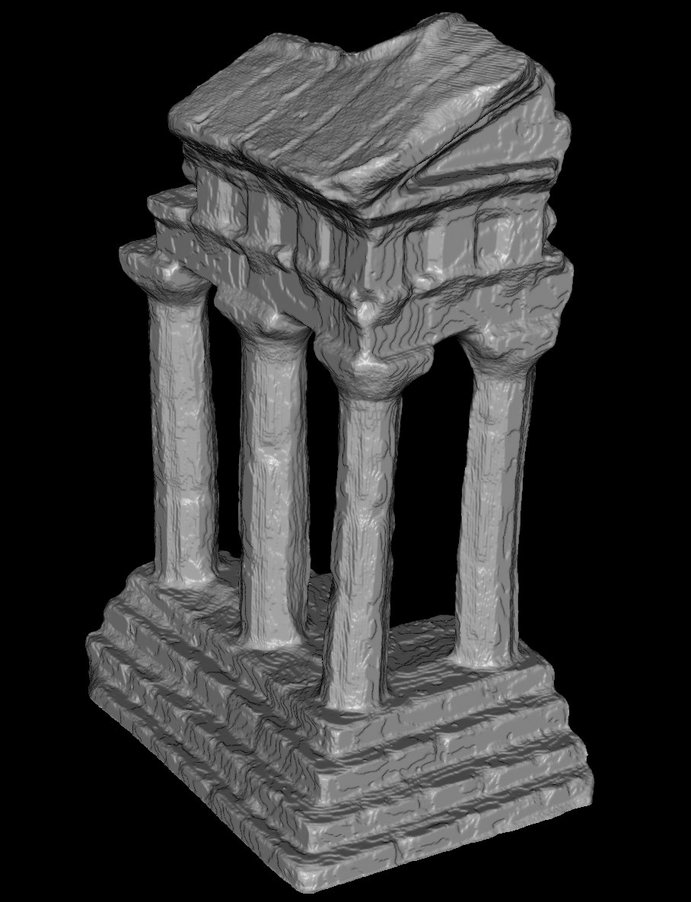} &
 \includegraphics[width=0.13\linewidth]{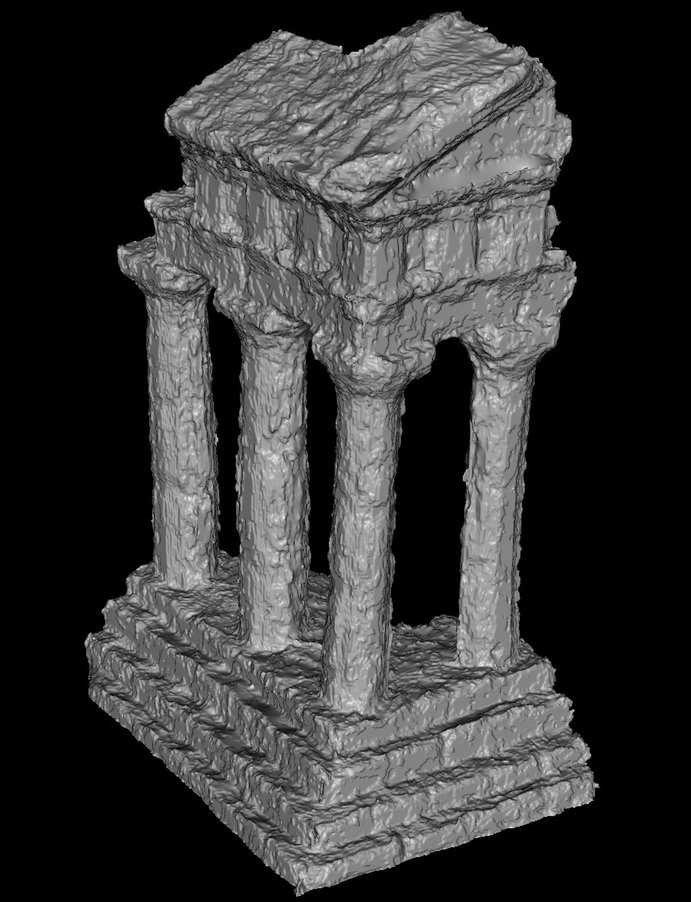} \\
 \includegraphics[width=0.13\linewidth]{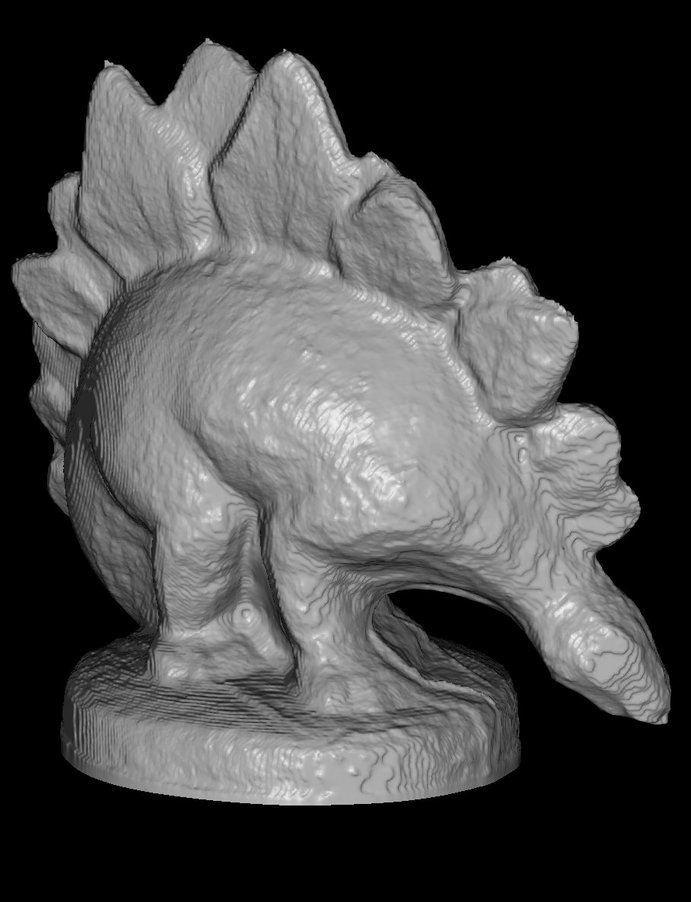} &
 \includegraphics[width=0.13\linewidth]{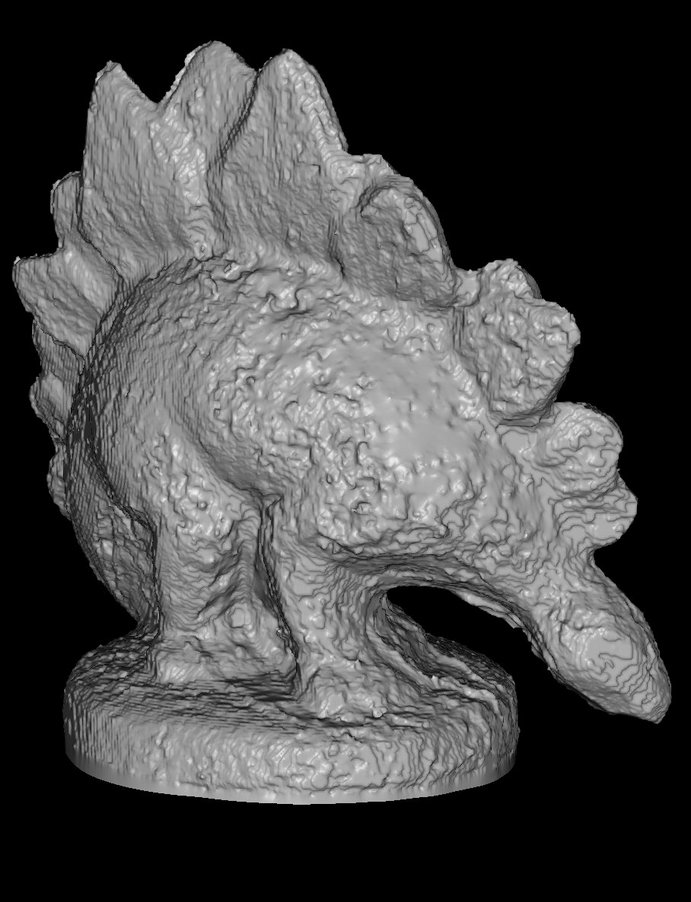} & 
  \includegraphics[width=0.13\linewidth]{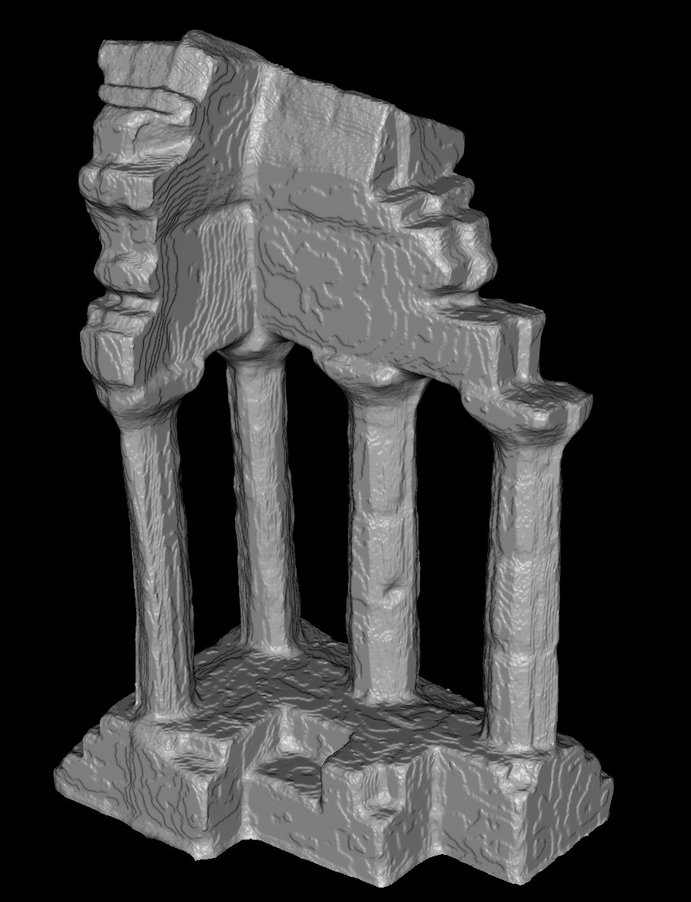} &
 \includegraphics[width=0.13\linewidth]{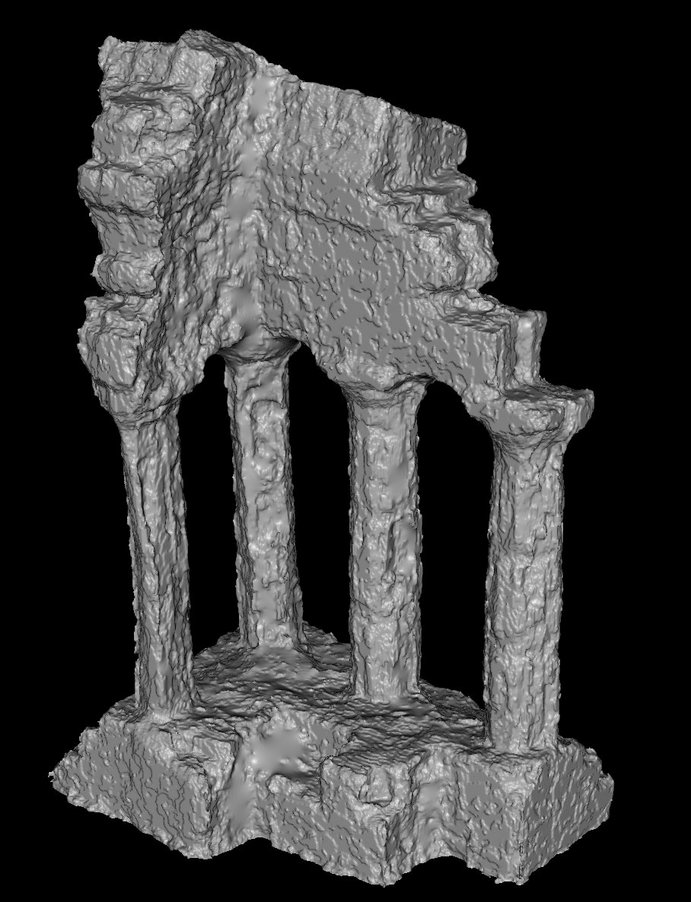}
    \end{tabular}
     \caption{\label{fig:additional_middlebury} Rendering of Middlebury results.}
 \end{center}
 \end{figure*}

  \begin{figure*}
   \begin{center}
   \begin{tabular}{cc}
   Providence & Catania\\
   \includegraphics[width=0.4\linewidth]{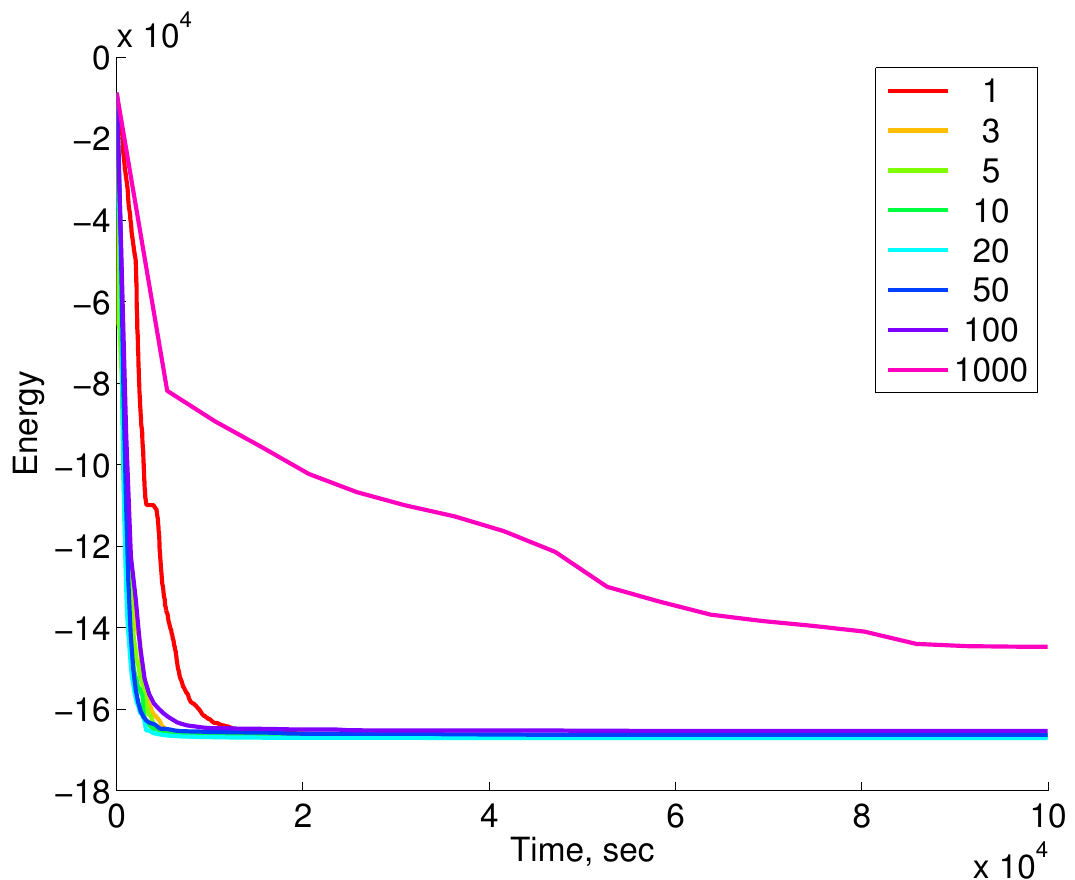} &
   \includegraphics[width=0.4\linewidth]{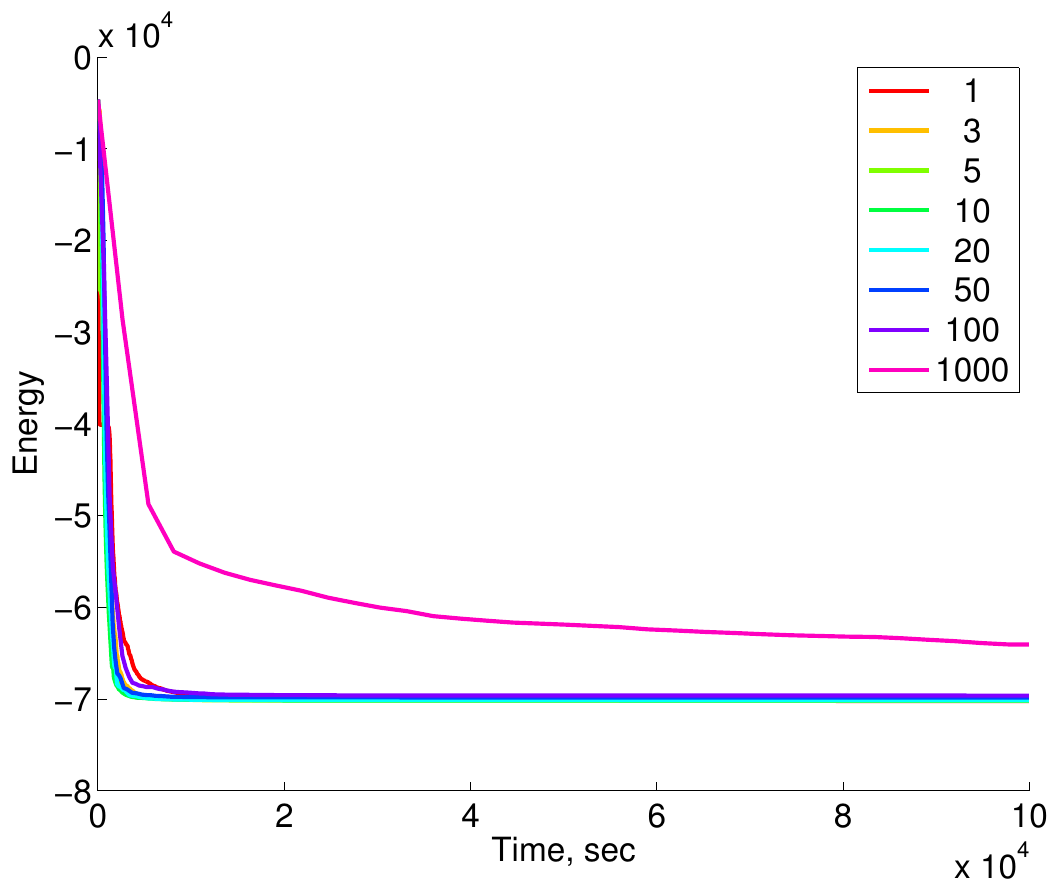}
   \end{tabular}
   \vspace{2mm}
   \caption{\label{fig:switches} Evolution of the energy over time for different numbers of iterations the convex minimization algorithm is run in between the execution of the majorization step.}
   \end{center}
   \end{figure*}
   
\begin{figure*}
   \begin{center}
      \includegraphics[width=0.4\linewidth]{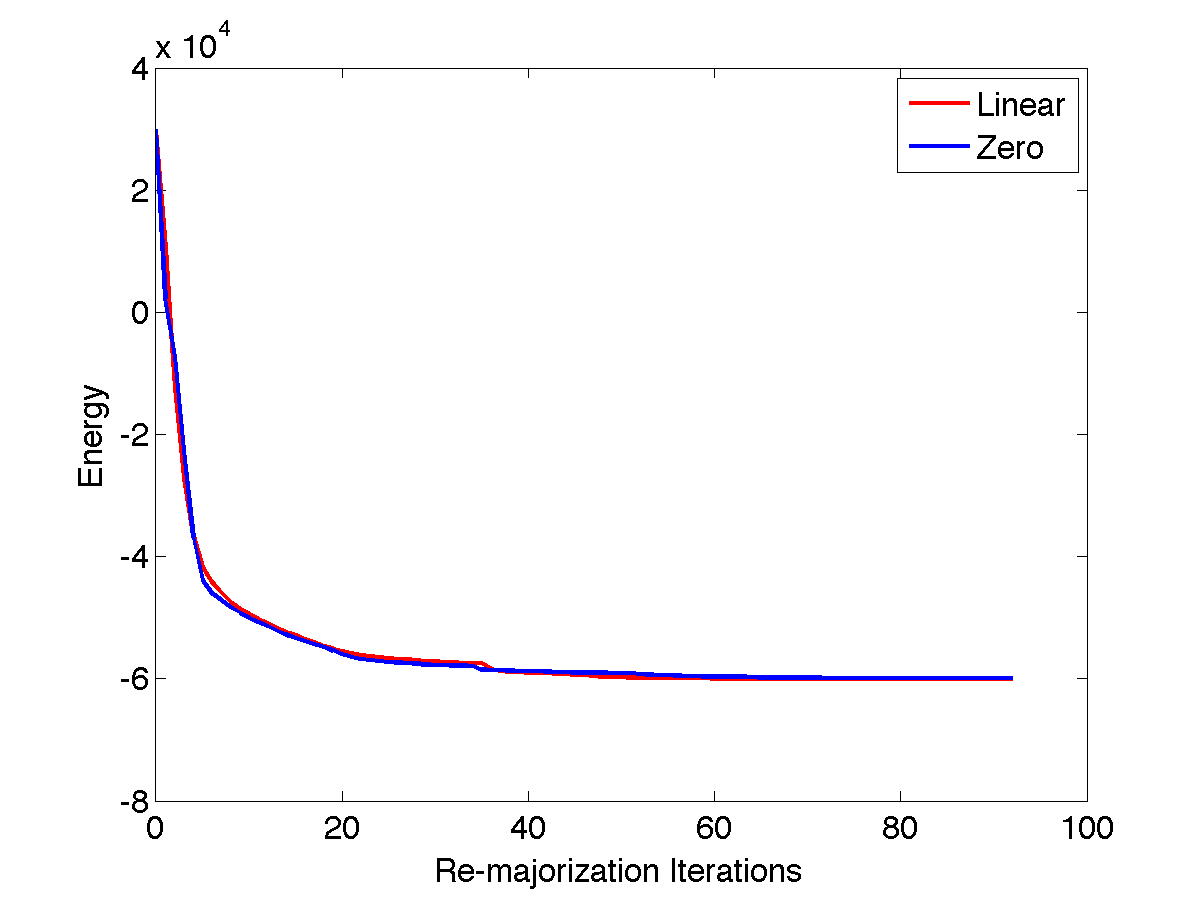}
   \vspace{2mm}
   \caption{\label{fig:tie} Evolution of the energy over iterations for two different re-majorization strategies. "Linear" means that the tie case is handled with the linear branch, "zero" means that constant branch with $0$ value is taken.}
   \end{center}
   \end{figure*}

\end{appendices}

\end{document}